\documentclass[nohyperref]{article}

\usepackage{enumitem}
\usepackage{amsfonts}       %
\usepackage{nicefrac}       %
\usepackage{microtype}      %
\usepackage{xcolor}         %

\usepackage{microtype}
\usepackage{graphicx}
\usepackage{subfigure}
\usepackage{multirow}
\usepackage{booktabs} %
\usepackage{minitoc}

\usepackage{hyperref}
\usepackage{url}            %

\usepackage{sty/wdel}
\usepackage{sty/rl-notations}
\usepackage{sty/dm-colors}

\usepackage[accepted]{icml2023}

\usepackage{natbib}
\usepackage{amsmath}
\usepackage{amssymb}
\usepackage{mathtools}
\usepackage{amsthm}

\usepackage[capitalize,noabbrev]{cleveref}

\newcount\issubmission
\ifnum\issubmission=0
\newcommand{\tadashi}[1]{\textcolor{dmorange500}{\textbf{[Tadashi: }#1\textbf{]}}}
\newcommand{\mfg}[1]{\textcolor{pink}{\textbf{[Matthieu: }#1\textbf{]}}}
\newcommand{\ywh}[1]{\textcolor{brown}{\textbf{[Wenhao: }#1\textbf{]}}}
\newcommand{\remi}[1]{\textcolor{cyan}{\textbf{[Remi: }#1\textbf{]}}}
\newcommand{\toshinori}[1]{\textcolor{magenta}{\textbf{[Toshinori: }#1\textbf{]}}}
\newcommand{\pierre}[1]{\textcolor{yellow}{\textbf{[pierre: }#1\textbf{]}}}
\newcommand{\nino}[1]{\textcolor{green}{\textbf{[Nino: }#1\textbf{]}}}

\else
\newcommand{\tadashi}[1]{}
\newcommand{\mfg}[1]{}
\newcommand{\ywh}[1]{}
\newcommand{\remi}[1]{}
\newcommand{\toshinori}[1]{}
\newcommand{\pierre}[1]{}
\newcommand{\nino}[1]{}
\fi

\theoremstyle{plain}
\newtheorem{theorem}{Theorem}[section]

\newtheorem{lemma}[theorem]{Lemma}
\newtheorem{corollary}[theorem]{Corollary}
\newtheorem{event}{Event}
\theoremstyle{definition}

\newtheorem{assumption}[theorem]{Assumption}
\theoremstyle{remark}

\usepackage{xspace}

\newcommand{\tmdvi}{\hyperref[algo:tabular mdvi]{\texttt{Tabular MDVI}}\xspace}
\newcommand{\vwlsmdvi}{\hyperref[algo:vwlsmdvi]{\texttt{VWLS-MDVI}}\xspace}
\newcommand{\wlsmdvi}{\hyperref[algo:wlsmdvi]{\texttt{WLS-MDVI}}\xspace}
\newcommand{\dvw}{\hyperref[algo:DVW general]{\texttt{DVW}}\xspace}
\newcommand{\kwalgo}{\hyperref[assumption:compute optimal design]{\texttt{ComputeOptimalDesign}}\xspace}
\newcommand{\varianceestimate}{\hyperref[algo:varianceestimate]{\texttt{VarianceEstimation}}\xspace}

\newcommand{\funcE}{\hyperref[event:f]{\cE_1}\xspace}
\newcommand{\vboundE}{\hyperref[event:v is bounded]{\cE_2}\xspace}
\newcommand{\EkboundE}{\hyperref[event:coarse E_k bound]{\cE_3}\xspace}
\newcommand{\epskboundE}{\hyperref[event:eps_k bound core set]{\cE_4}\xspace}
\newcommand{\EkrboundE}{\hyperref[event:refined E_k bound]{\cE_5}\xspace}
\newcommand{\phiqboundE}{\hyperref[event:phi q bound]{\cE_6}\xspace}
\newcommand{\phisigmaboundE}{\hyperref[event:phi sigma bound]{\cE_7}\xspace}

\newcommand\numeq[2]%
  {\stackrel{\scriptscriptstyle(\mkern-1.5mu#2\mkern-1.5mu)}{#1}}

\usepackage[textsize=tiny]{todonotes}

\icmltitlerunning{Regularization and Variance-Weighted Regression is Minimax Optimality in Linear MDPs}

\begin{document}

\twocolumn[
\icmltitle{Regularization and Variance-Weighted Regression Achieves \\ Minimax Optimality in Linear MDPs: Theory and Practice}

\icmlsetsymbol{equal}{*}

\begin{icmlauthorlist}  %
\icmlauthor{Toshinori Kitamura}{tokyo}
\icmlauthor{Tadashi Kozuno}{omron}
\icmlauthor{Yunhao Tang}{deepmind}
\icmlauthor{Nino Vieillard}{google}
\icmlauthor{Michal Valko}{deepmind}
\icmlauthor{Wenhao Yang}{peking}
\icmlauthor{Jincheng Mei}{google}
\icmlauthor{Pierre M\'enard}{magdeburg}
\icmlauthor{Mohammad Gheshlaghi Azar}{deepmind}
\icmlauthor{R\'emi Munos}{deepmind}
\icmlauthor{Olivier Pietquin}{google}
\icmlauthor{Matthieu Geist}{google}
\icmlauthor{Csaba Szepesv\'ari}{alberta,deepmind}
\icmlauthor{Wataru Kumagai}{tokyo}
\icmlauthor{Yutaka Matsuo}{tokyo}
\end{icmlauthorlist}

\icmlaffiliation{omron}{OMRON SINIC X, Japan}
\icmlaffiliation{peking}{Peking University}
\icmlaffiliation{google}{Google Research, Brain team}
\icmlaffiliation{tokyo}{The University of Tokyo, Japan}
\icmlaffiliation{deepmind}{DeepMind}
\icmlaffiliation{magdeburg}{Otto von Guericke University Magdeburg}
\icmlaffiliation{alberta}{University of Alberta}

\icmlcorrespondingauthor{Toshinori Kitamura}{toshinori-k@weblab.t.u-tokyo.ac.jp}

\icmlkeywords{Reinforcement Learning Theory, Deep Reinforcement Learning, Linear MDP}

\vskip 0.3in
]

\printAffiliationsAndNotice{}  %

\begin{abstract}
Mirror descent value iteration (MDVI), an abstraction of Kullback--Leibler (KL) and entropy-regularized reinforcement learning (RL), has served as the basis for recent high-performing practical RL algorithms. However, despite the use of function approximation in practice, the theoretical understanding of MDVI has been limited to tabular Markov decision processes (MDPs). We study MDVI with linear function approximation through its sample complexity required to identify an $\varepsilon$-optimal policy with probability $1-\delta$ under the settings of an infinite-horizon linear MDP, generative model, and G-optimal design. We demonstrate that least-squares regression weighted by the variance of an estimated optimal value function of the next state is crucial to achieving minimax optimality. Based on this observation, we present Variance-Weighted Least-Squares MDVI (VWLS-MDVI), the first theoretical algorithm that achieves nearly minimax optimal sample complexity for infinite-horizon linear MDPs. Furthermore, we propose a practical VWLS algorithm for value-based deep RL, Deep Variance Weighting (DVW). Our experiments demonstrate that DVW improves the performance of popular value-based deep RL algorithms on a set of MinAtar benchmarks.
\end{abstract}

\section{Introduction}

\looseness=-1
Kullback--Leibler (KL) divergence and entropy regularization play an important role in recent reinforcement learning (RL) algorithms.
These regularizations are often introduced to promote exploration \citep{haarnoja2017reinforcement,pmlr-v80-haarnoja18b}, make algorithms more robust to errors \citep{husain2021regularized,bellemare2016increasing}, and ensure that performance improves over time \citep{schulman2015trust}.  %
The behavior of RL algorithms under these regularizations can be studied using mirror descent value iteration (MDVI; \citet{geist2019theory}), a value iteration algorithm that incorporates KL and entropy regularization in its value and policy updates.
Notably, when both regularizations are combined, MDVI is proven to achieve nearly minimax optimal sample complexity\footnote{We only study the sample complexity (number of calls to a generative model) and ignore the computational complexity (total number of logical and arithmetic operations that the agent uses). \looseness=-1} with the generative model (simulator) in infinite-horizon MDPs, which indicates that it can exhibit good performance with relatively few samples \citep{kozuno2022kl}.
This analysis supports the state-of-the-art performance of the recent Munchausen DQN (M-DQN, \citet{vieillard2020munchausen}), which is a natural extension of MDVI to a value-based deep RL algorithm.

\looseness=-1
However, the minimax optimality of MDVI has only been proven for tabular Markov decision processes (MDPs), and does not consider the challenge of generalization in RL.
As practical RL algorithms often use function approximators to obtain generalizability, this leads to a natural question: \emph{Is MDVI minimax optimal with function approximation?}
The answer to this question should reveal room for improvement in existing practical MDVI-based algorithms such as M-DQN.
This study addresses the question by investigating the sample complexity of a model-free infinite-horizon $(\varepsilon, \delta)$-PAC RL algorithm, i.e., the expected number of calls to the generative model to identify an $\varepsilon$-optimal policy with a failure probability less than $\delta$, under the assumptions of linear MDP \citep{jin2020provably}, access to all the state-action pairs with a generative model, and a G-optimal design \citep{lattimore2020goodFeature}.
Intuitively, these assumptions allow us to focus on the value update rule, which is the core of RL algorithms, based on the following mechanisms; the access to all the state-action pairs with the generative model removes difficulties of exploration, the linear MDP provides a good representation, and the G-optimal design provides access to an effective dataset.
We explain in \cref{sec:related work} why the study of infinite-horizon RL is of value.

\looseness=-1
In \cref{sec:weighted least squares mdvi}, we provide positive and negative answers to the aforementioned question.
We demonstrate that a popular method for extending tabular algorithms to function approximation, i.e., regressing the target value with least-squares \citep{bellman1963polynomial,munos2005error}, can result in sub-optimal sample complexity in MDVI.
This suggests that in the case of function approximation, algorithms such as M-DQN, which rely mainly on the power of regularization, may exhibit a sub-optimal performance in terms of sample complexity.
However, we confirm that MDVI achieves nearly minimax optimal sample complexity when the least-squares regression is weighted by the variance of the optimal value function of the next state.
We prove these scenarios using our novel proof tool, the \emph{weighted Kiefer--Wolfowitz} (KW) theorem, which allows us to use the total variance (TV) technique \citep{azar2013minimax} to provide a $\sqrt{(1 - \gamma)^{-1}}$ tighter performance bound than the vanilla KW theorem \citep{kiefer1960equivalence,lattimore2020goodFeature}, where $\gamma$ denotes the discount factor.

\looseness=-1
Based on the theoretical observations, we propose both theoretical and practical algorithms; a minimax optimal extension of MDVI to infinite-horizon linear MDPs, called Variance-Weighted Least-Squares MDVI (\vwlsmdvi, \cref{sec:vwls-mdvi}), and a practical weighted regression algorithm for value-based deep RL, called Deep Variance Weighting (\dvw, \cref{sec:practical algorithm}).
\vwlsmdvi is the first-ever algorithm with nearly minimax sample complexity under the setting of both model-based and model-free infinite-horizon linear MDPs.
\dvw is also the first algorithm that extends the minimax optimal theory of function approximation to deep RL.
Our experiments demonstrate the effectiveness of \dvw to value-based deep RL through an environment where we can compute oracle values (\cref{subsec:gridworld experiment}) and a set of MinAtar benchmarks (\citet{young2019minatar}, \cref{subsec:minatar experiment}).

\section{Related Work}\label{sec:related work}

\begin{table}[t]
\caption{Sample complexity comparison to find an $\varepsilon$-optimal policy under infinite-horizon Linear MDP. In the table, $d$ denotes the dimension of a linear MDP and $\gamma$ denotes the discount factor .}
\label{table:sample complexity}
\vskip 0.15in
\begin{center}
\begin{small}
\begin{tabular}{ccc}
\toprule
Algorithm (Publication) & Complexity\\
\midrule
G-Sampling-and-Stop \citep{taupin2022best} & $\widetilde{\cO} \paren*{\frac{d^2}{\varepsilon^2(1-\gamma)^4}}$\\
\textbf{\vwlsmdvi (proposed in this study)} & $\widetilde{\cO} \paren*{\frac{d^2}{\varepsilon^2(1-\gamma)^3}}$ \\
\midrule
Lower Bound \citep{weisz2022confident} & $\Omega \paren*{\frac{d^2}{\varepsilon^2(1-\gamma)^3}}$\\
\bottomrule
\end{tabular}
\end{small}
\end{center}
\vskip -0.1in
\end{table}

\looseness=-1
\paragraph{Minimax Infinite-Horizon RL with Linear Function Approximation.}
The development of minimax optimal RL with linear function approximation has significantly advanced in recent years owing to the study of \citet{zhou2021nearlyMinimax}.
\citet{zhou2021nearlyMinimax} proposed the Bernstein-type \emph{self-normalized concentration inequality} \citep{abbasi2011improved} and combined it with variance-weighted regression (VWR) to achieve minimax optimal regret bound for linear mixture MDPs. 
Then, \citet{hu2022nearly} and \citet{he2022nearly} built upon the VWR technique for linear MDPs to achieve minimax optimality.
VWR has also been used for tight analyses in offline RL \citep{yin2022near,xiong2022nearly}, off-policy policy evaluation \citep{min2021variance}, and RL with nonlinear function approximation \citep{yin2022offline,agarwal2022vo}.

\looseness=-1
Despite the development of minimax optimal RL with linear function approximation, their results are limited to the setting of finite-horizon episodic MDPs. 
However, in practical RL applications, it is not uncommon to encounter infinite horizons, as can be observed in robotics \citep{miki2022learning}, recommendation \citep{maystre2023optimizing}, and industrial automation \citep{zhan2022deepthermal}.
Additionally, many practical deep RL algorithms, such as DQN \citep{mnih2015human} and SAC \citep{pmlr-v80-haarnoja18b}, are designed as model-free algorithms for the infinite-horizon discounted MDPs. 
Despite the practical importance of this topic, the minimax optimal algorithm for infinite-horizon discounted linear MDPs was unknown until this study.
Our study not only developed the first minimax optimal algorithm but also became the first study to naturally extend it to a practical deep RL algorithm.

\looseness=-1
\paragraph{Generative Model Assumption.}
In the infinite-horizon setting, the assumption of a generative model is not uncommon because, in contrast to the finite-horizon episodic setting, the environment cannot be reset, rendering exploration difficult \citep{azar2013minimax,sidford2018nearOptimal,agarwal2020modeBased}.
In fact, efficient learning in the infinite-horizon setting without the generative model is believed to be achievable only when an MDP has a finite diameter \citep{jaksch2010near}.

The problem setting of our theory, where the generative model can be queried for any state-action pair, is known as \emph{random access} generative model setting. 
For this setting, \citet{lattimore2020goodFeature} and \citet{taupin2022best} provided infinite-horizon sample-efficient algorithms with a G-optimal design; however, their sample complexity is not minimax optimal. 
\citet{yang2019sampleOptimal} proposed an algorithm with minimax optimal sample complexity for infinite-horizon MDPs; however, their algorithm relies on the special MDP structure, called anchor state-action pairs, as input to the algorithm.
In contrast, the proposed \vwlsmdvi algorithm can be executed as long as we have access to all state-action pairs.
Comparison of sample complexity with that of previous algorithms for infinite-horizon Linear MDPs is summarized in \cref{table:sample complexity}.

\looseness=-1
\paragraph{Computational Complexity.}
Unfortunately, the computational complexity of algorithms using a G-optimal design, including our theoretical algorithm, can be inefficient \citep{lattimore2020goodFeature}. 
This issue is addressed by extending the problem setting to more practical scenarios, e.g., \emph{local access}, where the agent can query to the generative model only previously visited state-action pairs \citep{yin2022efficient,weisz2022confident}, or online RL. 
We empirically address the issue by proposing the practical VWR algorithm, i.e., \dvw, and demonstrate its effectiveness in an online RL setting. 
Unlike previous practical algorithms that utilize weighted regression \citep{schaul2015prioritized,kumar2020discor,lee2021sunrise}, the proposed \dvw possesses a theoretical background of statistical efficiency.
We leave theoretical extensions to wider problem settings as future works.

\section{Preliminaries}\label{sec:preliminaries}

For a set $\cS$, we denote its complement and its size by $\cS^c$ and $|\cS|$, respectively.
For $N \in \N$, let $\brack*{N} \df \{1 \ldots N \}$.
For a measurable space, say $(\cS, \fF)$, the set of probability measures over $(\cS, \fF)$ is denoted by $\Delta(\cS, \fF)$ or $\Delta(\cS)$ when the $\sigma$-algebra is clear from the context.
$\E [X]$ and $\Var [X]$ denotes the expectation and variance of a random variable $X$, respectively.
The empty sum is defined to be $0$, e.g., $\sum_{i=j}^k c_i = 0$ if $j > k$.

\looseness=-1
We consider an infinite-horizon discounted MDP defined by $\paren*{\X, \A, \gamma, r, P}$,
where $\X$ denotes the state space,
$\A$ denotes finite action space with size $A$,
$\gamma \in [0, 1)$ denotes the discount factor,
$r: \XA \to [-1, 1]$ denotes the reward function,
and $P: \XA \to \Delta(\X)$ denotes the state-transition probability kernel.
We denote the sets of all bounded Borel-measurable functions over $\X$ and $\XA$ by $\cF_v$ and $\cF_q$, respectively.
Let $H$ be the (effective) time horizon $(1-\gamma)^{-1}$.
For both $\cF_v$ and $\cF_q$, let $\bzero$ and $\bone$ denote functions that output zero and one everywhere, respectively. 
Whether $\bzero$ and $\bone$ are defined in $\cF_v$ or $\cF_q$ shall be clear from the context.
All the scalar operators and inequalities applied to $\cF_v$ and $\cF_q$ should be understood point-wise.

\looseness=-1
With an abuse of notation, let $P$ be an operator from $\cF_q$ to $\cF_v$ such that $(P v) (x, a) = \int v (y) P(dy | x, a)$ for any $v \in \cF_v$.
A policy is a probability kernel over $\A$ conditioned on $\X$.
For any policy $\pi$ and $q \in \cF_q$, let $\pi$ be an operator from $\cF_v$ to $\cF_q$ such that $(\pi q) (x) = \sum_{a \in \A} \pi(a | x) q (x, a)$.
We adopt a shorthand notation, i.e., $P_{\pi} := P \pi$.
We define the Bellman operator $T_{\pi}$ for a policy $\pi$ as $T_{\pi} q := r + \gamma P_{\pi} q$, which has the unique fixed point, i.e., $\qf{\pi}$.
The state-value function $\vf{\pi}$ is defined as $\pi \qf{\pi}$.
An optimal policy $\pi_*$ is a policy such that
$\vf{*} := \vf{\pi_*} \geq \vf{\pi}$ for any policy $\pi$,
where the inequality is point-wise.

\subsection{Tabular MDVI}\label{subsec:tabular MDVI}
To better understand the motivation of our theorems for function approximation, we provide a background on Tabular MDVI of \citet{kozuno2022kl}.

\subsubsection{Tabular MDVI Algorithm}\label{subsubsec:tabular MDVI}

\looseness=-1
For any policies $\pi$ and $\mu$, let $\ent(\pi) \df -\pi \log \pi \in \cF_v$ be the entropy of $\pi$ and $\KL\paren*{\pi\|\mu} \df \pi\log \frac{\pi}{\mu} \in \cF_v$ be the KL divergence of $\pi$ and $\mu$.
For all $(x, a) \in \XA$, the update rule of Tabular MDVI is written as follows:
\begin{equation}\label{eq:MDVI update}
\begin{aligned}
    &q_{k+1} = r + \gamma \widehat{P}_k(M) v_k\;,\\
    \text{ where } \ 
    &v_k = \pi_k q_k - \tau \KL\paren{\pi_k\| \pi_{k-1}} + \kappa \ent\paren{\pi_k}\;,\\
    &\pi_k(a|x) \propto \pi_{k-1}(a|x)^{\alpha} \exp(\beta q_k(x, a)) \;.
\end{aligned}
\end{equation}
Here, we define $\alpha \df \tau / (\tau + \kappa)$ and $\beta \df 1 / (\tau + \kappa)$.
Furthermore, let
$
\widehat{P}_k (M) v_k: (x, a) \mapsto \frac{1}{M}\sum_{m=1}^{M} v_k (y_{k, m, x, a})\,
$
where $(y_{k, m, x, a})_{m=1}^{M}$ are $M \in \N$ samples obtained from the generative model $P(\cdot |x, a)$ at the $k$ th iteration.

Similar to \citet{kozuno2022kl}, we use the idea of the non-stationary policy \citep{scherrer2012use} to provide a tight analysis.
For a sequence of policies $(\pi_k)_{k \in \Z}$, let $P_j^i \df P_{\pi_i} P_{\pi_{i-1}} \cdots P_{\pi_{j+1}} P_{\pi_j}$ for $i \geq j$, otherwise let $P_j^i \df I$.
As a special case with $\pi_k = \pi_*$ for all $k$, let $P_*^i \df (P_{\pi_*})^i$.
Moreover, for a sequence of policies $(\pi_k)_{k=0}^K$,
let $\pi_k'$ be the non-stationary policy that follows $\pi_{k-t}$ at the $t$-th time step until $t=k$,
after which $\pi_0$ is followed.\footnote{The time step index $t$ starts from $0$.}
The value function of such a non-stationary policy is given by $\vf{\pi'_k} = \pi_k T_{\pi_{k-1}} \cdots T_{\pi_1} \qf{\pi_0}$.
While not covered in this work, we anticipate that our main results remain valid for the last policy case, at the expense of the range of valid $\varepsilon$, by extending the analysis of \citet{kozuno2022kl}.

\subsubsection{Techniques to Minimax Optimality}\label{subsubsec:technique to minimax}
The key to achieving the minimax optimality of Tabular MDVI is 
combining the \emph{averaging property} \citep{vieillard2020leverage} and \emph{TV} technique \citep{azar2013minimax}.

\looseness=-1
\paragraph{Averaging Property.}
Let $s_k \df \sum_{j=0}^{k-1} \alpha^j q_{k-j}$ be the moving average of past $q$-functions and $w_k$ be the function $x \mapsto \beta^{-1} \log \sum_{a \in \A} \exp \paren*{ \beta s_k (x, a) }$ over $\X$. 
Then, the update \eqref{eq:MDVI update} can be rewritten as (derivation in \cref{sec:equivalence proof}):
\begin{equation}\label{eq:average transform}
\begin{aligned}
    q_{k+1} = r + \gamma \widehat{P}_k(M) v_k\;,
\end{aligned}
\end{equation}
where $v_k = w_k - \alpha w_{k-1}$, and $\pi_k \parenc{a}{x} \propto \exp \paren*{ \beta s_k (x, a) }$.
To simplify the analysis, we consider the limit of $\tau, \kappa \to 0$ while keeping $\tau / (\tau+\kappa)$ constant.
This limit corresponds to letting $\beta \to\infty$, letting $w_k: x \mapsto \max_{a \in \A} s_k(x, a)$ over $\X$, and having $\pi_k$ be greedy with respect to $s_k$\footnote{
\looseness=-1
Even if $\beta$ is finite, the minimax optimality holds as long as $\beta$ is sufficiently large (\textbf{Remark 1} in \citet{kozuno2022kl}).}.

Intuitively, $s_k$, i.e., the moving average of past $q$-values, averages past errors caused during the update.
\citet{kozuno2022kl} confirmed that this allows Azuma--Hoeffding inequality (\cref{lemma:hoeffding}) to provide a tighter upper bound of $\infnorm{\vf{*} - \vf{\pi_k'}}$ than that in the absence of averaging, where errors appear as a sum of the norms \citep{vieillard2020leverage}.
We provide the pseudocode of Tabular MDVI with \eqref{eq:average transform} in \cref{appendix:missing algorithms}.

\paragraph{Total Variance Technique.}
\looseness=-1
The TV technique is a common theoretical technique used to sharpen the upper bound of $\infnorm{\vf{*} - \vf{\pi_k'}}$ (referred to as the performance bound in this study).
For any $v \in \mathcal{F}_v$, let $\PVar(v)$ be the ``variance'' function.
\begin{align*}
    \PVar(v): (x, a) \mapsto (P v^2) (x, a) - ( P v )^2 (x, a)\;.
\end{align*}
We often write $\sqrt{\PVar(v)}$ as $\sigma(v)$.
For a discounted sum of variances of policy values, the TV technique provides the following bound (the corollary follows from \cref{lemma:total variance}):
\begin{corollary}\label{corollary:TV tmdvi}
Let $\heartsuit_k^{\text{TV}} \df \sum_{j=0}^{k-1} \gamma^j \pi_kP_{k-j}^{k-1} \sigma(\vf{\pi_{k-j}})$ and $\clubsuit_k^{\text{TV}} \df \sum_{j=0}^{k-1} \gamma^j \pi_*P_*^j \sigma(\vf{*})$.
For any $k \in [K]$ in Tabular MDVI, 
$
\heartsuit_k^{\text{TV}} \leq \sqrt{2H^3}\bone
$
and 
$
\clubsuit_k^{\text{TV}} \leq \sqrt{2H^3}\bone\;.
$
\end{corollary}

\citet{kozuno2022kl} used this TV technique to improve the performance bound of Tabular MDVI.
As $\sigma(v_{\pi_{k-j}}) \leq H$ and $\sigma(v_{*}) \leq H$ due to \cref{lemma:popoviciu}, the TV technique provides approximately $\sqrt{H}$ tighter bound than the naive bounds of 
$\heartsuit_k^{\text{TV}} \leq H^2\bone$ and  $\clubsuit_k^{\text{TV}} \leq H^2\bone$.
This leads to $\sqrt{H}$ better performance bound.

\subsection{Linear MDP and G-Optimal Design}\label{subsec:LinearMDP}
We assume access to a good feature representation with which an MDP is linear \citep{jin2020provably}.

\begin{assumption}[Linear MDP]
\label{assumption:linear mdp}
  Suppose an MDP $\cM$ with the state-action space $\XA$.
  We have access to a known feature map $\phi: \XA \to \R^d$ that satisfies the following condition:
  there exist a vector $\psi \in \R^d$ and $d$ (signed) measures
  $\mu \df \paren{\mu_1, \ldots, \mu_d}$ on $\X$ such that
  $P \parenc{\cdot}{x, a} = \phi (x, a)^\top \mu$ for any $\paren{x, a} \in \XA$,
  and $r = \phi^\top \psi$.
  Let $\Phi\df \{ \phi (x, a) : (x, a) \in \XA \} \subset \R^d$ be the set of all feature vectors.
  We assume that $\Phi$ is compact and spans $\R^d$.
\end{assumption}
A crucial property of the linear MDP is that, for any policy $\pi$, $q_\pi$ is always linear in the feature map $\phi$~\citep{jin2020provably}.
The compactness and span assumptions of $\Phi$ are made for the purpose of constructing a G-optimal design later on.

Furthermore, we assume access to a good finite subset of $\XA$ called a core set $\cC$.
The key properties of the core set are that it has a few elements while
$\{\phi(y, b): (y, b) \in \cC\}$ provides a ``good coverage'' of the feature space
in the sense that we describe now.
For a distribution $\rho$ over $\XA$, let $G \in \R^{d \times d}$ and $g (\rho) \in \R$ be defined by
\begin{equation}\label{eq:optimal design}
    \begin{aligned}
     	G &\df \sum_{(y, b) \in \cC} \rho(y, b) \phi (y, b) \phi (y, b)^\top\; \\
        \text{ and }\ g (\rho) &\df \max_{(x, a) \in \XA} \phi (x, a)^\top G^{-1} \phi (x, a)\,,
    \end{aligned}
\end{equation}
respectively. We denote $\rho$ as the design, $G$ as the design matrix underlying $\rho$, and $\cC \df \mathrm{Supp} (\rho)$ as the support of $\rho$, which we denote as the core set of $\rho$.
The problem of finding a design that minimizes $g$ is known as the {\it G-optimal design} problem.
The Kiefer--Wolfowitz (KW) theorem \citep{kiefer1960equivalence} states the optimal design $\rho_*$ must satisfy $g (\rho_*) = d$.
Furthermore, the following theorem shows that there exists a near-optimal design with a small core set for $\Phi$.
The proof is provided in \cref{appendix:proof of KW}.

\begin{theorem}\label{theorem:KW}
Let $\ucC\df4d\log\log (d + 4) + 28$.
For $\Phi$ satisfying \cref{assumption:linear mdp}, there exists a design $\rho$ such that $g(\rho) \leq 2d$ and the core set of $\rho$ has size at most $\ucC$.
\end{theorem}

\section{MDVI with Linear Function Approximation}\label{sec:weighted least squares mdvi}

\looseness=-1
In this section, we provide essential components to extend MDVI from tabular to linear with minimax optimality.
To illustrate how linear MDVI fails or succeeds in attaining minimax optimality, we begin by introducing the general algorithm, called Weighted Least-Squares MDVI (\wlsmdvi).

\subsection{Weighted Least-Squares MDVI Algorithm}\label{subsec:weighted least squares mdvi algo}

Let $q_k(x, a) := \phi^\top(x, a)\theta_k$ be the linearly parameterized value function using the basis function $\theta_k \in \R^d$.
For this $q_k$, the moving average of past $q$-values can be implemented as
\looseness=-1
\begin{equation*}
\begin{aligned}
    s_{k} := \phi^\top {\btheta}_{k}
    \text{ where }\ \btheta_{k} = \theta_{k} + \alpha \btheta_{k-1}\;.
\end{aligned}
\end{equation*}
Using these $q_k$ and $s_k$, let $w_k$, $v_k$, and the policy $\pi_k$ be the same as those of \cref{subsubsec:technique to minimax}. 
Given a bounded positive weighting function $f: \XA \mapsto (0, \infty)$, we learn $\theta_k$ based on weighted least-squares regression.
\begin{equation}\label{eq:weighted lse}
\begin{aligned}
    &\theta_{k}
    =
    \argmin_{\theta \in \R^d} {
        \sum_{(y,b) \in \cCf} {
            {\frac{\rhof(y, b)}{f^2(y, b)}
            \left(
                \phi^\top(y, b) \theta - \hq_{k}(y, b)
            \right)^2 }
        }
    }\;,\\
    &\text{where }\; \hq_{k} (y,b) = \displaystyle r (y,b) + \gamma\widehat{P}_{k-1}(\NP) v_{k-1} (y,b)\;.
\end{aligned}
\end{equation}
Here, $\rhof$ is a design over $\XA$ and $\cCf \df \mathrm{Supp} (\rhof)$ is a core set of $\rhof$.
When $f=\bone$, we recover the vanilla least-squares regression \citep{bellman1963polynomial,munos2005error}, which is a common strategy in practice.
We call this algorithm \wlsmdvi.
The next section presents our novel theoretical tool to provide minimax sample complexity.

\subsection{Weighted Kiefer--Wolfowitz Theorem}\label{subsec:weighted kw}

Let $\theta_k^* \in \R^d$ be the oracle parameter satisfying $\phi^\top \theta_k^* = r + \gamma P v_{k-1}$.
$\theta^*_k$ is ensured to exist by the property of linear MDPs.
To derive the sample complexity, we need a bound of the regression errors outside the core set $\cC_f$, i.e., $\infnorm{\phi^\top(\theta_k - \theta_k^*)}$.
\citet{lattimore2020goodFeature} derived such a bound using \cref{theorem:KW}.

Instead of the vanilla G-optimal design, we consider the following \emph{weighted} design with a bounded positive function $f: \XA \mapsto (0, \infty)$.
For a design $\rho$ over $\XA$, let $\Gf \in \R^{d \times d}$ and $g_f (\rho) \in \R$ be defined by
\begin{equation}\label{eq:weighted optimal design}
\begin{aligned}
	G_f &\df \sum_{(y, b) \in \cC_f} \rho (y, b) \frac{\phi (y, b) \phi (y, b)^\top}{f (y, b)^2}\;,\\
	\text{ and } \ g_f (\rho) &\df \max_{(y, b) \in \XA} \frac{\phi (y, b)^\top G_f^{-1} \phi (y, b)}{f (y, b)^2}\;,
\end{aligned}
\end{equation}
respectively. \cref{eq:weighted optimal design} is the weighted generalization of \cref{eq:optimal design} with $\phi$ scaled by $1 / f$. 
For this weighted optimal design, we derived the weighted KW theorem, which almost immediately follows from \cref{theorem:KW} by considering a weighted feature map ${\phi}_f: (x, a) \mapsto \phi(x, a) / f(x, a)$.

\begin{theorem}[Weighted KW Theorem]\label{theorem:weighted KW}
For $\Phi$ satisfying \cref{assumption:linear mdp}, there exists a design $\rhof$ such that $g_f(\rhof) \leq 2d$ and the core set of $\rhof$ has size at most $\ucC$.
\end{theorem}

Such the design under \cref{assumption:linear mdp} with finite $\X$ can be obtained using the Frank-Wolfe algorithm of \textbf{Lemma 3.9} mentioned in \citet{todd2016minimum}.
We provide the pseudocode of Frank-Wolfe algorithm in \cref{appendix:missing algorithms}.
We assume that we have access to the weighted optimal design in constructing our theory:

\begin{assumption}[Weighted Optimal design]\label{assumption:compute optimal design}
There is an oracle called \texttt{ComputeOptimalDesign} that accepts a bounded positive function $f: \XA \mapsto (0, \infty)$ and returns $\rho_f$, $\cC_f$, and $G_f$ as in \cref{theorem:weighted KW}.
\end{assumption}

Combined with this \texttt{ComputeOptimalDesign}, we provide the pseudocode of \wlsmdvi in \cref{algo:wlsmdvi}.
The weighted KW theorem yields the following bound on the optimal design. The proof can be found in \cref{sec:proof of weighted kw bound}.

\begin{lemma}[Weighted KW Bound]\label{lemma:kw bound} 
    \looseness=-1
    Let $f: \XA \mapsto (0, \infty)$ be a positive function and $z$ be a function defined over $\cC_{f}$. 
    Then, there exists $\rho_{f} \in \Delta \paren*{\XA}$ with a finite support $\cC_{f} \df \mathrm{Supp} (\rho_{f})$ of size less than or equal to $\ucC$ such that
    \begin{equation*}
    \begin{aligned}
    &|\phi^\top\kwsum(f, z)| \leq 
    \sqrt{2d} f \coremaxf{y', b'}{f}\left|\frac{z(y', b')}{f(y', b')}\right|\;,
    \end{aligned}
    \end{equation*}
    where $\kwsum(f, z) := \displaystyle G_{f}^{-1} \sum_{(y, b)\in \cC_{f}} \frac{\rho_{f}(y, b)\phi(y, b) z(y, b)}{f^2(y, b)}$.
\end{lemma}

\subsection{Sample Complexity of \wlsmdvi}\label{subsec:sample complexity of wlsmdvi}

\looseness=-1
\cref{lemma:kw bound} helps derive the sample complexity of \wlsmdvi.
Let $\varepsilon_k$ be the sampling error for $v_{k-1}$ and $E_k$ be its moving average:
\begin{equation*}
\begin{aligned}
    &\varepsilon_k: (x, a) \mapsto \gamma \paren*{\widehat{P}_{k-1}(M) v_{k-1} - P v_{k-1}}(x, a) \; \\
    \text { and }\  &E_k: (x, a) \mapsto  \sum_{j=1}^k \alpha^{k-j} \varepsilon_j (x, a)\;.
\end{aligned}
\end{equation*}
Furthermore, for any non-negative integer $k$, let $A_{\gamma, k} \df \sum_{j=0}^{k-1} \gamma^{k-j} \alpha^j$, $A_k \df \sum_{j=0}^{k-1} \alpha^j$, and $A_\infty \df 1 / (1-\alpha)$.
Then, the performance bound of \wlsmdvi is derived as 
\begin{align}
&|\vf{*} - \vf{\pi_k'}| \leq 
\frac{\sqrt{2d}}{A_\infty} \paren*{\heartsuit_k^{\text{wls}} + \clubsuit_k^{\text{wls}}}
+ \diamondsuit_k \;,\label{eq:weighted lsmdvi bound} \\ 
\text{ where }\ & \heartsuit_k^{\text{wls}} \df \sum_{j=0}^{k-1} \gamma^j\pi_k P_{k-j}^{k-1} \left|\coremaxf{y, b}{f}\frac{E_{k-j}(y, b)}{f(y, b)}\right|f \; \nonumber \\ 
\text{ and }\ &\clubsuit_k^{\text{wls}} \df \sum_{j=0}^{k-1} \gamma^j\pi_* P_*^j \left|\coremaxf{y, b}{f}\frac{E_{k-j}(y, b)}{f(y, b)} \right|f \;. \nonumber
\end{align}
Here, $\diamondsuit_k \df 2H\paren*{\alpha^k + A_{\gamma, k} / A_\infty}\bone$.
The formal lemma can be found in \cref{lemma:non-stationary error propagation}.
This performance bound provides the negative and positive answers to our main question: \emph{Is MDVI minimax optimal with function approximation?}

\subsubsection{Negative Result of $f=\bone$}\label{subsec:negative result}
\looseness=-1
When $f=\bone$, the performance bound becomes incompatible with the TV technique (\cref{corollary:TV tmdvi}), which is necessary for minimax optimality.
In this case, $\heartsuit_k^{\text{wls}} = \clubsuit_k^{\text{wls}} = \sum_{j=0}^{k-1} \gamma^j |\coremax{y, b} E_{k-j} (y, b)|\bone$.
Therefore, even when we relate $E_{k-j}$ to $\sigma(v_{\pi_{k-j}}) \leq H \bone$ or $\sigma(v_{*}) \leq H \bone$ using a Bernstein-type inequality, we only obtain a $H^2$ bound inside the first term of the inequality \eqref{eq:weighted lsmdvi bound}.
This implies that the sample complexity can be sub-optimal, as we need more samples by $\sqrt{H}$ than using the TV technique to obtain a near-optimal policy.

\subsubsection{Positive Result of $f \approx \sigma(v_*)$}\label{subsec:positive result}

When we carefully select the weighting function $f$, the performance bound becomes compatible with the TV technique.
For example, when $f = \sigma(v_*)$ and $E_{k-j}$ is related to $\sigma(v_*)$ using a Bernstein-type inequality, we obtain $\sum_{j=0}^{k-1}\gamma^j \pi_* P_{*}^{j}\sigma(v_{*}) \leq H\sqrt{H} \bone$ inside $\clubsuit_k^{\text{wls}}$ owing to the TV technique.
This helps achieve a performance bound that is approximately $\sqrt{H}$ tighter than the bound of $f=\bone$.

\looseness=-1
Indeed, when $f\approx \sigma(v_*)$, we obtain the following minimax optimal sample complexity of \wlsmdvi:
\begin{theorem}[Sample complexity of \wlsmdvi with $f\approx \sigma(v_*)$, informally]\label{informal theorem: epsilon bound}
    When $\varepsilon \in (0, 1 / H]$, $\alpha=\gamma$, and $\sigma(v_*) \leq f \leq \sigma(v_*) + 2\sqrt{H}\bone$,
    \wlsmdvi outputs a sequence of policies $(\pi_k)_{k=0}^{\NIter}$ such that
    $
        \infnorm{\vf{*} - \vf{\pi'_{\NIter}}}
        \leq
        \varepsilon
    $
    with probability at least $1 - \delta$, using $\widetilde{\cO} \paren*{d^2H^3\varepsilon^{-2}\log(1 / \delta)}$ samples from the generative model.
\end{theorem}
The formal theorem and proof are provided in \cref{sec:update proofs}.
The sample complexity matches the lower bound by \citet{weisz2022confident} up to logarithmic factors. 
This means that \wlsmdvi is nearly minimax optimal as long as $f\approx \sigma(v_*)$ and $\varepsilon$ is sufficiently small.
The remaining challenge is to learn such weighting function.

\begin{algorithm}[t!]
    \caption{WLS-MDVI $(\alpha, f, \NIter, \NP)$}\label{algo:wlsmdvi}
    \begin{algorithmic}
    \STATE {\bfseries Input:} {$\alpha \in [0, 1)$, $f: \XA \to (0, \infty)$, $\NIter \in \N$, $\NP \in \N$.}
    \STATE Initialize $\theta_0 = \btheta_0 = \bzero \in \R^d$, $s_0 = \bzero \in \R^{\XA}$, and $w_{0} = w_{-1} = \bzero \in \R^\X$.
    \STATE $\rhof, \cCf, \Gf \df \kwalgo(f)$.
    \FOR{$k=0$ {\bfseries to} $K-1$}
        \STATE $v_k = w_k - \alpha w_{k-1}$.
        \FOR{{ each state-action pair} $\paren*{y, b} \in \cCf$}
            \STATE Compute $\hq_{k+1}(y, b)$ by \cref{eq:weighted lse}.
        \ENDFOR
        \STATE Compute $\theta_{k+1}$ by \cref{eq:weighted lse}.
        \STATE $\btheta_{k+1} = \theta_{k+1} + \alpha \btheta_k$ and $s_{k+1} = \phi^\top \btheta_{k+1}$.
        \STATE $w_{k+1}(x) = \max_{a\in \A} s_{k+1}(x, a)$ for each $x \in \X$.
    \ENDFOR
    \STATE {\bfseries Return:} {$v_\NIter$ and $(\pi_k)_{k=0}^{\NIter}$ , where $\pi_k$ is greedy policy with respect to $s_k$\;}
    \end{algorithmic}
\end{algorithm}

\section{Variance Weighted Least-Squares MDVI}\label{sec:vwls-mdvi}
\looseness=-1
In this section, we present a simple algorithm for learning the weighting function and introduce our \vwlsmdvi, which combines the weight learning algorithm with \wlsmdvi to achieve minimax optimal sample complexity. 

\subsection{Learning the Weighting Function}\label{subsec:variance estimation}

As stated in \cref{informal theorem: epsilon bound}, the weighting function should be close to $\sigma(v_*)$ by a factor of $\sqrt{H}$.
We accomplish this by learning the weighting function in two steps: learning a $\sqrt{H}$-optimal value function (\cref{subsubsec:sqrt H optimal learning}) and learning the variance of the value function (\cref{subsubsec:variance learning}).

\subsubsection{Learning the $\sqrt{H}$-optimal value function}\label{subsubsec:sqrt H optimal learning}

\cref{informal theorem: sqrt H bound} shows that \wlsmdvi with $f=\bone$ yields a $\sqrt{H}$-optimal value function with sample complexity that is $1 / \varepsilon$ smaller than that of \cref{informal theorem: epsilon bound}.

\begin{theorem}[Sample complexity of \wlsmdvi with $f=\bone$, informally]\label{informal theorem: sqrt H bound}
    When $\varepsilon \in (0, 1 / H]$, $\alpha = \gamma$, and $f = \bone$, \wlsmdvi outputs $v_{\NIter}$ satisfying 
    $
        \infnorm{\vf{*} - v_{\NIter}}
        \leq
        \frac{1}{2}\sqrt{H}
    $
    with probability at least $1 - \delta$, using $\widetilde{\cO} \paren*{d^2H^3\varepsilon^{-1}\log(1 / \delta)}$ samples from the generative model.
\end{theorem}
The formal theorem and proof are provided in \cref{sec:update proofs}.

\begin{algorithm}[t!]
    \caption{VarianceEstimation $(v_\sigma, \NsigmaTwo)$}\label{algo:varianceestimate}
    \begin{algorithmic}
        \STATE {\bfseries Input:} $v_\sigma \in \R^\X$, $\NsigmaTwo \in \N$.
        \STATE $\rho, \cC, G \df \kwalgo(\bone)$.
        \FOR{ each state-action pair $\paren*{x, a} \in \cC$}
            \STATE Compute $\hPVar(x, a)$ by \cref{eq:var sample}.
        \ENDFOR
    \STATE $\omega = G^{-1} \sum_{(x,a) \in \cC} \rho(x,a) \phi(x,a) \hPVar (x,a) $.
    \STATE {\bfseries Return:} $\omega$.
    \end{algorithmic}
\end{algorithm}

\subsubsection{Learning the Variance Function}\label{subsubsec:variance learning}
\looseness=-1
Given a $\sqrt{H}$-optimal value function $v_\sigma$ by \cref{informal theorem: sqrt H bound}, we linearly approximate the variance function as $\PVar_\omega(x, a) \df \phi^T(x, a)\omega$ with $\omega \in \R^d$.
Using $\rho$, $\cC$, and $G$ of the vanilla optimal design, $\omega$ is learned using least-squares estimation.
\begin{equation} \label{eq:var sample}
\begin{aligned}
    &\omega
    =
    G^{-1} \sum_{(x, a) \in \cC} {
        \rho(x, a) \phi(x, a)
        \hPVar(x, a)
    }\;,\; \text{ where } \\
    &\hPVar(x, a) =  \frac{1}{2M_\sigma}\sum_{m=1}^{M_\sigma} \paren[\Big]{v_\sigma (y_{m, x, a}) - v_\sigma(z_{m, x, a})}^2\;.
\end{aligned}
\end{equation}
Here,
$(y_{m, x, a})_{m=1}^{M_\sigma}$ and $(z_{m, x, a})_{m=1}^{M_\sigma}$ denote $M_\sigma$ independent samples from $P(\cdot |x, a)$.
\looseness=-1

The pseudocode of the algorithm is shown in \cref{algo:varianceestimate}.
\cref{informal theorem: evaluate error} shows that with a small number of samples, the learned $\omega$ estimates $\sigma(v_*)$ with $\sqrt{H}$ accuracy.
\looseness=-1

\begin{theorem}[Accuracy of \varianceestimate, informally]\label{informal theorem: evaluate error}
    When $v_\sigma$ satisfies $\infnorm{v_* - {v_\sigma}} \leq \frac{1}{2}\sqrt{H}$, \varianceestimate outputs $\omega$ such that
    $\sigma(v_*) \leq \sqrt{\max(\phi^\top\omega, \bzero)} + \sqrt{H}\bone\leq \sigma(v_*) + 2\sqrt{H} \bone$
    with probability at least $1 - \delta$, using $\widetilde{\cO} \paren*{d^2H^2 \log(1 / \delta)}$ samples from the generative model.
\end{theorem}
The formal theorem and proof are provided in \cref{sec:evaluate proof}.

\subsection{Put Everything Together}\label{subsec:proposal}

\looseness=-1
The proposed \vwlsmdvi algorithm consists of three steps: (1) executing \wlsmdvi with $f=\bone$, (2) performing \varianceestimate, and (3) executing \wlsmdvi again with the output from (2). 
The technical novelty of our theory lies in the ingenuity to run \wlsmdvi twice to use the TV technique, which was not seen in previous studies such as \citet{lattimore2020goodFeature} and \citet{kozuno2022kl}.
By combining these three steps, the \vwlsmdvi obtains an $\epsilon$-optimal policy within minimax optimal sample complexity.
\begin{theorem}[Sample complexity of \vwlsmdvi, informally]\label{informal theorem: sample complexity of vwls mdvi}
    When $\varepsilon \in (0, 1 / H]$ and $\alpha = \gamma$, 
    \vwlsmdvi outputs a sequence of policies $\pi'$ such that
    $
        \infnorm{\vf{*} - \vf{\pi'}}
        \leq
        \varepsilon
    $
    with probability at least $1 - \delta$, using 
    $\widetilde{\cO} \paren*{d^2H^3 \varepsilon^{-2} \log (1 / \delta)}$
    samples from the generative model.
\end{theorem}
The formal theorem and proof are provided in \cref{sec:proof of vwls mdvi}, and the pseudocode of the algorithm is provided in \cref{algo:vwlsmdvi}.
The sample complexity of \vwlsmdvi matches the lower bound described by \citet{weisz2022confident} up to logarithmic factors as long as $\varepsilon$ is sufficiently small. 
This is the first algorithm that achieves nearly minimax sample complexity under inifinite-horizon linear MDPs.

\begin{algorithm}[t!]
    \caption{VWLS-MDVI $(\alpha, \NIter, \NP, \widetilde{\NIter}, \widetilde{\NP}, \NsigmaTwo)$}\label{algo:vwlsmdvi}
    \begin{algorithmic}
        \STATE {\bfseries Input:} {$\alpha \in [0, 1)$, $f : \XA \mapsto (0, \infty)$, $\NIter, \widetilde{\NIter} \in \N$, $\NP, \widetilde{\NP} \in \N$, $\NsigmaTwo \in \N$.}
        \STATE $v_\NIter, \_ = \wlsmdvi(\alpha, \bone, \NIter, \NP)$.
        \STATE $\omega = \varianceestimate(v_\NIter, \NsigmaTwo)$.
        \STATE $\tsigma = \min\paren*{\sqrt{\max(\phi^T\omega, \bzero)} + \sqrt{H}\bone, H\bone}$.
        \STATE $\_, \pi' = \wlsmdvi(\alpha, \tsigma, \widetilde{\NIter}, \widetilde{\NP})$.
        \STATE {\bfseries Return: } {$\pi'$}
    \end{algorithmic}
\end{algorithm}

\section{Deep Variance Weighting}\label{sec:practical algorithm}
\looseness=-1
Motivated on the theoretical observations, we propose a practical algorithm to re-weight the least-squares loss of value-based deep RL algorithms, called Deep Variance Weighting (\dvw).

\subsection{Weighted Loss Function for the $Q$-Network}\label{subsec:weighted Q loss}

As Munchausen DQN (M-DQN, \citet{vieillard2020munchausen}) is the effective deep extension of MDVI, we use it as our base algorithm to apply \dvw.
However, the proposed \dvw can be potentially applied to any DQN-like algorithms\footnote{
\citet{van2019use} stated that DQN may not be a completely model-free algorithm, which could potentially conflict with the model-free structure of \vwlsmdvi. 
Nevertheless, we do not consider such discrepancies from our theory to be problematic, as the primary aim of \dvw is to improve the popular algorithms rather than to validate the theoretical analysis.}.
We provide the pseudocode for the general case in \cref{algo:DVW general} and for online RL in \cref{appendix:missing algorithms}.

Similar to M-DQN, let $q_\theta$ be the $q$-network and $q_{\btheta}$ be its target $q$-network with parameters $\theta$ and $\btheta$, respectively. 
In this section, $\px$ denotes the next state sampled from $P(\cdot | x, a)$. $\widehat{\E}_\cB$ denotes the expectation over using samples $(x, a, r, x') \in \XA \times \R \times \X$ from some dataset $\cB$.
With a weighting function $f: \XA \to (0, \infty)$, we consider the following weighted version of M-DQN's loss function:
\begin{equation}\label{eq:weighted M-DQN loss}
    \begin{aligned}
        &\cL(\theta) = \widehat{\E}_\cB \left[\paren[\Big]{\frac{r_{\btheta}(x, a) + \gamma v_{\btheta}(\px) - q_\theta(x, a)}{f(x, a)}}^2\right]\,,
    \end{aligned}
\end{equation}
where $r_{\btheta} = r + \tau \log \pi_{\btheta}$, $\pi_{\btheta}(a | x) \propto \exp \paren*{\beta q_{\btheta}(x, a)}$, and $v_{\btheta}(\px) = \sum_{\pa \in \A} \pi_{\btheta}(\pa | \px) \paren*{q_{\btheta} (\px, \pa) - \frac{1}{\beta}\log \pi_{\btheta}(\pa|\px)}$.
\Cref{eq:weighted M-DQN loss} is equivalent to M-DQN when $f = \bone$.
Furthermore, when $\tau=\kappa=0$, we assume that $\tau \log \pi_{\btheta} = \frac{1}{\beta} \log \pi_{\btheta} = \bzero$ and $\sum_{\pa \in \A} \pi_{\btheta}(\pa | \px) q_{\btheta} (\px, \pa) = \max_{\pa \in \A} q_{\btheta} (\px, \pa)$.
This allows us to generalize \cref{eq:weighted M-DQN loss} to DQN's loss when $f=\bone$ and $\tau=\kappa=0$.

\looseness=-1
We update $\theta$ by stochastic gradient descent (SGD) with respect to $\cL(\theta)$. 
We replace $\btheta$ with $\theta$ for every $F$ iteration.

\begin{algorithm}[tb]
    \caption{DVW for (Munchausen-)DQN}\label{algo:DVW general}
    \begin{algorithmic}
    \STATE {\bfseries Input:} {$\theta$, $\omega$, $K \in \N$, $F \in \N$, $\cB$}
    \STATE Set $\btheta = \htheta = \theta$ and $\bomega = \omega$. $\eta = 1$.
    \FOR{$k=0$ {\bfseries to} $K-1$}
        \STATE Sample a random batch of transition $B_k \in \cB$. 
        \STATE On $B_k$, update $\omega$ by $\cL(\omega)$ of \eqref{eq:variance loss}.
        \STATE On $B_k$, update $\eta$ by $\cL(\eta)$ of \eqref{eq:weight scaler loss}.
        \STATE On $B_k$ and $f^{\text{DVW}}$ of \eqref{eq:practical weight function}, update $\theta$ by $\cL(\theta)$ of \eqref{eq:weighted M-DQN loss}.
        \IF{ $k \mod F = 0$}
            \STATE $\htheta \leftarrow \btheta$, $\btheta \leftarrow \theta$, $\bomega \leftarrow \omega$.
        \ENDIF
    \ENDFOR
    \STATE {\bfseries Return:} {A greedy policy with respect to $q_\theta$\;}
    \end{algorithmic}
\end{algorithm}

\subsection{Loss Function for the Variance Network}

\looseness=-1
Let $\PVar_\omega$ be the variance network with parameter $\omega$.
We also define $q_{\htheta}$ as the preceding target $q$-network to $q_{\btheta}$.
The parameter $\htheta$ of $q_{\htheta}$ is replaced with $\btheta$ for every $F$ iteration.

For sufficiently large $F$, we expect that $q_{\btheta}$ well approximates $q_{\btheta} \approx r_{\htheta} + \gamma P v_{\htheta}$.
Using this approximation and based on \varianceestimate, we construct the loss function for the variance network as
\begin{equation}\label{eq:variance loss}
    \begin{aligned}
        &\cL(\omega) = \widehat{\E}_\cB \left[ h\paren[\Big]{y^2 - \PVar_\omega (x, a)}\right]\,,
    \end{aligned}
\end{equation}
where $y = r_{\htheta}(x, a) + \gamma v_{\htheta}(\px) - q_{\btheta}(x, a)$.
Here, we use the Huber loss function $h$: for $x\in \R$, $h(x) = x^2$ if $x < 1$; otherwise, $h(x)=|x|$. 
This is to alleviate the issue with large $y^2$ in contrast to the L2 loss. 
We update $\omega$ by iterating SGD with respect to $\cL(\omega)$.

\subsection{Weighting Function Design}

\looseness=-1
According to \vwlsmdvi, the weighting function $f$ should be inversely proportional to the learned variance function with lower and upper thresholds.
Moreover, uniformly scaling $f$ with some constant variables does not affect the solution of weighted regression.
Therefore, we design the weighting function $f^{\text{DVW}}$ such that 
\begin{equation}\label{eq:practical weight function}
    \frac{1}{f^{\text{ DVW }}(x, a)^2} = \max\paren*{\frac{\eta}{\PVar_{\bomega}(x, a) + \overline{c}_f}, \underline{c}_f}\,,
\end{equation}
where $\eta \in (0, \infty)$ denotes a scaling constant, $\underline{c}_f$ and $\overline{c}_f \in (0, \infty)$ denote constants for the lower and upper thresholds, respectively.
Here, we use the frozen parameter $\bomega$, which is replaced with $\omega$ for every $F \in \N$ iteration, as we should use the weight learned via \cref{eq:variance loss}.

\looseness=-1
To further stabilize training, we automatically adjust $\eta$ so that $\widehat{\E}_\cB[f^{\text{ DVW }}(x, a)^{-2}] \approx 1$.
We adjust $\eta$ by SGD with respect to the following loss function: 
\begin{equation}\label{eq:weight scaler loss}
\cL(\eta) = \paren*{\widehat{\E}_\cB\brack*{
\frac{\eta}{\PVar_{\bomega}(x, a) + \overline{c}_f}}
 - 1}^2\,,
\end{equation}
where the term ${\eta} / \paren*{\PVar_{\bomega}(x, a) + \overline{c}_f}$ is the value inside the $\max$ of \cref{eq:practical weight function}. The $\max$ is removed to avoid zero gradient.
While the target value can be set to a value other than $1$, doing so would be equivalent to adjusting the learning rate in the standard SGD. 
To avoid introducing an unnecessary hyperparameter, we have fixed the target value to $1$.

\begin{figure}[tbh]
\vskip 0.2in
\centering
\includegraphics[width=\linewidth]{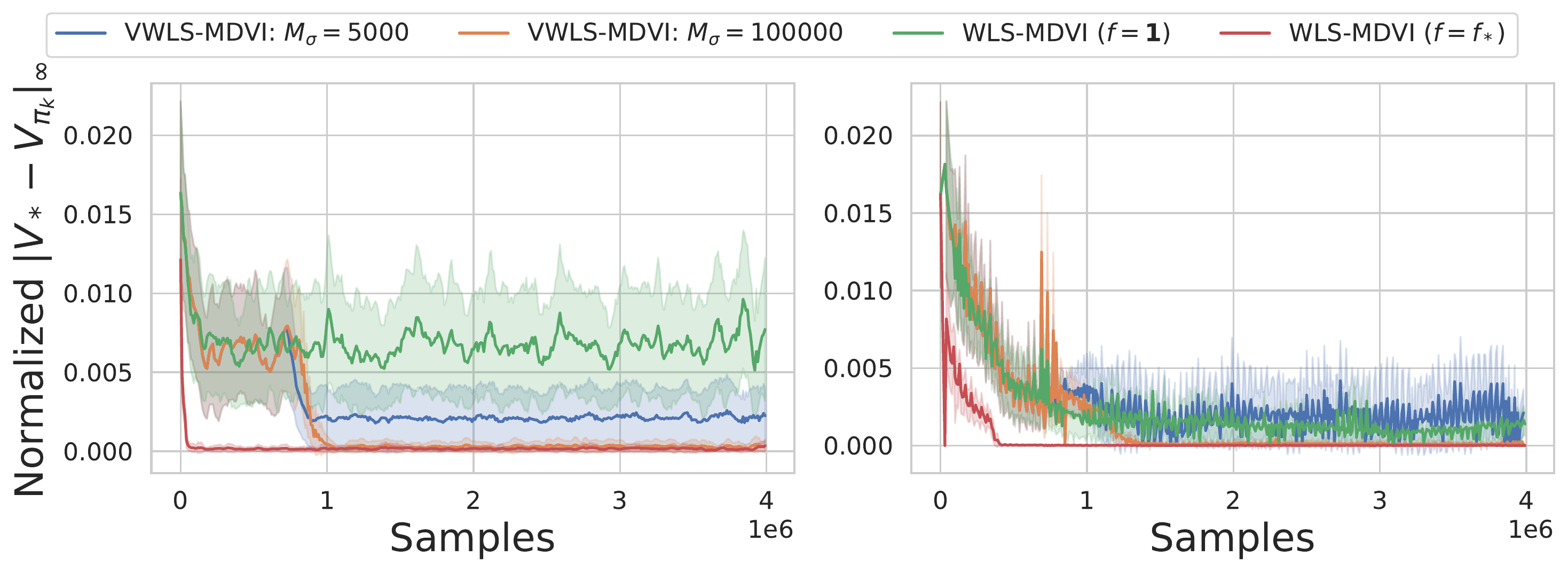}
\caption{Comparison of the normalized gaps. \vwlsmdvi switches to the second run of \wlsmdvi around $10^6$ samples. \textbf{Left:} $M=\widetilde{M}=100$ and \textbf{Right:} $M=\widetilde{M}=1000$.}
\label{fig:hard MDP experiment}
\vskip -0.2in
\end{figure}

\section{Experiments}\label{sec:experiments}

This section reports the experimental sample efficiency of the proposed \vwlsmdvi and deep RL with \dvw.

\looseness=-1
\subsection{Linear MDP Experiments}\label{subsec:linear mdp experiment}
To empirically validate the negative and positive claims made in \cref{subsec:sample complexity of wlsmdvi} and demonstrate the sample efficiency of \vwlsmdvi, we compare \vwlsmdvi to \wlsmdvi with two different weighting functions: $f=\bone$ and $f = f_*$, where $f_*\df \min\paren{\sigma(v_*) + \sqrt{H}\bone, H\bone}$ is the oracle weighting function from \cref{informal theorem: epsilon bound}.
The evaluation is conducted on randomly generated hard linear MDPs that are based on {\bf Theorem H.3} in \citet{weisz2022confident}. 
For simplicity, all algorithms use the last policy for evaluation. 
Specifically, for the $k \in [K]$ th iteration to update the parameter $\theta$, we report the normalized optimality gap $\infnorm{v_* - v_{\pi_k}} / \infnorm{v_*}$ in terms of the total number of samples used so far.
We normalize the gap by $\infnorm{v_*}$ as the maximum gap can vary depending on the MDPs.

\looseness=-1
\cref{fig:hard MDP experiment} compares algorithms under $M=\widetilde{M}=100$ (Left) and $M=\widetilde{M}=1000$ (Right).
The results are averaged over $300$ random MDPs. 
For WLS-MDVI ($f=1$), increasing $M$ from 100 to 1000 results in a smaller optimality gap, which is expected due to the increase in the number of samples. 
On the other hand, WLS-MDVI ($f=f_*$) achieves a gap very close to $0$ even with $M=100$, demonstrating the effectiveness of variance-weighted regression in improving sample efficiency, as claimed in \cref{subsec:sample complexity of wlsmdvi}.
Similarly, it is observed that the VWLS-MDVI ($M_\sigma=100000$) achieves a smaller gap with much fewer samples than that of WLS-MDVI. 
However, the gap of VWLS-MDVI ($M_\sigma=5000$) does not reach that of $f=f_*$. This suggests that the accuracy of the \varianceestimate is important for guaranteeing good performance.
Further experimental details are provided in \cref{subsec:linear mdp details}.

\begin{figure}[tb]
\vskip 0.2in
\centering
\includegraphics[width=\linewidth]{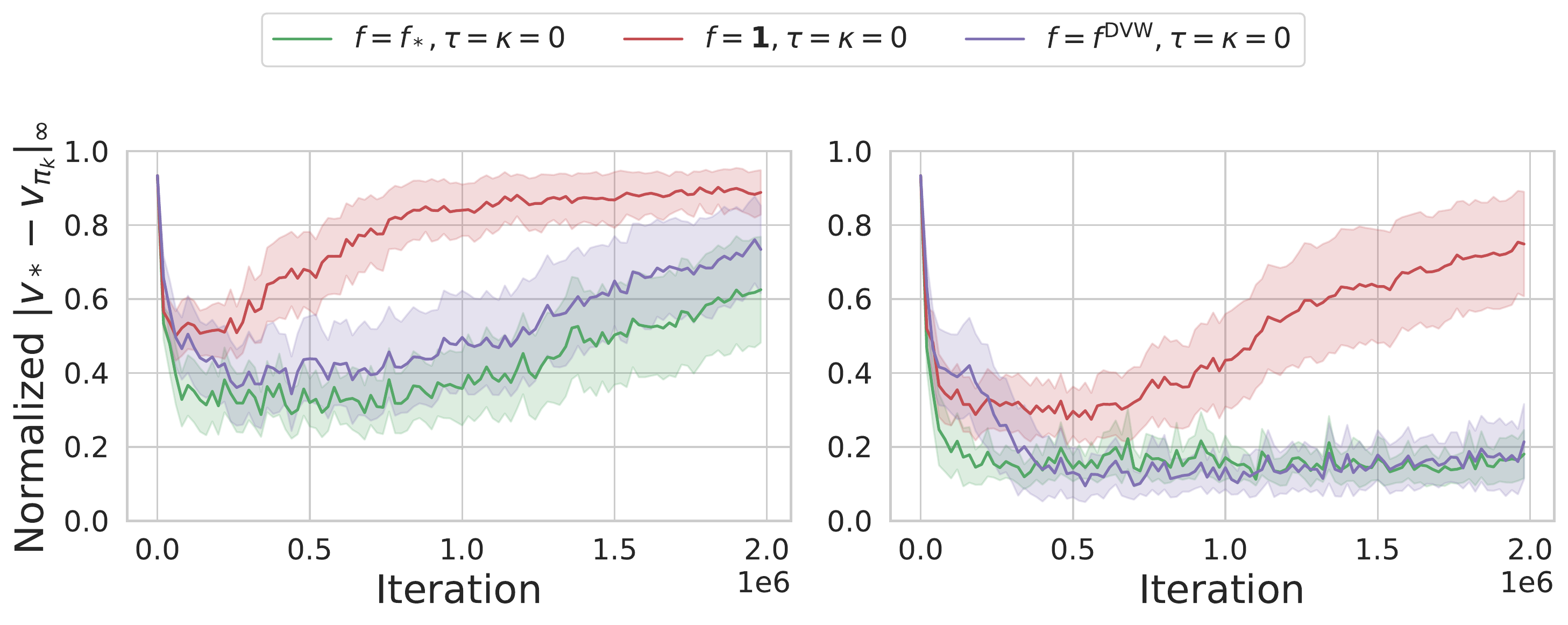}
\includegraphics[width=\linewidth]{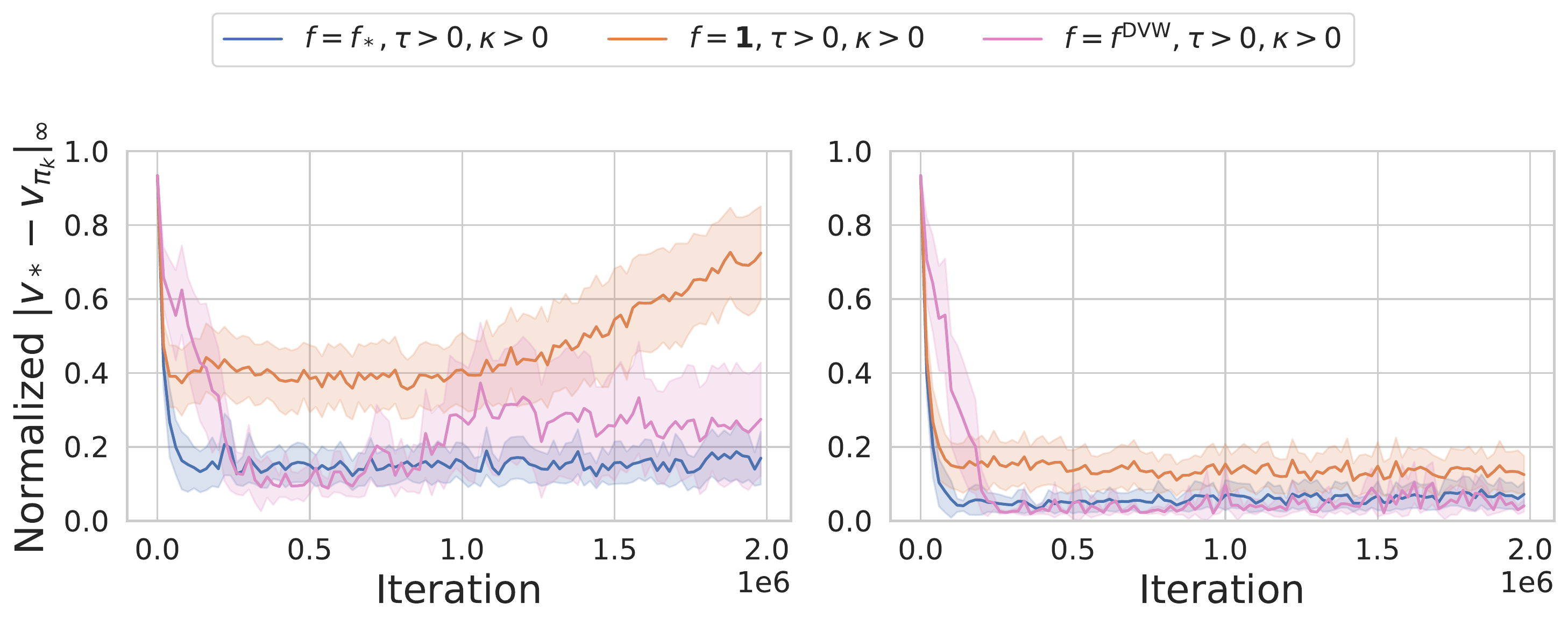}
\caption{Comparison of the normalized gaps.  \textbf{Top Row:} $\tau=\kappa=0$ and \textbf{Bottom Row:} $\tau> 0, \kappa > 0$. \textbf{Left Column:} $M=3$ and \textbf{Right Column:} $M=10$.}
\label{fig:full-state-action gridworld}
\vskip -0.2in
\end{figure}

\begin{figure*}[tb]
\vskip 0.2in
\centering
\includegraphics[width=0.3\linewidth]{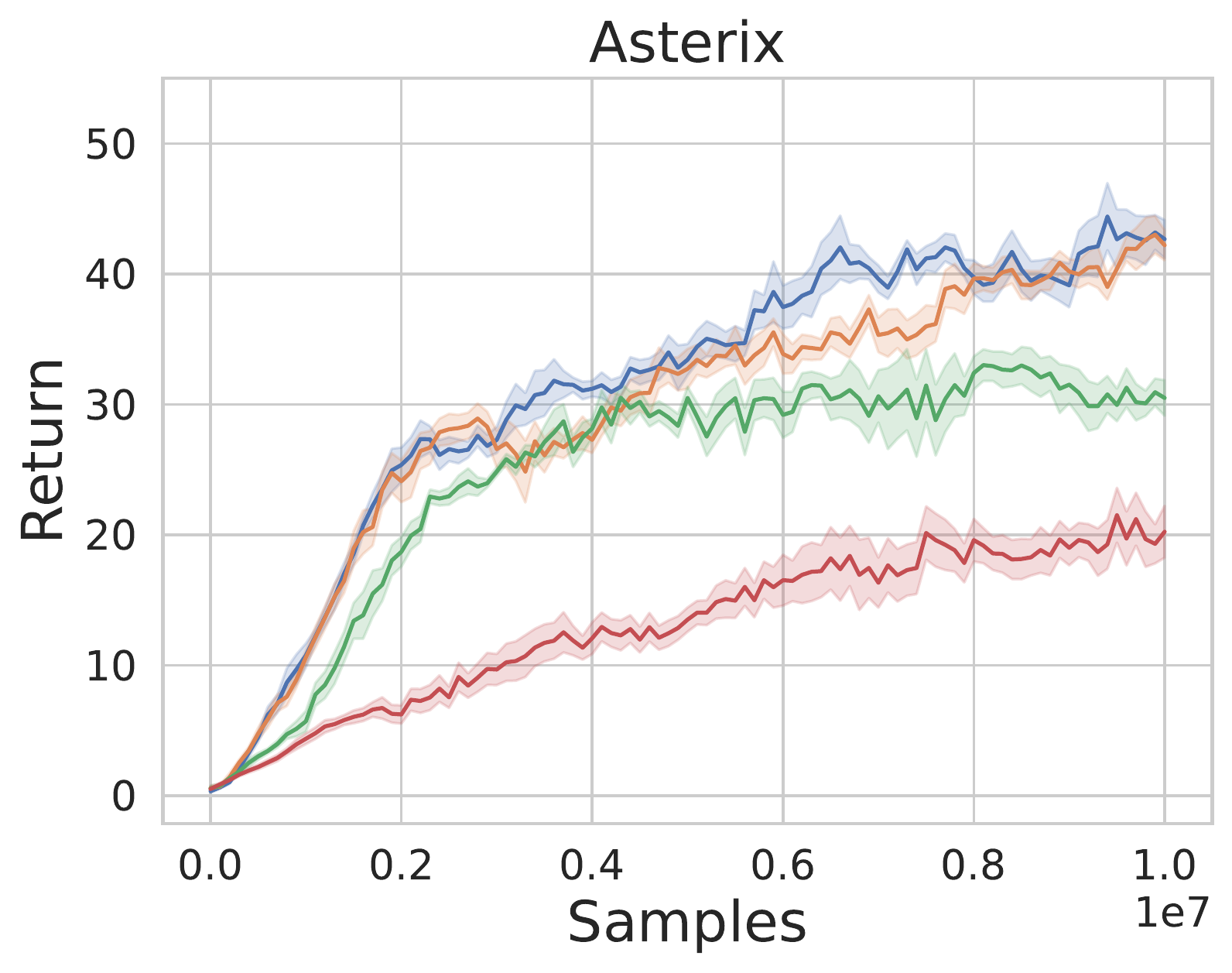}\quad\quad
\includegraphics[width=0.3\linewidth]{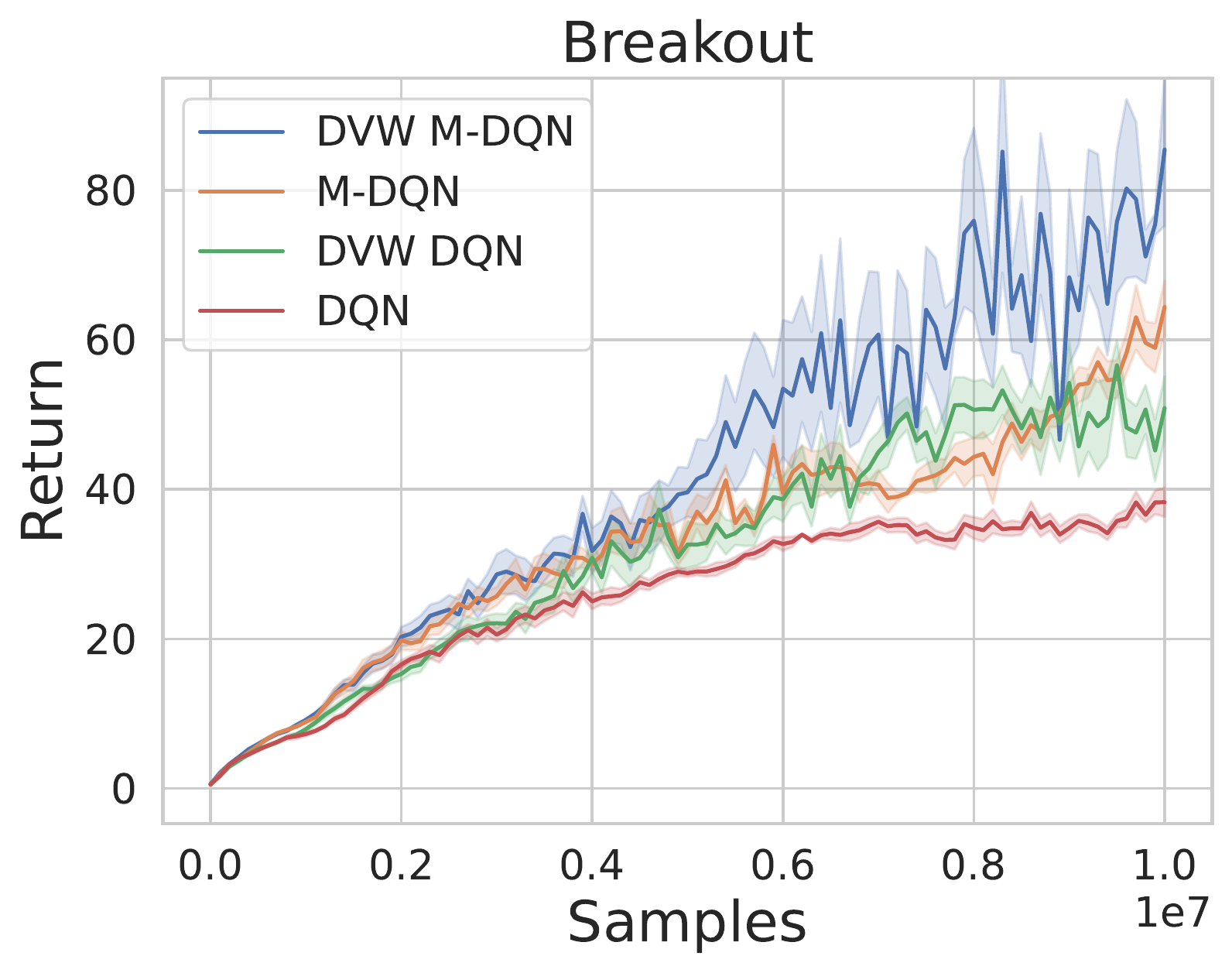}\quad\quad
\includegraphics[width=0.3\linewidth]{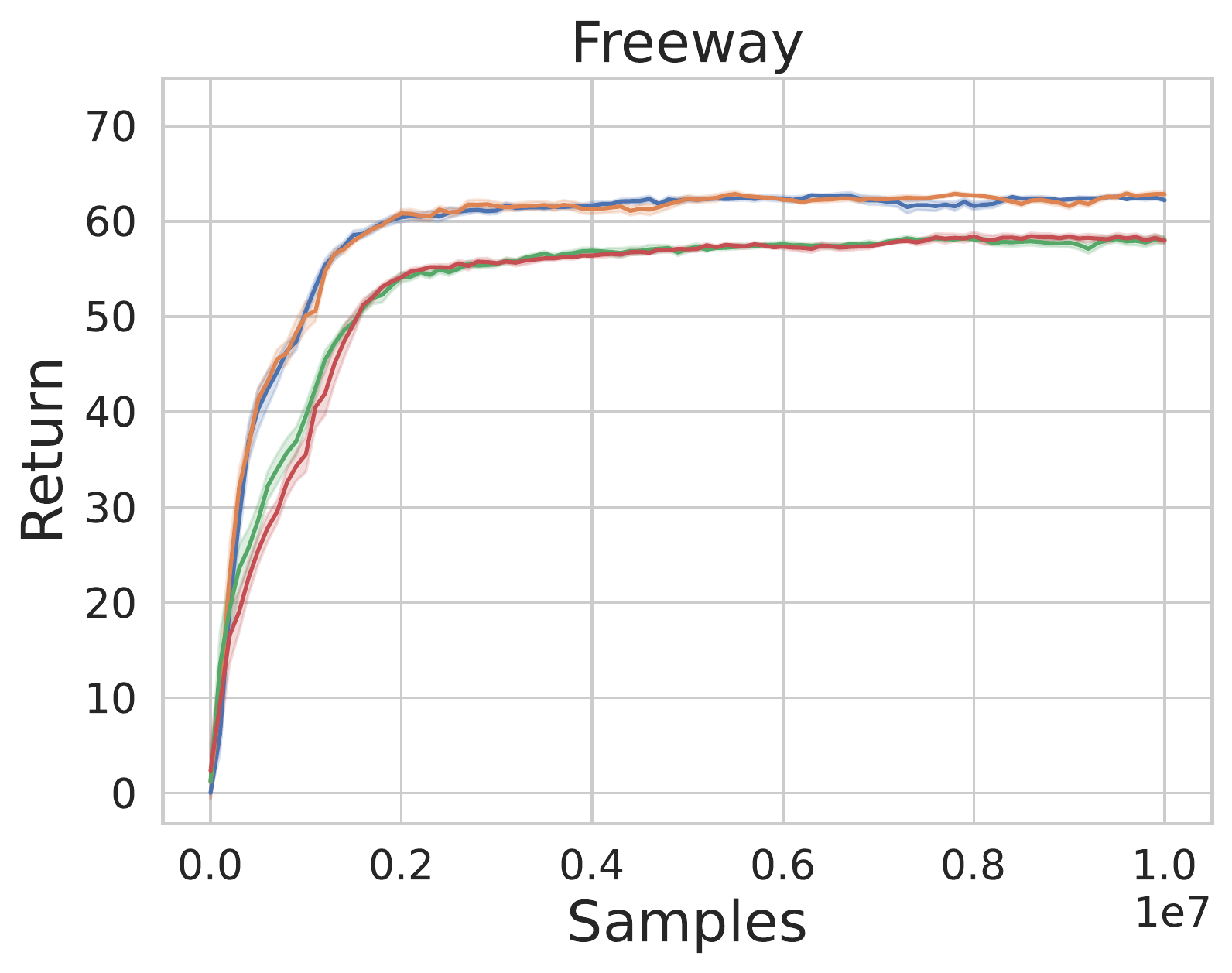}
\medskip
\includegraphics[width=0.3\linewidth]{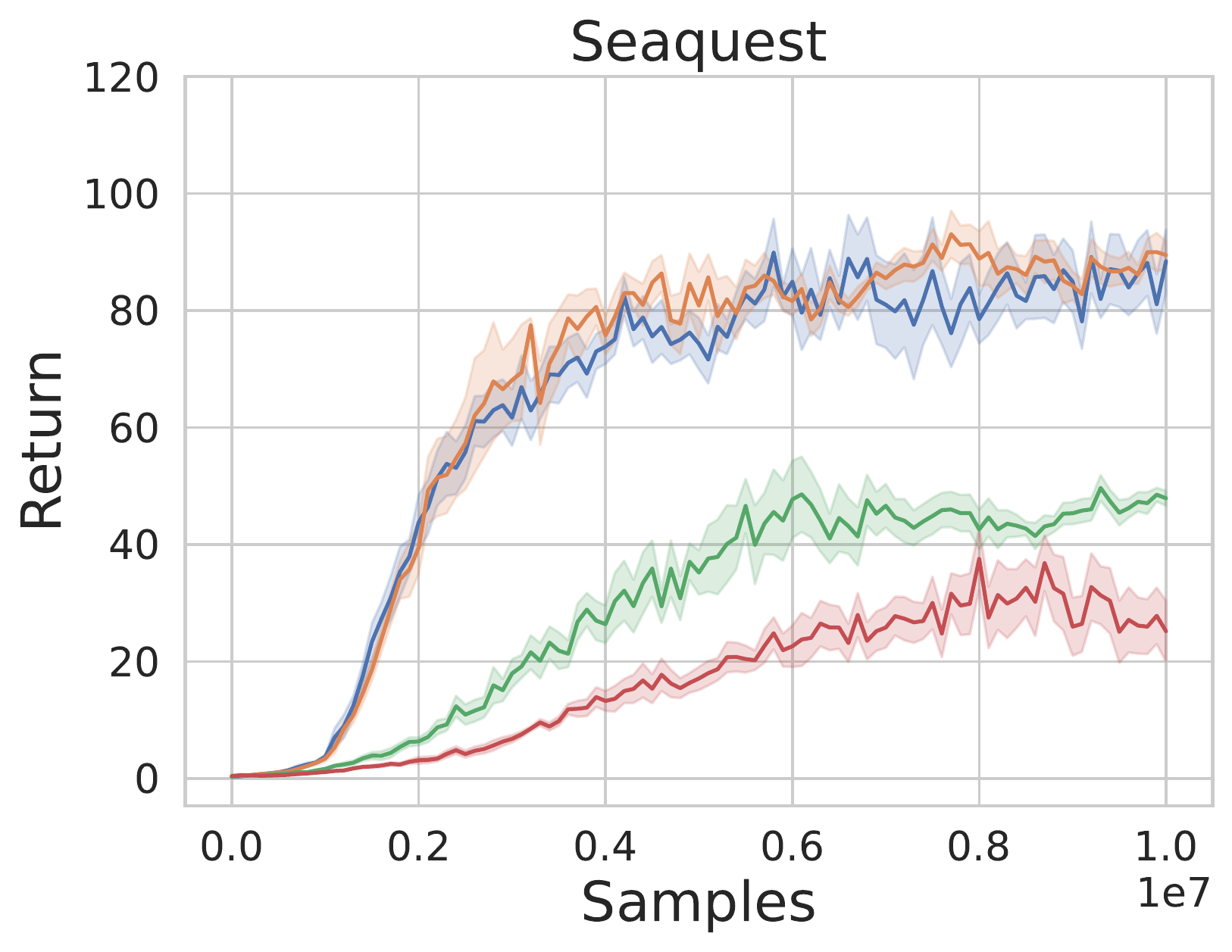}\quad\quad
\includegraphics[width=0.3\linewidth]{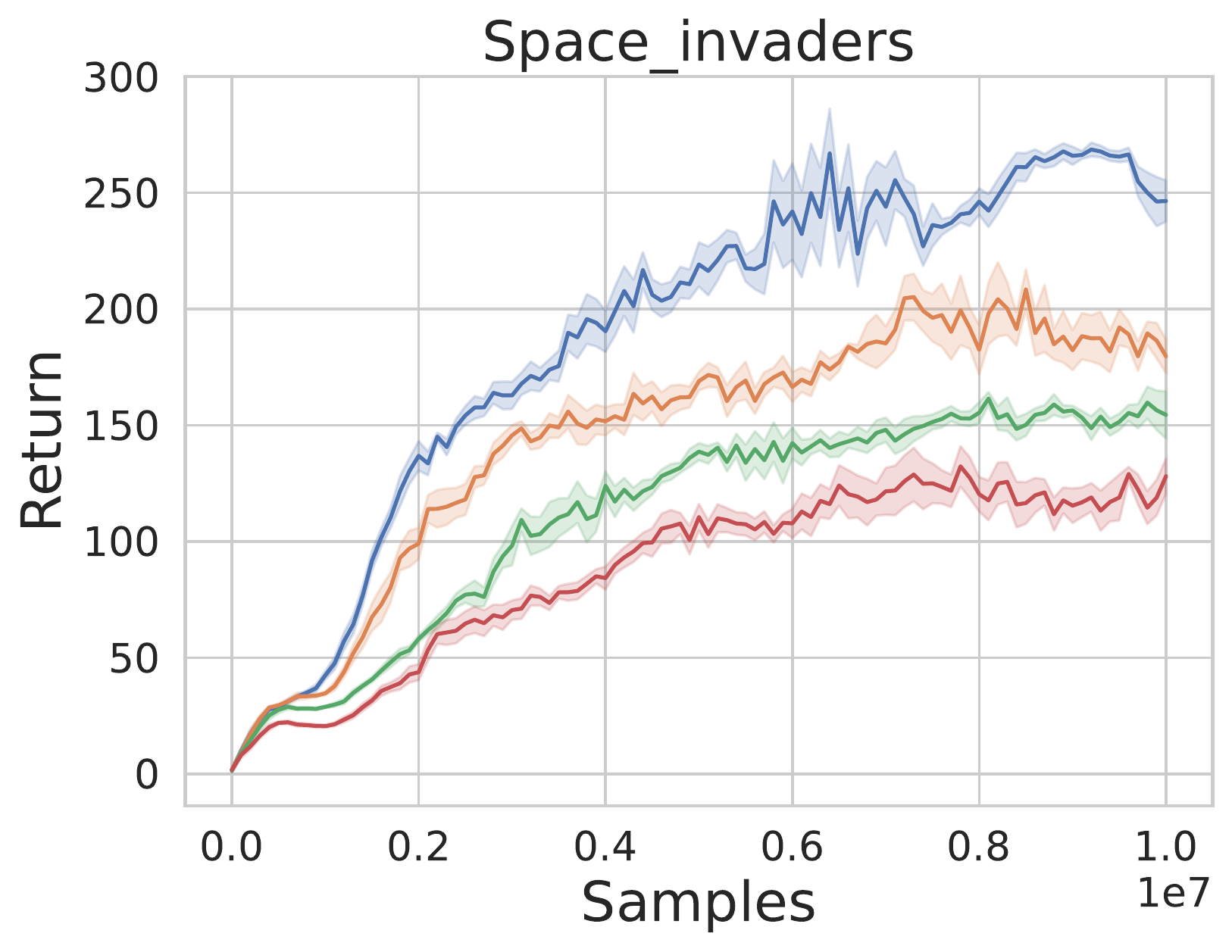}
\caption{Comparison of returns on MinAtar benchmarks. We report the return of the greedy policy with respect to $q_\theta$ for each algorithm.}
\label{fig:minatar}
\vskip -0.2in
\end{figure*}

\subsection{Deep RL Experiments}
We perform two deep RL experiments to evaluate the effectiveness of \dvw: one to compare \dvw with the oracle weighting function of \cref{informal theorem: epsilon bound}, and another to demonstrate the effectiveness of \dvw to online deep RL.
The details of the experiments are provided in \cref{subsec:gridworld experiment}.

\subsubsection{Comparison of $f = f^{\text{DVW}}$ with $f \approx \sigma(v_*)$}\label{subsec:gridworld experiment}

\looseness=-1
To investigate the effectiveness of \dvw, %
we evaluate the behavior of M-DQN with weighted regression \eqref{eq:weighted M-DQN loss} under three weighting functions: the oracle weighting ($f=f_*$), the uniform weighting ($f=\bone$), and the \dvw weighting ($f=f^{\text{DVW}}$).
Furthermore, for the purpose of ablation study, we compare the algorithms with and without regularization ($\tau > 0, \kappa > 0$ vs $\tau=0, \kappa=0$).
To remove the challenge of exploration for didactic purposes, we use a dataset $\cB$, which is constructed by pairs of $(x, a, r, x')$ for the entire state-action space with $M$ next-state samples. In other words, $\cB$ is a dataset of size $MXA$.

\looseness=-1
We evaluate them in randomly generated environments where we can compute oracle values. 
Specifically, we use a modified version of the gridworld environment described by \citet{fu2019diagnosing}.
For the $k$ th iteration to update the $q$-networks, we evaluate the normalized optimality gap averaged over $20$ environments and $3$ random seeds for each.

\looseness=-1
\cref{fig:full-state-action gridworld} compares algorithms under $M=3$ (Left Column) and $M=10$ (Right Column).
In both cases, \dvw consistently achieves a smaller gap compared to $f=\bone$, and moreover, the gap of \dvw is comparable to that of the oracle weighting $f=f_*$. 
In addition, the gap is smaller when $\tau > 0, \kappa > 0$ compared to when $\tau = \kappa = 0$. 
It can be inferred that \dvw weighting and KL-entropy regularization contribute to improving sample efficiency, and that performance is significantly improved when both are present. 

\subsubsection{DVW for Online RL}\label{subsec:minatar experiment}

\looseness=-1
We evaluate the effectiveness of DVW using a set of the challenging benchmarks for online RL.
Similar to \cref{subsec:gridworld experiment}, we evaluate four algorithms that varied with and without \dvw (DVW vs N/A), and with and without regularization (M-DQN vs DQN).
We compare their performance on the MinAtar environment \citep{young2019minatar}, which possesses high-dimensional features and more challenging exploration than \cref{subsec:gridworld experiment}, while facilitating fast training.
For a fair comparison, the algorithms use the same network architecture and same epsilon-greedy exploration strategy.
Each algorithm is executed five times with different random seeds for each environment.

\looseness=-1
\cref{fig:minatar} shows the average returns of the algorithms.
We observe that \dvw improves the performance of M-DQN and DQN in almost all the environments.
Although our theory does not cover online RL, this experiment suggests that the extension of \dvw to wider problem settings is effective.

\section{Conclusion}\label{sec:conclusion}

In this study, we proposed both a theoretical algorithm, i.e., \vwlsmdvi, and a practical algorithm, i.e., \dvw. 
\vwlsmdvi achieved the first-ever nearly minimax optimal sample complexity in infinite-horizon Linear MDPs by utilizing the combination of KL-entropy regularization and variance-weighted regression.
We extended our theoretical observations and developed the \dvw algorithm, which re-weights the least-squares loss of value-based RL algorithms using the estimated variance of the value function.
Our experiments demonstrated that \dvw effectively helps improve the performance of value-based deep RL algorithms.

\section*{Acknowledgements}
Csaba Szepesv\'ari gratefully acknowledges the funding from Natural Sciences and Engineering Research Council (NSERC) of Canada and the Canada CIFAR AI Chairs Program for Amii.

\bibliography{refs}
\bibliographystyle{icml2023}

\newpage
\appendix
\onecolumn
\section*{Contents}
{%
\begin{itemize}[topsep=2pt,partopsep=2pt,itemsep=4pt,parsep=4pt]
    \item \cref{app:notations} lists notations for the theoretical analysis and their meaning;
    \item \cref{sec:equivalence proof} proves the MDVI transformation stated in \cref{subsubsec:technique to minimax};
    \item \cref{sec:auxiliary} provides auxiliary lemmas necessary for proofs;
    \item \cref{app:total variance} provides the formal theorem and the proofs of the total variance technique;
    \item \cref{appendix:proof of KW} provides the proof of the existence of a small core set for a compact set (\cref{theorem:KW});
    \item \cref{sec:proof of weighted kw bound} provides the proof of the weighted Kiefer Wolfowitz bound (\cref{lemma:kw bound});
    \item \cref{sec:update proofs} provides the formal theorems and the proofs for sample complexity of \wlsmdvi (\cref{informal theorem: epsilon bound} and \cref{informal theorem: sqrt H bound})
    \item \cref{sec:evaluate proof} provides the formal theorem and the proof for sample complexity of \varianceestimate (\cref{informal theorem: evaluate error})
    \item \cref{sec:proof of vwls mdvi} provides the formal theorem and the proof for sample complexity of \vwlsmdvi (\cref{informal theorem: sample complexity of vwls mdvi});
    \item \cref{appendix:missing algorithms} provides the pseudocode of algorithms missed in the main pages;
    \item \cref{sec:experiment details} provides the details of experiments stated in \cref{sec:experiments}.
\end{itemize}
}

\section{Notations for Theoretical Analysis}
\label{app:notations}

\begin{table}[H]
	\centering
	\caption{Table of Notations for Theoretical Analysis}
	\begin{tabular}{@{}l|l@{}}
		\toprule
		\thead{Notation} & \thead{Meaning} \\ \midrule
	$\A$, $\X$ & action space of size $A$,  state space\\
	$\gamma$, $H$ & discount factor in $[0, 1)$ and effective horizon $H \df 1 / (1 - \gamma)$\\
	$\phi$, $d$ & feature map of a linear MDP and its dimension (\cref{assumption:linear mdp})\\
	$r$ & reward function bounded by $1$\\
	$P$, $P_\pi$ & transition kernel, $P_\pi \df P \pi$\\
        $\cF_v$, $\cF_q$ & the sets of all bounded Borel-measurable functions over $\X$ and $\XA$, respectively \\
        \midrule
	$\pi_k'$ & a non-stationary policy that follows $\pi_k, \pi_{k-1}, \ldots$ sequentially (\cref{subsec:tabular MDVI})\\
	$P^i_j$, $P_*^i$ & $P^i_j \df P_{\pi_i} P_{\pi_{i-1}} \cdots P_{\pi_{j+1}} P_{\pi_j}$ and $P_*^i \df (P_{\pi_*})^i$ (\cref{subsec:tabular MDVI})\\
        $T_{\pi}$, $T^i_j$ & Bellman operator for a policy $\pi$, $T^i_j \df T_{\pi_i} T_{\pi_{i-1}} \cdots T_{\pi_{j+1}} T_{\pi_j}$ (\cref{subsec:tabular MDVI})\\
        $\vf{\pi'_k}$ & value function of $\pi'_k$; $\vf{\pi'_k} = \pi_k T_{\pi_{k-1}} \cdots T_{\pi_1} \qf{\pi_0}$. \\
	\midrule
	$\varepsilon$, $\delta$ & admissible suboptimality, admissible failure probability\\
	\midrule
	$\varepsilon_k$ & $\varepsilon_k: (x, a) \mapsto \gamma \widehat{P}_{k-1}(M) v_{k-1} (x, a) - \gamma P v_{k-1} (x, a)$ in \wlsmdvi \\
	$E_k$ & $E_k : (x, a) \mapsto \sum_{j=1}^k \alpha^{k-j} \varepsilon_j (x, a)$\\
        \midrule
        $f$ & a bounded positive weighting function over $\XA$ \\
        $\rhof$ & a design over $\XA$ \\
        $\cCf$, $\ucC$ & core set, $\ucC\df4d\log\log (d + 4) + 28$ (\cref{subsec:LinearMDP})\\
        $\Gf$  & design matrix with respect to $f$, $\phi$, and $\rhof$ (\cref{theorem:weighted KW})\\
        $\uf, \lf$ & $\uf \df \max_{(x, a)\in \XA} f(x, a)$, $\lf \df \min_{(x, a)\in \XA} f(x, a)$ (\cref{sec:update proofs})\\
	$\kwsum(f_1, f_2)$ & solution of a weighted least-squares estimation (\cref{lemma:kw bound})\\
	\midrule
        $\widehat{P}_k$ & $\widehat{P}_k (M) v_k: (x, a) \mapsto \frac{1}{M}\sum_{m=1}^{M} v_k (y_{k, m, x, a})$ (\cref{subsec:tabular MDVI})\\
        $\theta_k$, $\btheta_k$ & parameter of $q_k$ in \wlsmdvi ($q_k = \phi^\top  \theta_k$), $\btheta_{k} = \theta_{k} + \alpha \btheta_{k-1} = \sum^{k}_{j=0} \alpha^{k-j}\theta_j$ \\
        $\theta_k^*$, $\btheta_k^*$ & parameter that satisfies  $\phi^\top  \theta^*_{k} = r + \gamma Pv_{k-1}$, ${\btheta}^*_k = \sum^{k}_{j=1} \alpha^{k-j}\theta^*_j$ (\cref{sec:update proofs})\\
 	$s_k$, $v_k$, $w_k$ & $s_k \df  \phi^\top (x, a) {\btheta}_{k}$, $v_k \df w_k - \alpha w_{k-1}$, $w_k (x) \df \max_{a \in \A} s_k (x, a)$ (\wlsmdvi) \\
        $\alpha, \beta$ & weights for MDVI updates $\alpha \df \tau / (\tau + \kappa)$, $\beta \df 1 / (\tau + \kappa)$ (\cref{subsec:tabular MDVI})\\
 	$K, M$ & the number of iterations and the number of samples from the generative model in \wlsmdvi \\
        \midrule
        $\hPVar$ & $\hPVar(x, a) =  \frac{1}{2M}\sum_{m=1}^{M} \paren[\Big]{v_\sigma (y_{m, x, a}) - v_\sigma(z_{m, x, a})}^2$ (\cref{subsec:variance estimation}) \\
        $M_\sigma$ & number of samples from the generative model in \varianceestimate \\
        $v_\sigma$ & the input value function to \varianceestimate \\
        $\omega$ & parameter for \varianceestimate \\
        \midrule
	$A_{k}, A_\infty, A_{\gamma, k}$ & $\sum_{j=0}^{k-1} \alpha^j$, $\sum_{j=0}^{\infty} \alpha^j$, $\sum_{j=0}^{k-1} \alpha^j \gamma^{k-j}$\\
	$\bF_{k, m}$ & $\sigma$-algebra in the filtration for \wlsmdvi (\cref{sec:update proofs})\\
	$\bF_{m}$ & $\sigma$-algebra in the filtration for \varianceestimate (\cref{sec:evaluate proof})\\
        $\iota_1$, $\iota_{2, n}$ & $\iota_1 = \log (2 \pc \ucC K \delta)$, $\iota_{2, n} = \log (2 \pc^2 \ucC K / (\pc - n) \delta)$ for $n \in \N$ (\cref{sec:update proofs}) \\
       $\xi_{2, n}$ & $\xi_{2, n} = \iota_{2, n} + \log \log_2 (16KH^2)$ (\cref{sec:update proofs}) \\
	$\square$ & an indefinite constant independent of $H$, $\aX$, $\aA$, $\varepsilon$, and $\delta$ (\cref{sec:update proofs})\\
        \midrule
 	$\funcE$ & event of $f$ close to $\sigma(v_*)$ \\
 	$\vboundE$ & event of $v_k$ bound for all $k$ \\
	$\EkboundE$ & event of small $E_k$ for all $k$ (not variance-aware)\\
	$\epskboundE$ & event of small $\varepsilon_k$ for all $k$ (not variance-aware)\\
	$\EkrboundE$ & event of small $E_k$ for all $k$ (variance-aware)\\
	$\phiqboundE$ & event of $v_\sigma$ close to $v_*$ (\cref{sec:evaluate proof})\\
	$\phisigmaboundE$ & event of learned $\phi^\top \omega$ close to $\sigma(v_*)$ (\cref{sec:evaluate proof})\\
	\bottomrule
	\end{tabular}
\end{table}

\section{Equivalence of MDVI Update Rules \citep{kozuno2022kl}}\label{sec:equivalence proof}

We show the equivalence of MDVI's updates \eqref{eq:MDVI update} to those used in \tmdvi.
The following transformation is identical to that of \citet{kozuno2022kl} but is included here for completeness.
We first recall MDVI's updates \eqref{eq:MDVI update}:
\begin{gather*}
    q_{k+1} = r + \gamma \widehat{P}_k \pi_k \paren*{
        q_k - \tau \log \frac{\pi_k}{\pi_{k-1}} - \kappa \log \pi_k
    }\,,
\end{gather*}
\begin{align*}
    \text{where }
    \pi_k \parenc*{\cdot}{x}
    =
    \argmax_{p \in \Delta(\A)}
    \sum_{a \in \A} p (a) \paren*{
        q_k (x, a) - \tau \log \frac{p (a)}{\pi_{k-1} \parenc*{a}{x}} - \kappa \log p (a)
    }
    \text{ for all } x \in \X\,,
\end{align*}

The policy update can be rewritten in a closed-form solution as follows (e.g., Equation~(5) of \citet{kozuno2019theoretical}):
\begin{gather*}
    \pi_k (a|x)
    =
    \frac{
        \pi_{k-1} (a|x)^\alpha \exp \paren*{ \beta q_k (x, a) }
    }{
        \sum_{b \in \A} \pi_{k-1} (b|x)^\alpha \exp \paren*{ \beta q_k (x, b) }
    }\,,
\end{gather*}
where $\alpha \df \tau / (\tau + \kappa)$, and $\beta \df 1 / (\tau + \kappa)$.
It can be further rewritten as, defining $s_k = q_k + \alpha s_{k-1}$,
\begin{gather*}
    \pi_k (a|x)
    =
    \frac{
        \exp \paren*{ \beta s_k (x, a) }
    }{
        \sum_{b \in \A} \exp \paren*{ \beta s_k (x, b) }
    }\,.
\end{gather*}
Plugging in this policy expression to $v_k$, we deduce that
\begin{align*}
    v_k (x)
    &=
    \frac{1}{\beta} \log \sum_{a \in \A} \exp \paren*{ \beta q_k (x, a) + \alpha\log \pi_{k-1} (a|x) }
    \\
    &=
    \frac{1}{\beta} \log \sum_{a \in \A} \exp \paren*{ \beta s_k (x, a) }
    -
    \frac{\alpha}{\beta} \log \sum_{a \in \A} \exp \paren*{ \beta s_{k-1} (x, a) }\,.
\end{align*}
\citet[Appendix~B]{kozuno2019theoretical} show that when $\beta \to \infty$,
$
    v_k (x)
    =
    w_k (x) - \alpha w_{k-1} (x)\,.
$
Furthermore, the Boltzmann policy becomes a greedy policy.
Accordingly, the update rules used in \tmdvi is a limit case of the original MDVI updates.
\section{Auxiliary Lemmas}\label{sec:auxiliary}

In this appendix, we prove some auxiliary lemmas used in the proof.
Some of the lemmas are identical to those of \citet{kozuno2022kl} but are included here for completeness.

\begin{lemma}\label{lemma:cond prob inequality}
    For any events $A$ and $B$, $\P (A \cap B) \geq \P(B) - \P(A^c | B)$.
\end{lemma}
\begin{proof}
    $\P(A \cap B) = \P ( (A \cup B^c) \cap B)
        \geq 1 - \P (A^c \cap B) - \P (B^c) = \P(B) - \P(A^c \cap B)\,.
    $
    The claim holds by $\P (A^c \cap B) = \P (A^c | B) \P (B) \leq \P (A^c | B)$. 
\end{proof}
   
\begin{lemma}\label{lemma:sqrt inequality}
    For any positive real values $a$ and $b$, $\sqrt{a + b} \leq \sqrt{a} + \sqrt{b}$.
\end{lemma}

\begin{proof}
    Indeed, $a + b \leq a + 2 \sqrt{ab} + b = (\sqrt{a} + \sqrt{b})^2$.
\end{proof}

\begin{lemma}\label{lemma:square gap bound}
    Let $a$, $b$, and $c$ are positive real values. If $|a^2 - b^2| \leq c^2$, then $|a - b| \leq c$.
\end{lemma}

\begin{proof}
    Without loss of generality, assume that $a \geq b$.
    Then, $c^2 \geq (a^2 - b^2) = (a + b)(a - b) \geq (a - b)^2$ and thus $|a - b| \leq c$.
\end{proof}

\begin{lemma}\label{lemma:square inequality}
    For any real values $(a_n)_{n=1}^N$, $(\sum_{n=1}^N a_n)^2 \leq N \sum_{n=1}^N a_n^2$.
\end{lemma}

\begin{proof}
    Indeed, from the Cauchy–Schwarz inequality,
    \begin{align*}
        \paren*{ \sum_{n=1}^N a_n \cdot 1 }^2
        \leq
        \paren*{ \sum_{n=1}^N 1 } \paren*{ \sum_{n=1}^N a_n^2 }
        =
        N \sum_{n=1}^N a_n^2\,,
    \end{align*}
    which is the desired result.
\end{proof}

\begin{lemma}\label{lemma:A_gamma_k bound}
    For any $k \in [K]$,
    \begin{align*}
        A_{\gamma, k}
        =
        \begin{cases}
            \gamma \dfrac{\alpha^k - \gamma^k}{\alpha - \gamma} & \text{if } \alpha \neq \gamma
            \\
            k \gamma^k & \text{otherwise}
        \end{cases}\,.
    \end{align*}
\end{lemma}

\begin{proof}
    Indeed, if $\alpha \neq \gamma$
    \begin{align*}
        A_{\gamma, k}
        =
        \sum_{j=0}^{k-1} \alpha^j \gamma^{k-j}
        =
        \gamma^k \frac{(\alpha / \gamma)^k - 1}{(\alpha / \gamma) - 1}
        =
        \gamma \frac{\alpha^k - \gamma^k}{\alpha - \gamma}\,.
    \end{align*}
    If $\alpha = \gamma$, $A_{\gamma, k} = k \gamma^k$ by definition.
\end{proof}

\begin{lemma}\label{lemma:log reciprocal inequality}
    For any real value $x \in (0, 1]$, $1 - x \leq \log (1/x)$.
\end{lemma}

\begin{proof}
    Since $\log (1/x)$ is convex and differentiable, $\log (1/x) \geq \log (1/y) - (x - y) / y$. Choosing $y=1$, we concludes the proof.
\end{proof}

\begin{lemma}\label{lemma:k gamma to k-th inequality}
    Suppose $\alpha, \gamma \in [0, 1)$, $\varepsilon \in (0, 1]$, $c \in [1, \infty)$, $m \in \N$, and $n \in [0, \infty)$.
    Let $K \df \dfrac{m}{1-\alpha} \log \dfrac{c H}{\varepsilon}$.
    Then,
    \begin{align*}
        K^n \alpha^K
        \leq
        \paren*{\frac{mn}{(1 - \alpha)e}}^n
        \paren*{\dfrac{\varepsilon}{c H}}^{m-1}\,.
    \end{align*}
\end{lemma}

\begin{proof}
    Using \cref{lemma:log reciprocal inequality} for $\alpha \in [0, 1)$,
    \begin{align*}
        K
        = \frac{m}{1-\alpha} \log \dfrac{c H}{\varepsilon}
        \geq \log_\alpha \paren*{\dfrac{\varepsilon}{c H}}^m.
    \end{align*}
    Therefore,
    \begin{align*}
        K^n \alpha^K
        \leq
        \paren*{\frac{m}{1 - \alpha} \log \dfrac{c H}{\varepsilon}}^n \paren*{\dfrac{\varepsilon}{c H}}^m
        =
        \frac{m^n}{(1 - \alpha)^n}
        \paren*{\dfrac{\varepsilon}{c H}}^m
        \paren*{\log \dfrac{c H}{\varepsilon}}^n\,.
    \end{align*}
    Since $x \paren*{\log \dfrac{1}{x}}^n \leq \paren*{\dfrac{n}{e}}^n$ for any $x \in (0, 1]$ as shown later,
    \begin{align*}
        K^n \alpha^K
        \leq
        \paren*{\frac{mn}{(1 - \alpha)e}}^n
        \paren*{\dfrac{\varepsilon}{c H}}^{m-1}\,.
    \end{align*}
    
    Now it remains to show $f (x) \df x \paren*{\log \dfrac{1}{x}}^n \leq \paren*{\dfrac{n}{e}}^n$ for $x < 1$.
    We have that
    \begin{align*}
        f'(x) = (- \log x)^n - n (- \log x)^{n-1}
        \implies
        f'(x) = 0
        \text{ at }
        x = e^{-n}.
    \end{align*}
    Therefore, $f$ takes its maximum $\paren*{\dfrac{n}{e}}^n$ at $e^{-n}$ when $x \in (0, 1)$.
\end{proof}

The following lemma is a special case of a well-known inequality that
for any increasing function $f$
\begin{align*}
    \sum_{k=1}^K f (k) \leq \int_1^{K+1} f(x) dx\,.
\end{align*}

\begin{lemma}\label{lemma:sum of k power}
    For any $K \in \N$ and $n \in [0, \infty)$,
    $
        \displaystyle \sum_{k=1}^K k^n \leq \frac{1}{n+1} (K+1)^{n+1}
    $.
\end{lemma}

\section{Tools from Probability Theory}

We extensively use the following concentration inequality, which is derived based on a proof idea of Bernstein's inequality \citep{bernstein1946theory,boucheron2013concentration} for a martingale \citep[Excercises 5.14 (f)]{lattimore2020bandit}. For a real-valued stochastic process $(X_n)_{n=1}^N$ adapted to a filtration $(\cF_n)_{n=1}^N$, we let $\E_n [X_n] \df \E \brackc{X_n}{\cF_{n-1}}$ for $n \geq 1$, and $\E_1 [X_1] \df \E \brack{X_1}$.

\begin{lemma}[Azuma-Hoeffding Inequality]\label{lemma:hoeffding}
    Consider a real-valued stochastic process $(X_n)_{n=1}^N$ adapted to a filtration $(\cF_n)_{n=1}^N$.
    Assume that $X_n \in [l_n, u_n]$ and $\E_n [X_n] = 0$ almost surely, for all $n$. Then,
	\begin{align*}
	\P \paren*{
      \sum_{n=1}^N X_n
      \geq
      \sqrt{
        \sum_{n=1}^N \frac{(u_n - l_n)^2}{2} \log \frac{1}{\delta}
      }
    }
    \leq \delta
	\end{align*}
	for any $\delta \in (0, 1)$.
\end{lemma}

\begin{lemma}[Conditional Azuma-Hoeffding's Inequality]\label{lemma:conditional hoeffding}
    Consider a real-valued stochastic process $(X_n)_{n=1}^N$ adapted to a filtration $(\cF_n)_{n=1}^N$.
    Assume that $\E_n [X_n] = 0$ almost surely, for all $n$. 
    Furthermore, let $\cE$ be an event that implies $X_n \in [l_n, u_n]$ with $\P (\cE) \geq 1 - \delta'$ for all $n$ and for some $\delta' \in (0, 1)$.
    Then,
	\begin{align*}
        \P \parenc*{
            \sum_{n=1}^N X_n
            \geq
              \sqrt{
            \sum_{n=1}^N \frac{(u_n - l_n)^2}{2} \log \frac{1}{\delta (1 - \delta')}
          }
        }{
            \cE
        } \leq \delta
	\end{align*}
	for any $\delta \in (0, 1)$.
\end{lemma}

\begin{proof}
    Let $A$ denote the events of
    \begin{align*}
    \sum_{n=1}^N X_n
            \geq
              \sqrt{
            \sum_{n=1}^N \frac{(u_n - l_n)^2}{2} \log \frac{1}{\delta (1 - \delta')}
          }\;.
    \end{align*}
    Accordingly,
    \begin{align*}
        \P (A | \cE)
        = \frac{\P (A \cap \cE)}{\P (\cE)}
        \numeq{\leq}{a} \frac{\delta (1-\delta')}{\P (\cE)}
        \numeq{\leq}{b} \delta\,,
    \end{align*}
    where (a) follows from the Azuma-Hoeffding inequality (\cref{lemma:hoeffding}), and (b) follows from $\P(\cE) \geq 1-\delta'$.
\end{proof}

\begin{lemma}[Lemma 13 in \citet{zhang2021modelFree}]\label{lemma:bernstein}
    Consider a real-valued stochastic process $(X_n)_{n=1}^N$ adapted to a filtration $(\cF_n)_{n=1}^N$.
    Suppose that $|X_n| \leq U$ and $\E_n [X_n] = 0$ almost surely, for all $n$ and for some $U \in [0, \infty)$.
    Then, letting $V_N \df \sum_{n=1}^N \E_n[X_n^2]$,
	\begin{align*}
        \P \paren*{
            \left|\sum_{n=1}^N X_n\right|
            \geq
            2\sqrt{2}\sqrt{V_N \log\paren*{\frac{1}{\delta}}}
            + 2 \sqrt{\epsilon \log\paren*{\frac{1}{\delta}}}
            + 2U\log \paren*{\frac{1}{\delta}}
        } \leq 2 \paren*{\log_2\paren*{\frac{NU^2}{\epsilon} + 1}}\delta\;,
	\end{align*}
	for any $\epsilon, \delta > 0$.
\end{lemma}

In our analysis, we use the following corollary of this inequality.

\begin{lemma}[Conditional Bernstein-type Inequality]\label{lemma:conditional bernstein}
    Consider a real-valued stochastic process $(X_n)_{n=1}^N$ adapted to a filtration $(\cF_n)_{n=1}^N$.
    Suppose that $\E_n [X_n] = 0$ almost surely, for all $n$.
    Furthermore, let $\cE$ be an event that implies $|X_n| \leq U$ with $\P (\cE) \geq 1 - \delta'$ for all $n$, for some $\delta' \in (0, 1)$ and $U \in [0, \infty)$.
    Then, letting $V_N \df \sum_{n=1}^N \E_n[X_n^2]$,
	\begin{align*}
        \P \parenc*{
            \left|\sum_{n=1}^N X_n\right|
            \geq
            2\sqrt{2}\sqrt{(1 + V_N) \log\paren*{\frac{2\log_2\paren*{NU^2}}{\delta (1 - \delta')}}}
            + 2U\log \paren*{\frac{2\log_2(NU^2)}{\delta(1 - \delta')}}
        }{\cE} \leq \delta \;,
	\end{align*}
	for any $\delta > 0$.
\end{lemma}

\begin{proof}
    Let $A$ and $B$ denote the events of
    \begin{align*}
        \left|\sum_{n=1}^N X_n\right|
        \geq
        2\sqrt{2}\sqrt{(1 + V_N) \log\paren*{\frac{2\log_2(NU^2)}{\delta (1 - \delta')}}}
        + 2U\log \paren*{\frac{2\log_2(NU^2)}{\delta(1 - \delta')}}\;
    \end{align*}
    and $|X_n| \leq U$ for all $n$, respectively.
    Since $\cE \subset B$, it follows that $A \cap \cE \subset A \cap B$, and
    $
        \P (A \cap \cE) \leq \P (A \cap B)
    $.
    Accordingly,
    \begin{align*}
        \P (A | \cE)
        = \frac{\P (A \cap \cE)}{\P (\cE)}
        \leq \frac{\P (A \cap B)}{\P (\cE)}
        \numeq{\leq}{a} \frac{\delta (1-\delta')}{\P (\cE)}
        \numeq{\leq}{b} \delta\,,
    \end{align*}
    where (a) follows from \cref{lemma:bernstein} with $1 + \sqrt{V_N} \geq \sqrt{1 + V_N}$ due to \cref{lemma:sqrt inequality} and $\epsilon = 2$. Then, (b) follows from $\P(\cE) \geq 1-\delta'$.
\end{proof}

\begin{lemma}[Popoviciu's Inequality for Variances]\label{lemma:popoviciu}
    The variance of any random variable bounded by $x$ is bounded by $x^2$.
\end{lemma}
\section{Total Variance Technique \citep{kozuno2022kl}}\label{app:total variance}

This section introduces the total variance technique for non-stationary policy. 
The proof is identical to that of \citet{kozuno2022kl} but is included here for completeness.

The following lemma is due to \citet{azar2013minimax}.

\begin{lemma}\label{lemma:variance decomposition}
    Suppose two real-valued random variables $X, Y$ whose variances, $\Var X$ and $\Var Y$, exist and are finite. Then, $\sqrt{\Var X} \leq \sqrt{\Var \brack*{X - Y}} + \sqrt{\Var Y}$.
\end{lemma}

For completeness, we prove \cref{lemma:variance decomposition}.

\begin{proof}
  Indeed, from Cauchy-Schwartz inequality,
  \begin{align*}
    \Var X
    &=
    \Var \brack{X - Y + Y}
    \\
    &=
    \Var \brack{X - Y}
    + \Var Y
    + 2 \E \brack*{ (X - Y - \E \brack{X-Y} ) (Y - \E Y) }
    \\
    &\leq
    \Var \brack{X - Y}
    + \Var Y
    + 2 \sqrt{
      \Var \brack{X - Y} \Var Y
    }
    =
    \paren*{
      \sqrt{\Var \brack*{X - Y}} + \sqrt{\Var Y}
    }^2\,.
  \end{align*}
  This is the desired result.
\end{proof}

The following lemma is an extension of Lemma~7 by \citet{azar2013minimax} and its refined version by \citet{agarwal2020modeBased}.

\begin{lemma}\label{lemma:total variance}
    Suppose a sequence of deterministic policies $( \pi_k )_{k=0}^K$
    and let
    \begin{align*}
        \qf{\pi'_k}
        \df
        \begin{cases}
            r + \gamma P \vf{\pi'_{k-1}} & \text{for } k \in [K]
            \\
            \qf{\pi_0} & \text{for } k = 0
        \end{cases}\,.
    \end{align*}
    Furthermore, let $\sigma_k^2$ and $\Sigma_k^2$ be non-negative functions over $\XA$ defined by
    \begin{align*}
        \sigma_k^2 (x, a)
        \df
        \begin{cases}
            P \paren{ \vf{\pi'_{k-1}} }^2 (x, a) - \paren{P \vf{\pi'_{k-1}} }^2 (x, a) & \text{for } k \in [K]
            \\
            P \paren{ \vf{\pi_0} }^2 (x, a) - \paren{P \vf{\pi_0} }^2 (x, a) & \text{for } k = 0
        \end{cases}
    \end{align*}
    and
    \begin{gather}
        \Sigma_k^2 (x, a)
        \df
        \E_k \brackc*{
            \paren*{ \sum_{t=0}^\infty \gamma^t r (X_t, A_t) - \qf{\pi'_k} (X_0, A_0) }^2
        }{X_0=x, A_0=a}
    \end{gather}
    for $k \in \{0\} \cup [K]$, where $\E_k$ is the expectation over $(X_t, A_t)_{t=0}^\infty$ wherein $A_t \sim \pi_{k-t} (\cdot | X_t)$
    until $t = k$, and $A_t \sim \pi_0 (\cdot | X_t)$ thereafter.
    Then,
    \begin{align*}
        \sum_{j=0}^{k-1} \gamma^{j + 1} P_{k-j}^{k-1} \sigma_{k-j} \leq \sqrt{2 H^3}\bone
    \end{align*}
    for any $k \in [K]$.
\end{lemma}

For its proof, we need the following lemma.

\begin{lemma}\label{lemma:variance Bellman eq for non-stationary policy}
    Suppose a sequence of deterministic policies $( \pi_k )_{k = 0}^K$ and notations in \cref{lemma:total variance}.
    Then, for any $k \in [K]$, we have that
    \begin{align*}
        \Sigma_k^2
        =
        \gamma^2 \sigma_k^2 + \gamma^2 P_{\pi_{k-1}} \Sigma_{k-1}^2\,.
    \end{align*}
\end{lemma}

\begin{proof}
    Let $R_s^u \df \sum_{t=s}^u \gamma^{t-s} r (X_t, A_t)$
    and $\E_k \brackc*{\cdot}{x, a} \df \E_k \brackc*{\cdot}{X_0=x, A_0=a}$.
    We have that
    \begin{align*}
        \Sigma_k^2 (x, a)
        =
        \E_k \brackc*{
          \paren*{ R_0^\infty - \qf{\pi'_k} (X_0, A_0) }^2
        }{x, a}
        \df
        \E_k \brackc*{
          \paren*{ I_1 + \gamma I_2 }^2
        }{x, a}\,,
    \end{align*}
    where $I_1 \df r (X_0, A_0) + \gamma \qf{\pi'_{k-1}} (X_1, A_1) - \qf{\pi'_k} (X_0, A_0)$, and $I_2 \df R_1^\infty - \qf{\pi'_{k-1}} (X_1, A_1)$.
    With these notations, we see that
    \begin{align*}
    \Sigma_k^2 (x, a)
    &=
    \E_k \brackc[\big]{
      I_1^2 + \gamma^2 I_2^2 + 2 \gamma I_1 I_2
    }{x, a}
    \\
    &=
    \E_k \brackc[\big]{
      I_1^2
      + \gamma^2 I_2^2
      + 2 \gamma I_1 \E_{k-1} \brackc*{I_2}{X_1, A_1}
    }{x, a}
    \\
    &=
    \E_k \brackc[\big]{I_1^2}{x, a} + \gamma^2 \E_k \brackc[\big]{I_2^2}{x, a}
    \\
    &=
    \E_k \brackc[\big]{I_1^2}{x, a}
    + \gamma^2 P^{\pi_{k-1}} \Sigma_{k-1}^2 (x, a)\,,
    \end{align*}
    where the second line follows from the law of total expectation,
    and the third line follows since
    $\E_{k-1} \brackc*{I_2}{X_1, A_1} = 0$ due to the Markov property.
    The first term in the last line is $\gamma^2 \sigma_k^2 (x, a)$
    because
    \begin{align*}
        \E_k \brackc[\big]{I_1^2}{x, a}
        &\numeq{=}{a}
        \gamma^2 \E_k \brackc[\Bigg]{
            \paren[\Big]{
                \underbrace{\qf{\pi_{k-1}'} (X_1, A_1)}_{
                    \vf{\pi_{k-1}'} (X_1) \text{ from (b)}
                }
                - (P \vf{\pi_{k-1}'}) (X_0, A_0)
            }^2
        }{x, a}
        \\
        &=
        \gamma^2 \paren*{ P \paren*{\vf{\pi_{k-1}'}}^2} (x, a)
        + \gamma^2 (P \vf{\pi_{k-1}'})^2 (x, a)
        - 2 (P \vf{\pi_{k-1}'})^2 (x, a)
        \\
        &=
        \gamma^2 \paren*{ P \paren*{\vf{\pi_{k-1}'}}^2} (x, a)
        - \gamma^2 (P \vf{\pi_{k-1}'})^2 (x, a)\,,
    \end{align*}
    where (a) follows from the definition that $\qf{\pi_k'} = r + \gamma P \vf{\pi'_{k-1}}$,
    and (b) follows since the policies are deterministic.
    From this argument, it is clear that
    $
        \Sigma_k^2
        =
        \gamma^2 \sigma_k^2 + \gamma^2 P_{\pi_{k-1}} \Sigma_{k-1}^2\,,
    $
    which is the desired result.
\end{proof}

Now, we are ready to prove \cref{lemma:total variance}.

\begin{proof}[Proof of \cref{lemma:total variance}]
  Let $H_k \df \sum_{j=0}^{k-1} \gamma^j$. Using Jensen's inequality twice,
  \begin{align*}
    \sum_{j=0}^{k-1} \gamma^{j + 1} P_{k-j}^{k-1} \sigma_{k-j}
    &\leq
    \sum_{j=0}^{k-1} {
      \gamma^{j + 1}
      \sqrt{
        P_{k-j}^{k-1} \sigma_{k-j}^2
      }
    }
    \\
    &\leq
    \gamma H_k \sum_{j=0}^{k-1} {
      \frac{\gamma^{j + 1}}{H_k}
      \sqrt{
        P_{k-j}^{k-1} \sigma_{k-j}^2
      }
    }
    \\
    &\leq
    \sqrt{
      H_k
      \sum_{j=0}^{k-1} {
        \gamma^{j+2} P_{k-j}^{k-1} \sigma_{k-j}^2
      }
    }
    \leq
    \sqrt{
      H
      \sum_{j=0}^{k-1} {
        \gamma^{j+2} P_{k-j}^{k-1} \sigma_{k-j}^2
      }
    }\,.
  \end{align*}
  From \cref{lemma:variance Bellman eq for non-stationary policy}, we have that
  \begin{align*}
    \sum_{j=0}^{k-1} {
      \gamma^{j+2} P_{k-j}^{k-1} \sigma_{k-j}^2
    }
    &=
    \sum_{j=0}^{k-1} {
      \gamma^j P_{k-j}^{k-1} \paren*{
        \Sigma_{k-j}^2 - \gamma^2 P^{\pi_{k-1-j}} \Sigma_{k-1-j}^2
      }
    }
    \\
    &=
    \sum_{j=0}^{k-1} {
      \gamma^j P_{k-j}^{k-1} \paren*{
        \Sigma_{k-j}^2
        - \gamma P^{\pi_{k-1-j}} \Sigma_{k-1-j}^2
        + \gamma (1-\gamma) P^{\pi_{k-1-j}} \Sigma_{k-1-j}^2
      }
    }
    \\
    &=
    \sum_{j=0}^{k-1} \gamma^j P_{k-j}^{k-1} \Sigma_{k-j}^2
    - \sum_{j=1}^k \gamma^j P_{k-j}^{k-1} \Sigma_{k-j}^2
    + \gamma (1-\gamma) \sum_{j=0}^{k-1} \gamma^j P_{k-1-j}^{k-1} \Sigma_{k-1-j}^2\,.
  \end{align*}
  The final line is equal to
  $
    \Sigma_k^2
    - \gamma^k P_0^{k-1} \Sigma_0^2
    + \gamma (1-\gamma) \sum_{j=0}^{k-1} \gamma^j P_{k-1-j}^{k-1} \Sigma_{k-1-j}^2
  $.
  Finally, from the monotonicity of stochastic matrices
  and that $\bzero \leq \Sigma_j^2 \leq H^2 \bone$ for any $j$,
  \begin{align*}
    \sum_{j=0}^{k-1} \gamma^{j + 1} P_{k-j}^{k-1} \sigma_{k-j} \leq \sqrt{2 H^3} \bone\,.
  \end{align*}
  This concludes the proof.
\end{proof}
\section{Proof of \cref{theorem:KW}}\label{appendix:proof of KW}

As a reminder, let $\Phi\df \{ \phi (x, a) : (x, a) \in \XA \} \subset \R^d$.
For $G \in \R^{d\times d}$ and $\phi \in \R^d$, we use the notation $\|\phi\|_G^2 := \phi^\top G \phi$.
Additionally, we use the operator norm of a matrix $G$ and denote it as $\|G\|=\sup_{\phi^\top \phi=1} \sqrt{(G\phi)^\top G \phi}$.

We first introduce an algorithm for computing the G-optimal design for finite $\X$, called the Frank-Wolfe algorithm from \citet{todd2016minimum}.
The pseudocode is provided in \cref{algo:frank wolfe}.
The following theorem shows that \cref{algo:frank wolfe} outputs a near-optimal design with a small core set.

\begin{theorem}[\textbf{Proposition 3.17}, \citet{todd2016minimum}]\label{theorem:FW}
Let $\ucC\df4d\log\log (d+4) + 28$.
For $\Phi$ satisfying \cref{assumption:linear mdp} and if $\Phi$ is finite, \cref{algo:frank wolfe} with $f: \XA \to (0, \infty)$ and $\varepsilon^{\mathrm{FW}}=d$ outputs a design $\rho$ such that $g(\rho) \leq 2d$ and the core set $C$ with size at most $\ucC$.
\end{theorem}

We extend the theorem to a compact $\Phi$ by passing to the limit.
The proof of \cref{theorem:KW} is a modification of \textbf{Exercise 21.3} in \citet{lattimore2020bandit}.

\begin{proof}[Proof of \cref{theorem:KW}]
Suppose that $\Phi$ satisfies \cref{assumption:linear mdp} such that $\Phi$ is a compact subset of $\R^d$ and spans $\R^d$.
Let $\left(\Phi_n\right)_n$ be a sequence of finite subsets with $\Phi_n \subset \Phi_{n+1}$.
We suppose that $\Phi_n$ spans $\R^d$ and $\lim _{n \rightarrow \infty} D\left(\Phi, \Phi_n\right)=0$ where $D$ is the Hausdorff metric. 
Then let $\rho_n$ be a $G$-optimal design for $\Phi_n$ with support of size at most $\ucC$ and $G_n\df\sum_{(x, a)\in \XA}\rho_n(x, a)\phi(x, a) \phi(x, a)^\top$. 
Such the design is ensured to exist by \cref{theorem:FW}.
Given any $\phi \in \Phi$, we have

\begin{equation}\label{eq:phi-norm}
    \|\phi\|_{G_n^{-1}} \leq \min _{b \in \Phi_n}\left(\|\phi-b\|_{G_n^{-1}}+\|b\|_{G_n^{-1}}\right) \leq \sqrt{2d}+\min _{b \in \Phi_n}\|\phi-b\|_{G_n^{-1}} \;,
\end{equation}

where the first inequality is due to the triangle inequality and the second inequality is due to \cref{theorem:FW}.
Let $W \in \R^{d\times d}$ be an invertible matrix and $w_i \in \R^d$ be its $i \in [d]$ th column.
We suppose that $w_i \in \Phi$ for any $i \in [d]$. Such $W$ can be constructed due to the assumption that $\Phi$ spans $\R^d$.
Then, the operator norm of $G^{-1/2}_n$ is bounded by

\begin{equation}\label{eq:G_n norm tmp}
\left\|G_n^{-1 / 2}\right\| =\left\|W^{-1} W G_n^{-1 / 2}\right\| 
\leq\left\|W^{-1}\right\|\left\|G_n^{-1 / 2} W\right\|
=\left\|W^{-1}\right\| \sup_{\phi^\top \phi=1} \|W\phi\|_{G_n^{-1}}\;,
\end{equation}

where the last equality is due to $\left\|G_n^{-1 / 2} W\right\|=\sup_{\phi^\top \phi=1} \sqrt{(G_n^{-1 / 2} W\phi)^\top G_n^{-1 / 2} W \phi}=\sup_{\phi^\top \phi=1} \|W\phi\|_{G_n^{-1}}$.
Let $\phi_i$ be the $i$ th element of $\phi \in \R^d$.
\cref{eq:G_n norm tmp} is further bounded by

\begin{equation*}
\sup_{\phi^\top \phi=1} \|W\phi\|_{G_n^{-1}}
\leq \sup_{\phi^\top \phi=1} \sum_{i=1}^d\left|\phi_i\right|\underbrace{\left\|w_i\right\|_{G_n^{-1}}}_{\leq \sqrt{2d}}
\leq 2d\;.
\end{equation*}

Therefore, we have $\left\|G_n^{-1 / 2}\right\| \leq 2d \left\|W^{-1}\right\|$.
Taking the limit $n\to \infty$ shows that

\begin{equation*}
\begin{aligned}
\limsup _{n \rightarrow \infty}\|\phi\|_{G_n^{-1}} & \numeq{\leq}{a} \sqrt{2d}+\limsup _{n \rightarrow \infty} \min _{b \in \Phi}\|\phi-b\|_{G_n^{-1}} \\
& \numeq{\leq}{b} \sqrt{2d}+2d\left\|W^{-1}\right\| \limsup _{n \rightarrow \infty} \min _{b \in \Phi}\sqrt{(\phi-b)^\top(\phi - b)}
= \sqrt{2d} \;,
\end{aligned}
\end{equation*}
where (a) is due to \eqref{eq:phi-norm} and (b) uses $\left\|G_n^{-1 / 2}\right\| \leq 2d \left\|W^{-1}\right\|$.

Since 
$
\|\cdot\|_{G_n^{-1}}: \Phi \rightarrow \R
$ is continuous and $\Phi$ is compact, it follows that
\begin{equation}\label{eq:limsup phi bound}
\limsup _{n \rightarrow \infty} \sup _{\phi \in \Phi}\|\phi\|_{G_n^{-1}}^2 \leq 2d\;.
\end{equation}

Notice that $\rho_n$ may be represented as a tuple of vector/probability pairs with at most $\ucC$ entries and where the vectors lie in $\Phi$. 
Since the set of all such tuples with the obvious topology forms a compact set, it follows that $\left(\rho_n\right)$ has a cluster point $\rho^*$, which represents a distribution on $\Phi$ with support at most $\ucC$. 
Then, \cref{eq:limsup phi bound} shows that $g\left(\rho^*\right) \leq 2d$. 
This concludes the proof.
\end{proof}

\section{Proof of Weighted KW Bound (\cref{lemma:kw bound})}\label{sec:proof of weighted kw bound}

\begin{proof}
$\left|\phi^\top  (x, a) \kwsum(f, z)\right|$ can be rewritten as
\begin{equation} \label{eq:kwsum bound one}
    \begin{aligned}
        \left|\phi^\top  (x, a) \kwsum(f, z)\right| 
        &= \left|\phi^\top  (x, a) G_{f}^{-1} \sum_{(y, b)\in \cC_{f}} \rho_{f}(y, b)\frac{\phi(y, b)}{f(y, b)}\frac{z(y, b)}{f(y, b)}\right| \\
        &\numeq{\leq}{a} \left|\sum_{(y, b)\in \cC_{f}} \rho_{f}(y, b)\phi^\top (x, a) G_{f}^{-1} \frac{\phi(y, b)}{f(y, b)} \right| \coremaxf{y', b'}{f} \left| \frac{z(y', b')}{f(y', b')}\right|\\
        &\numeq{\leq}{b} \sum_{(y, b)\in \cC_{f}} \left|\rho_{f}(y, b)\phi^\top (x, a) G_{f}^{-1} \frac{\phi(y, b)}{f(y, b)} \right| \coremaxf{y', b'}{f} \left| \frac{z(y', b')}{f(y', b')}\right|\;,
    \end{aligned}
\end{equation}
where (a) is due to H\"older's inequality and (b) is due to the triangle inequality.

Next, for any $(x,a) \in \XA$, we have
\begin{equation} \label{eq:kwsum bound two}
    \begin{aligned}
        \left(\sum_{(y, b)\in \cC_{f}}\left| \rho_{f}(y, b) \phi(x, a)^\top G_{f}^{-1}\frac{\phi(y, b)}{f(y, b)}\right|\right)^2 
        &\numeq{\leq}{a} \sum_{(y, b)\in \cC_{f}}\rho_{f}(y, b) \left|\phi(x, a)^\top G_{f}^{-1}\frac{\phi(y, b)}{f(y, b)}\right|^2 \\
        &\numeq{=}{b}  f^2(x, a) \underbrace{\frac{\phi(x, a)^\top }{f(x, a)} G_{f}^{-1} \frac{\phi(x, a)}{f(x, a)}}_{\leq 2d \text{ from \cref{theorem:weighted KW}}}
    \end{aligned}
\end{equation}
where (a) is due to Jensen's inequality, (b) is due to the definition of $G_{f}$.
The claim holds by taking the square root for both sides of the inequality \eqref{eq:kwsum bound two} and applying the result to the inequality \eqref{eq:kwsum bound one}.
\end{proof}

\section{Formal Theorems and Proofs of \cref{informal theorem: epsilon bound} and \cref{informal theorem: sqrt H bound}}\label{sec:update proofs}

This section provides the concrete proofs of \cref{informal theorem: epsilon bound} and \cref{informal theorem: sqrt H bound}.
Instead of the informal theorems of \cref{informal theorem: epsilon bound} and \cref{informal theorem: sqrt H bound}, we are going to prove the formal theorems below, \cref{theorem: epsilon bound} and \cref{theorem: sqrt H bound}, respectively.

\begin{theorem}[Sample complexity of \wlsmdvi with $f\approx \sigma(v_*)$]\label{theorem: epsilon bound}
    Let $\pc$ be a positive constant such that $8 \geq \pc \geq 6$ and $\tsigma \in \cF_q$ be a random variable. 
    Assume that $\varepsilon \in (0, 1 / H]$ and an event 
    $$\sigma(v_*) \leq \tsigma \leq \sigma(v_*) + 2\sqrt{H}\bone\;$$
    occurs with probability at least $1 - 4\delta / \pc$.
    Define
    \begin{equation*}
    \begin{aligned}
    &f^{\text{wls}} \df \max\paren*{\min(\tsigma, H\bone), \sqrt{H}\bone}\,,\\
    &\NIter^{\text{wls}}\df\left\lceil\frac{3}{1-\alpha} \log {{c_{1} H}}+1\right\rceil \,,\\
    \text{ and }\  &\NP^{\text{wls}}\df\left\lceil{\frac{c_2dH^2}{\varepsilon^2}\log\paren*{\frac{2\pc^2\ucC K^{\text{wls}}}{(\pc - 5)\delta }\log_2 \frac{16K^{\text{wls}} H^2}{(c_0 - 5)\delta}}}\right\rceil\,
    \end{aligned}
    \end{equation*}
    where $c_1, c_2 \geq 1$ are positive constants and $\ucC = 4d\log\log (d+4) + 28$. Then, there exist $c_1, c_2 \geq 1$ independent of $d$, $H$, $\aX$, $\aA$, $\varepsilon$, and $\delta$ such that \wlsmdvi is run with the settings 
    $\alpha = \gamma$, $f = f^{\text{wls}}$, $\NIter=\NIter^{\text{wls}}$, $\NP=\NP^{\text{wls}}$
    it outputs a sequence of policies $(\pi_k)_{k=0}^{\NIter}$ such that
    $
        \infnorm{\vf{*} - \vf{\pi'_{\NIter}}}
        \leq
        \varepsilon
    $
    with probability at least $1 - \delta$, using $\widetilde{\cO} \paren*{\ucC \NIter \NP} = \widetilde{\cO} \paren*{d^2H^3 / \varepsilon^2}$ samples from the generative model.
\end{theorem}

\begin{theorem}[Sapmle complexity of \wlsmdvi with $f=\bone$]\label{theorem: sqrt H bound}
    Assume that $\varepsilon \in (0, 1 / H]$.
    Let $\pc$ be a positive constant such that $8 \geq \pc \geq 6$.
    Define
    \begin{equation*}
        \begin{aligned}
        &\NIter^{\text{ls}}\df\left\lceil\frac{3}{1-\alpha} \log {c_{3} H}+1\right\rceil \,\\
        \text{ and }\ &\NP^{\text{ls}}\df\left\lceil\frac{c_4dH^2}{\varepsilon}\log{\frac{2 \pc^2 \ucC K^{\text{ls}}}{(\pc - 5)\delta}}\right\rceil\,
        \end{aligned}
    \end{equation*}
    where $c_3, c_4 \geq 1$ are positive constants and $\ucC = 4d\log\log (d + 4) + 28$.
    Then, there exist $c_3, c_4 \geq 1$ independent
    of $d$, $H$, $\aX$, $\aA$, $\varepsilon$ and $\delta$ such that
    when \wlsmdvi is run with the settings $\alpha = \gamma$, $f = \bone$, $\NIter=\NIter^{\text{ls}}$, and $\NP=\NP^{\text{ls}}$,
    it outputs $v_{\NIter}$ such that
    $
        \infnorm{\vf{*} - v_{\NIter}}
        \leq
        \frac{1}{2}\sqrt{H}
    $
    with probability at least $1 - 3\delta / \pc$, using $\widetilde{\cO} \paren*{\ucC \NIter \NP} = \widetilde{\cO} \paren*{d^2H^3 / \varepsilon}$ samples from the generative model.
\end{theorem}

The proof sketch is provided in \cref{subsec:proof sketch}.

\subsection{Notation and Frequently Used Facts for Proofs}

Before moving on to the proofs, we introduce some notations and frequently used facts for theoretical analysis.

\looseness=-1
\paragraph{Notation for proofs.}
$\square$ denotes an indefinite constant that changes throughout the proof and is independent of $d$, $H$, $\aX$, $\aA$, $\varepsilon$, and $\delta$.

\looseness=-1
For a sequence of policies $(\pi_k)_{k\in \Z}$, we let  $T_j^i \df T_{\pi_i} T_{\pi_{i-1}} \cdots T_{\pi_{j+1}} T_{\pi_j}$ for $i \geq j$, and $T_j^i \df I$ otherwise.

For $k\in \{1, \ldots, N\}$, we write $\theta^*_k \in \R^d$ as the underlying unknown parameter vector satisfying $\phi^\top \theta^*_{k} = r + \gamma Pv_{k-1}$.
$\theta^*_k$ is ensured to exist by the property of linear MDPs.
We also write $\btheta^*_k$ as its past moving average, i.e., ${\btheta}^*_k = \sum^{k}_{j=1} \alpha^{k-j}\theta^*_j$.

\looseness=-1
For \cref{theorem: sqrt H bound}, $\bF_{k, m}$ denotes the $\sigma$-algebra generated by random variables $\{ y_{j, n, x, a} | (j, n, x, a) \in [k-2] \times [M] \times \XA \} \cup \{ y_{j, n, x, a} | (j, n, x, a) \in \{k-1\} \times [m-1] \times \XA \}$.
With an abuse of notation, for \cref{theorem: epsilon bound}, $\bF_{k, m}$ denotes the $\sigma$-algebra generated by random variables $\{\widetilde{\sigma}\}\cup\{ y_{j, n, x, a} | (j, n, x, a) \in [k-2] \times [M] \times \XA \} \cup \{ y_{j, n, x, a} | (j, n, x, a) \in \{k-1\} \times [m-1] \times \XA \}$.
Whether $\bF_{k, m}$ is for \cref{theorem: sqrt H bound} or \cref{theorem: epsilon bound} shall be clear from the context.

\looseness=-1
For the bounded positive function $f$ used in \wlsmdvi, we introduce the shorthand $\uf \df \max_{(x, a)\in \XA} f(x, a)$ and $\lf \df \min_{(x, a)\in \XA} f(x, a)$.

\looseness=-1
Finally, throughout the proof, for $8 \geq c_0 > n > 0$, we write 
$\iota_1 \df \log (2 \pc \ucC K / \delta)$,
$\iota_{2, n} \df \iota_1 + \log(\pc / (\pc - n)) = \log (2 \pc^2 \ucC K / (\pc - n) \delta) $, 
and $\xi_{2, n} \df \iota_{2, n} + \log \log_2 (16KH^2)$.
Note that for any $8 \geq c_0 > n > 0$,
\begin{equation}\label{eq:iota bound}
\xi_{2, n} \geq \iota_{2, n} \geq \iota_1
\end{equation}
due to $8 \geq c_0 - n > 0$ and $16 K H^2 / \delta \geq 16$.
Whether $K$ is from \cref{theorem: epsilon bound} or \cref{theorem: sqrt H bound} shall be clear from the context.

\looseness=-1
\paragraph{Frequently Used Facts.}

\looseness=-1
Recall that $A_{\gamma, k} \df \sum_{j=0}^{k-1} \gamma^{k-j} \alpha^j$ and $A_k \df \sum_{j=0}^{k-1} \alpha^j$ for any non-negative integer $k$ with $A_\infty \df 1 / (1-\alpha)$.
We often use $\alpha = \gamma$ due to the settings of \cref{theorem: epsilon bound,theorem: sqrt H bound}. This indicates that $A_\infty = H$ and $A_{\gamma, k} = k\gamma^k$.

Recall that 
$
\theta_{k}
    =
    \argmin_{\theta \in \R^d} {
        \sum_{(y,b) \in \cCf} {
            {\frac{\rhof(y, b)}{f^2(y, b)}
            \left(
                \phi^\top(y, b) \theta - \hq_{k}(y, b)
            \right)^2 }
        }
    }
$.
Using the definition of 
$\kwsum$ defined in \cref{lemma:kw bound}
and 
$G_f$ defined in \cref{eq:weighted optimal design}, the closed-form solution to $\theta_k$ is represented as $\theta_k = \kwsum(f, \hat{q}_k)$.
In the similar manner, $\theta_k^* = \kwsum(f, \phi^\top \theta_k^*)$.

Since $\hat{q}_k - \phi^\top\theta_k^* = \varepsilon_k$, we have
\begin{align*}
\theta_{k} - \theta^*_{k} &= \kwsum(f, \hat{q}_k) - \kwsum(f, \phi^\top\theta_k^*) = \kwsum(f, \varepsilon_k) \\
\text{ and }\ \btheta_{k} - \btheta^*_{k} &= \kwsum\paren*{f, \sum^{k}_{j=1}\alpha^{k-j}\varepsilon_j} = \kwsum(f, E_k)\;,
\end{align*}
Moreover, for any $k \in \{1, \ldots, K\}$, we have that
\begin{equation}\label{eq:rewrite s_k}
    \begin{aligned}
    s_k & = \phi^\top \btheta_k \\
    &= \phi^\top \btheta^*_k + \phi^\top\kwsum(f, E_k)\\
    &= \sum_{j=1}^k \alpha^{k-j} (r + \gamma P(w_{j-1} - \alpha w_{j-2})) + \phi^\top\kwsum(f, E_k)\\
    &= A_k r + \gamma P w_{k-1} + \phi^\top\kwsum(f, E_k)\;.
    \end{aligned}
\end{equation}

In addition, we often mention the ``monotonicity'' of stochastic matrices:
any stochastic matrix $\rho$ satisfies that
$\rho v \geq \rho u$ for any vectors $v, u$ s.t. $v \geq u$.
Examples of stochastic matrices in the proof are $P$, $\pi$, $P^{\pi}$, and $\pi P$.
The monotonicity property is so frequently used that we do not always mention it.

\subsection{Proof Sketch}\label{subsec:proof sketch}

This section provides proof sketches of \cref{theorem: epsilon bound,theorem: sqrt H bound}, those are necessary to show \cref{theorem: sample complexity of vwls mdvi}.
The proofs follow the strategy of \citet{kozuno2022kl} but with modifications for the linear function approximation.

\paragraph{Step 1: Error Propagation Analysis.}
The proof of \cref{theorem: epsilon bound} is done by deriving a tight bound for $\vf{*} - \vf{\pi_K'}$.
Recall that $K$ is the number of iterations in \wlsmdvi and $\kwsum$ is the operator defined in \cref{lemma:kw bound}.
The following lemmas provide the bound for any $k \in [K]$. We provide the proof in \cref{subsec:proof of non-stationary error propagation}.

\begin{lemma}[Error Propagation Analysis ($\vf{\pi_k'}$)]\label{lemma:non-stationary error propagation}
    For any $k \in [K]$, $\bzero \leq \vf{*} - \vf{\pi'_k} \leq \Gamma_k$ where 
    \begin{align*}
        \Gamma_k
        \df \displaystyle \frac{1}{A_\infty} \sum_{j=0}^{k-1} \gamma^j \paren*{
            \pi_k P_{k-j}^{k-1} - \pi_* P_*^j
        } \phi^\top \kwsum(f, E_{k-j})
        + 2 H \paren*{ \alpha^k + \frac{A_{\gamma, k}}{A_\infty} } \bone\,.
    \end{align*}
\end{lemma}

Let
\begin{equation*}
    \begin{aligned}
        \heartsuit_k &\df H^{-1}\sum_{j=0}^{k-1} \gamma^j \pi_k P_{k-j}^{k-1} \abs*{\phi^\top \kwsum(f, E_{k-j})},
        \\
        \clubsuit_k &\df H^{-1}\sum_{j=0}^{k-1} \gamma^j \pi_* P_*^j \abs*{\phi^\top \kwsum(f, E_{k-j})},
        \\
        \text{and }
        \diamondsuit_k &\df \square H \paren*{ \alpha^k + \frac{A_{\gamma, k}}{A_\infty} }.
    \end{aligned}
\end{equation*}
We derive the bound of $\|\vf{*} - \vf{\pi_K'}\|_\infty$ 
 by bounding $\heartsuit_K$, $\clubsuit_K$ and $\diamondsuit_K$.
Since $\diamondsuit_k$ can be easily controlled by \cref{lemma:k gamma to k-th inequality}, we focus on the bounds of $\heartsuit_K$ and $\clubsuit_K$.
To derive the tight bounds of $\heartsuit_K$ and $\clubsuit_K$, we need to transform them into ``TV technique compatible'' forms; we will transform $\heartsuit_k$ into $\sum_{j=0}^{k-1} \gamma^j \pi_k P_{k-j}^{k-1} \sigma(v_{\pi_{k-j}})$ and $\clubsuit_k$ into $\sum_{j=0}^{k-1} \gamma^j\pi_* P_*^j \sigma(v_*)$.
The transformations are provided in \textbf{Step 3} and \textbf{4}.

On the other hand, the proof of \cref{theorem: sqrt H bound} is done by deriving a coarse bound of $\vf{*} - \vf{K}$. 
Then, the following bound (\cref{lemma:v error prop}) is helpful. The proof is provided in \cref{subsec:proof of v error prop}.

\begin{lemma}[Error Propagation Analysis ($v_k$)]\label{lemma:v error prop}
    For any $k \in [K]$,
    \begin{align*}
        - 2 \gamma^k H \bone - \sum_{j=0}^{k-1} \gamma^j \pi_{k-1} P_{k-j}^{k-1} \phi^\top\kwsum(f, \varepsilon_{k-j})
        \leq
        \vf{*} - v_k
        \leq
        \Gamma_{k-1} + 2 H \gamma^k \bone - \sum_{j=0}^{k-1} \gamma^j \pi_{k-1} P_{k-1-j}^{k-2} \phi^\top\kwsum(f, \varepsilon_{k-j})\,.
    \end{align*}
\end{lemma}

We first prove \cref{theorem: sqrt H bound} in the next \textbf{Step 2} since it is straightforward compared to \cref{theorem: epsilon bound}.

\looseness=-1
\paragraph{Step 2: Prove \cref{theorem: sqrt H bound}.}
Note that $f = \bone$ in \cref{theorem: sqrt H bound}.
As you can see from \cref{lemma:v error prop}, we need the bounds of $|\phi^\top\kwsum(\bone, \varepsilon_k)|$ and $|\phi^\top\kwsum(\bone, E_k)|$ for the proof.

\looseness=-1
By bounding $\varepsilon_k$ and $E_k$ using the Azuma-Hoeffding inequality (\cref{lemma:hoeffding}), the weighted KW bound with $f=\bone$ (\cref{lemma:kw bound}) and the settings of \cref{theorem: sqrt H bound} yild $|{\phi^\top \kwsum(\bone, \varepsilon_k)}| \leq \widetilde{\mathcal{O}}(1/ \sqrt{H}) \bone$ and $|{\phi^\top \kwsum(\bone, E_k)}| \leq \widetilde{\mathcal{O}}(\sqrt{H}) \bone$ with high-probability.
Furthermore, $\diamondsuit_K$ is bounded by $\widetilde{\mathcal{O}}(1)$ due to \cref{lemma:k gamma to k-th inequality}.

\looseness=-1
Inserting these results into \cref{lemma:v error prop}, we obtain $\infnorm{\vf{*}-v_K} \leq \widetilde{\mathcal{O}}(\sqrt{H})$ with high-probability, which is the desired result of \cref{theorem: sqrt H bound}.

The detailed proofs of \textbf{Step 2} are provided in \cref{subsec:kwsum bounds proofs} and \cref{subsec: proof of sqrt H bound}.

\looseness=-1
\paragraph{Step 3: Refined Bound of $\clubsuit_K$ for \cref{theorem: epsilon bound}.}
Recall that the weighting function $f$ satisfies $\sigma(v_*) \leq f \leq \sigma (v_*) + 2\sqrt{H} \bone$ and $\sqrt{H}\bone \leq f \leq H\bone$ in \cref{theorem: epsilon bound}. 
The assumptions allow us to apply TV technique to $\clubsuit_K$ when the bound of $\phi^\top\kwsum(f, E_k)$ scales to $f$.
This is where the weighted KW bound (\cref{lemma:kw bound}) comes in.

\looseness=-1
Due to \cref{lemma:kw bound}, we have $|\phi^\top\kwsum(f, E_k)| \leq \sqrt{2d}f \coremaxf{y, b}{f}|E_k(y, b) / f(y, b)|$.
Thus, the tight bound of can be obtained by tightly bounding $\coremaxf{y, b}{f}|E_k(y, b) / f(y, b)|$.

\looseness=-1
By applying the Bernstein-type inequality (\cref{lemma:conditional bernstein}) to $|E_k / f|$, discounted sum of $\sigma(v_{j}) / f$ from $j=1$ to $k$ appears inside the bound of $|E_k / f|$.
We decompose it as $\sigma(\vf{j}) / f \leq |\vf{*} - \vf{j}| / \sqrt{H}+ \bone$ by \cref{lemma:variance decomposition} and \cref{lemma:popoviciu}.
Therefore, we obtain a discounted sum of $|\vf{*} - \vf{j}|$ in $|E_k / f|$ bound, which can be bounded in a similar way to \textbf{Step 2}. 

\looseness=-1
Now we have the bound of $\phi^\top\kwsum(f, E_k)$ which scales to $f$.
Combined with the settings of \cref{theorem: epsilon bound}, we obtain $|\phi^\top\kwsum(f, E_k)| \leq  \widetilde{\mathcal{O}}(\varepsilon (\sigma(v_*) / \sqrt{H} + \bone)) $.
The TV technique is therefore applicable to $\clubsuit_K$ and thus $\clubsuit_K \leq \widetilde{\mathcal{O}}(\varepsilon)$. 

\looseness=-1
The detailed proofs of \textbf{Step 3} are provided in \cref{subsubsec:clubsuit proof}.

\paragraph{Step 4: Refined Bound of $\heartsuit_K$ for \cref{theorem: epsilon bound}.}
We need a further transformation since TV technique in $\heartsuit_K$ requires $\sigma(\vf{\pi_k'})$, not $\sigma(\vf{*})$.
To this end, we decompose $\sigma(\vf{*})$ as $\sigma(\vf{*}) \leq \sigma(\vf{*} - \vf{\pi_{k}'}) + \sigma(\vf{\pi_{k}'}) \leq |\vf{*} - \vf{\pi_{k}'}| + \sigma(\vf{\pi_{k}'})$ by \cref{lemma:variance decomposition} and \cref{lemma:popoviciu}.
Thus, we need the bound of $|\vf{*} - \vf{\pi_{k}'}|$ which requires the coarse bound of $\infnorm{\phi^\top\kwsum(f, E_k)}$.

By applying the Azuma-Hoeffding inequality to $E_k$, the settings of \cref{theorem: sqrt H bound} yields $\infnorm{\phi^\top\kwsum(f, E_k)} \leq \widetilde{\mathcal{O}}(\sqrt{H})$.
Inserting this bound to \cref{lemma:non-stationary error propagation}, $\infnorm{\vf{*} - \vf{\pi_{k}'}} \leq \widetilde{\mathcal{O}}(\sqrt{H}) + 2(H + k)\gamma^k$ (\cref{lemma:non-stationary coarse bound}).

By taking a similar procedure as \textbf{Step 3}, together with the bound of $\infnorm{\vf{*} - \vf{\pi_{k}'}}$, 
$\heartsuit_K$ is bounded by $\widetilde{\mathcal{O}}(\varepsilon H^{-1.5})\sum_{j=0}^{k-1} \pi_k P_{k-j}^{k-1} (\sigma(\vf{\pi_{k-j}}) + \sqrt{H}\bone)$.
Then, the TV technique yilds $\heartsuit_K \leq \widetilde{\mathcal{O}}(\varepsilon)$.

The detailed proofs of \textbf{Step 4} are provided in \cref{subsubsec:heartsuit proof}.

Finally, we obtain the desired result of \cref{theorem: epsilon bound} by inserting the bounds of $\heartsuit_K$ and $\clubsuit_K$ to \cref{lemma:non-stationary coarse bound} (\cref{subsec: proof of epsilon bound}.)

\subsection{Proofs of Error Propagation Analysis (\textbf{Step 1})}\label{subsec:proofs of error propagation analysis}

\subsubsection{Proof of \texorpdfstring{
        \cref{lemma:non-stationary error propagation}
    }{
        Lemma~\ref{lemma:non-stationary error propagation}
    }
}\label{subsec:proof of non-stationary error propagation}
\begin{proof}
    Note that
    \begin{align*}
        \bzero
        \leq
        \vf{*} - \vf{\pi'_k}
        =
        \frac{A_k}{A_\infty} \paren*{ \vf{*} - \vf{\pi'_k} }
        + \alpha^k \paren*{ \vf{*} - \vf{\pi'_k} }
        \leq
        \frac{A_k}{A_\infty} \paren*{ \vf{*} - \vf{\pi'_k} }
        + 2 H \alpha^k \bone
    \end{align*}
    due to $\vf{*} - \vf{\pi'_k} \leq 2 H \bone$.
    Therefore, we need an upper bound for $A_k (\vf{*} - \vf{\pi'_k})$.
    We decompose $A_k (\vf{*} - \vf{\pi'_k})$ to $A_k \vf{*} - w_k$ and $w_k - A_k\vf{\pi'_k}$.
    Then, we derive upper bounds for each of them
    (inequalities \eqref{eq:vstar - m s_k bound upper bound} and \eqref{eq:m s_k - vpi prime bound upper bound}, respectively).
    The desired result is obtained by summing up those bounds.
    
    \paragraph{Upper bound for $A_k \vf{*} - w_k$.}
    We prove by induction that for any $k \in [K]$,
    \begin{align}\label{eq:vstar - m s_k bound upper bound}
        A_k \vf{*} - w_k
        \leq
        H A_{\gamma, k} \bone - \sum_{j=0}^{k-1} \gamma^j \pi_* P_*^j \phi^T\kwsum(f, E_{k-j})\,.
    \end{align}
    We have that
    \begin{align*}
        A_k \vf{*} - w_k
        &\numeq{\leq}{a}
        \pi_* (A_k \qf{*} - s_k)
        \\
        &\numeq{=}{b}
        \pi_* \paren*{
            A_k \qf{*} - A_k r - \gamma P w_{k-1} - \phi^T\kwsum(f, E_k)
        }
        \\
        &\numeq{=}{c}
        \pi_* \paren*{
            \gamma P (A_k \vf{*} - w_{k-1}) - \phi^T\kwsum(f, E_k)
        }
        \\
        &\numeq{\leq}{d}
        \pi_* \paren*{
            \gamma P (A_{k-1} \vf{*} - w_{k-1})
            + \alpha^{k-1} \gamma H \bone
            - \phi^T\kwsum(f, E_k)
        }\,,
    \end{align*}
    where (a) is due to the greediness of $\pi_k$,
    (b) is due to the equation \eqref{eq:rewrite s_k},
    (c) is due to the Bellman equation for $\qf{*}$,
    and (d) is due to the fact that $(A_k - A_{k-1}) \vf{*} = \alpha^{k-1} \vf{*} \leq \alpha^{k-1} H \bone$.
    For $k=1$, using (a), (b), and (c) with the facts that $w_0 = \bzero$ and $A_1=1$, we have
    \begin{equation*}
        A_1 \vf{*} - w_1 
        \leq
        \pi_* \paren*{
            \gamma P \vf{*} - \phi^T\kwsum(f, E_1)
        }
        \leq \gamma H \bone - \pi_* \phi^T\kwsum(f, E_1)\,
    \end{equation*}
    and thus the inequality \eqref{eq:vstar - m s_k bound upper bound} holds for $k=1$.
    From the step (d) above and induction, it is straightforward to verify that the inequality \eqref{eq:vstar - m s_k bound upper bound} holds for other $k$.

    \paragraph{Upper bound for $w_k - A_k \vf{\pi'_k}$.}
    We prove by induction that for any $k \in [K]$,
    \begin{equation}\label{eq:m s_k - vpi prime bound upper bound}
        w_k - A_k \vf{\pi'_k}
        \leq
        H A_{\gamma, k} \bone + \sum_{j=0}^{k-1} \gamma^j \pi_k P_{k-j}^{k-1} \phi^T\kwsum(f, E_{k-j})\,.
    \end{equation}
    Recalling that $\vf{\pi'_k} = \pi_k T_0^{k-1} \qf{\pi_0}$,
    we deduce that
    \begin{align*}
        w_k - A_k \vf{\pi'_k}
        &\numeq{=}{a}
        \pi_k \paren*{ s_k - A_k T_0^{k-1} \qf{\pi_0} }
        \\
        &\numeq{=}{b}
        \pi_k \paren*{
            A_k r + \gamma P w_{k-1}
            - A_k T_1^{k-1} \qf{\pi_0}
            + \phi^T\kwsum(f, E_k)
        }
        \\
        &\numeq{=}{c}
        \pi_k \paren*{
            \gamma P \paren*{w_{k-1} - A_k \vf{\pi'_{k-1}}} + \phi^T\kwsum(f, E_k)
        }
        \\
        &\numeq{\leq}{d}
        \pi_k \paren*{
            \gamma P (w_{k-1} - A_{k-1} \vf{\pi'_{k-1}})
            + \alpha^{k-1} \gamma H \bone
            + \phi^T\kwsum(f, E_k)
        }\,,
    \end{align*}
    where (a) follows from the definition of $w_k$,
    (b) is due to the equation \eqref{eq:rewrite s_k} and $T_0^{k-1} \qf{\pi_0} = T_1^{k-1} \qf{\pi_0}$,
    (c) is due to the equation $r - T_1^{k-1} \qf{\pi_0} = -P\vf{\pi'_{k-1}}$ which follows from the definition of the Bellman operator,
    and (d) is due to the fact that $(A_k - A_{k-1}) \vf{\pi'_{k-1}} = \alpha^{k-1} \vf{\pi'_{k-1}} \geq - \alpha^{k-1} H \bone$.
    For $k=1$, using (a), (b), and (c) with the facts that $w_0 = \bzero$ and $A_1=1$, we have
    \begin{equation*}
     w_1 - A_1 \vf{\pi'_1} = 
     \pi_1 \paren*{
        -\gamma P \vf{\pi'_{0}} + \phi^T\kwsum(f, E_1)}
    \leq 
     \gamma H \bone + \pi_1 \phi^T\kwsum(f, E_1)\,,
    \end{equation*}
    and thus the inequality \eqref{eq:m s_k - vpi prime bound upper bound} holds for $k=1$.
    From the step (d) above and induction, it is straightforward to verify that the inequality \eqref{eq:m s_k - vpi prime bound upper bound} holds for other $k$.
\end{proof}

\subsubsection{Proof of \texorpdfstring{
        \cref{lemma:v error prop}
    }{
        Lemma~\ref{lemma:v error prop}
    }
}\label{subsec:proof of v error prop}

We first prove an intermediate result.

\begin{lemma}\label{lemma:pre v error prop}
    For any $k \in [K]$,
    \begin{align*}
        \vf{\pi'_{k-1}}
        + \sum_{j=0}^{k-1} \gamma^j \pi_{k-1} P_{k-1-j}^{k-2} \phi^\top \kwsum(f, \varepsilon_{k-j})
        - \gamma^k H \bone
        \leq
        v_k
        \leq
        \vf{\pi'_k}
        + \sum_{j=0}^{k-1} \gamma^j \pi_k P_{k-j}^{k-1} \phi^\top \kwsum(f, \varepsilon_{k-j})
        + \gamma^k H \bone\,.
    \end{align*}
\end{lemma}

\begin{proof}
    From the greediness of $\pi_{k-1}$,
    $
        v_k
        =
        w_k - \alpha w_{k-1}
        \leq
        \pi_k (s_k - \alpha s_{k-1})
        =
        \pi_k (
            r + \gamma P v_{k-1} + \phi^\top \kwsum(f, \varepsilon_k)
        )
    $.
    By induction on $k$, therefore,
    \begin{align*}
        v_k
        \leq
        \sum_{j=0}^{k-1} \gamma^j \pi_k P_{k-j}^{k-1} \paren*{
            r + \phi^\top \kwsum(f, \varepsilon_{k-j})
        }
        + \underbrace{{\gamma^k \pi_k P_0^{k-1} v_0}}_{=\bzero}
        \leq
        \sum_{j=0}^{k-1} \gamma^j \pi_k P_{k-j}^{k-1} \paren*{
            r + \phi^\top \kwsum(f, \varepsilon_{k-j})
        }\;.
    \end{align*}
    Note that
    \begin{align*}
        T_0^{k-1} \qf{\pi_0}
        =
        \sum_{j=0}^{k-1} \gamma^j P_{k-j}^{k-1} r + \gamma^k \underbrace{P_0^{k-1} \qf{\pi_0}}_{\geq - H\bone}
        \implies
        \sum_{j=0}^{k-1} \gamma^j P_{k-j}^{k-1} r
        \leq
        T_0^{k-1} \qf{\pi_0} + \gamma^k H\,.
    \end{align*}
    Accordingly, 
    $
        v_k
        \leq
        \pi_k T_0^{k-1} \qf{\pi_0}
        + \sum_{j=0}^{k-1} \gamma^j \pi_k P_{k-j}^{k-1} \phi^\top \kwsum(f, \varepsilon_{k-j})
        + \gamma^k H \bone\,.
    $
    
    Similarly, from the greediness of $\pi_k$,
    $
        v_k
        =
        w_k - \alpha w_{k-1}
        \geq
        \pi_{k-1} (s_k - \alpha s_{k-1})
        \geq
        \pi_{k-1} (
            r + \gamma P v_{k-1} + \phi^\top \kwsum(f, \varepsilon_k)
        )
    $.
    By induction on $k$, therefore,
    \begin{align*}
        v_k
        \geq
        \sum_{j=0}^{k-1} \gamma^j \pi_{k-1} P_{k-1-j}^{k-2} \paren*{
            r + \phi^\top \kwsum(f, \varepsilon_{k-j})
        } 
        + \underbrace{\gamma^{k-1} \pi_{k-1} P_0^{k-2} P v_0}_{=\bzero}\,.
    \end{align*}
    Note that $T_0^{k-2} \qf{\pi_0} = T_0^{k-2} (r + \gamma P \vf{\pi_0})$ and
    \begin{align*}
        T_{0}^{k-2} \qf{\pi_0}
        =
        \sum_{j=0}^{k-1} \gamma^j P_{k-1-j}^{k-2} r + \gamma^k \underbrace{P_0^{k-2} P \vf{\pi_0}}_{\leq H \bone}
        \implies
        \sum_{j=0}^{k-1} \gamma^j P_{k-1-j}^{k-2} r
        \geq
        T_0^{k-2} \qf{\pi_0} - \gamma^k H\,.
    \end{align*}
    Accordingly, 
    $
        v_k
        \geq
        \pi_{k-1} T_0^{k-2} \qf{\pi_0}
        + \sum_{j=0}^{k-1} \gamma^j \pi_{k-1} P_{k-1-j}^{k-2} \phi^\top \kwsum(f, \varepsilon_{k-j})
        - \gamma^k H \bone\,.
    $
\end{proof}

\begin{proof}[Proof of \cref{lemma:v error prop}]
    From \cref{lemma:pre v error prop} and $\pi_k T_{\pi_{k-1}} \cdots T_{\pi_1} \qf{\pi_0} = \vf{\pi'_k} \leq \vf{*}$, we have that
    \begin{equation}
        \vf{\pi'_{k-1}}
        + \sum_{j=0}^{k-1} \gamma^j \pi_{k-1} P_{k-1-j}^{k-2} \phi^\top \kwsum(f, \varepsilon_{k-j})
        - 2 \gamma^k H \bone
        \leq
        v_k
        \leq
        \vf{*}
        + \sum_{j=0}^{k-1} \gamma^j \pi_k P_{k-j}^{k-1} \phi^\top \kwsum(f, \varepsilon_{k-j})
        + 2 \gamma^k H \bone\,,
    \end{equation}
    where we loosened the bound by multiplying $\gamma^k H$ by $2$.
    By simple algebra, for any $k \in [\NIter]$,
    \begin{align}
    v_* - v_k &\geq - 2 \gamma^k H - \sum_{j=0}^{k-1} \gamma^j \pi_k P_{k-j}^{k-1} \phi^\top \kwsum(f, \varepsilon_{k-j})\,\label{eq:vk bound lhs}\\
    \text{ and }\quad \vf{\pi'_{k-1}} -v_k &\leq
    2 \gamma^k H \bone
    - \sum_{j=0}^{k-1} \gamma^j \pi_{k-1} P_{k-1-j}^{k-2} \phi^\top \kwsum(f, \varepsilon_{k-j})\,. \label{eq:vpi-vk-tmp-bound}
    \end{align}
    For the second inequality, from \cref{lemma:non-stationary error propagation},
    \begin{equation}\label{eq: v gap lower bound tmp}
        \vf{\pi'_{k-1}}
        \geq
        \vf{*}
        - \frac{1}{A_\infty} \sum_{j=0}^{k-2} \gamma^j \paren*{
            \pi_{k-1} P_{k-1-j}^{k-2} - \pi_* P_*^j
        } \phi^\top \kwsum(f, E_{k-1-j})
        - 2 H \paren*{ \alpha^{k-1} + \frac{A_{\gamma, k-1}}{A_\infty} } \bone
    \end{equation}
    for any $k \in \{2, \ldots, K\}$.
    Since $v_{\pi_0} \geq v_* - 2H$ and the empty sum is defined to be $0$, the inequality \eqref{eq: v gap lower bound tmp} holds for $k = 1$.
    Therefore, by applying \eqref{eq: v gap lower bound tmp} to \eqref{eq:vpi-vk-tmp-bound}, we have that
    \begin{equation}\label{eq:vk bound rhs}
    \begin{aligned}
        \vf{*} - v_k
        &\leq
        2 H \gamma^k + 2 H \paren*{ \alpha^{k-1} + \frac{A_{\gamma, k-1}}{A_\infty} } \bone
        \\
        &\hspace{2em}
        + \frac{1}{A_\infty} \sum_{j=0}^{k-2} \gamma^j \paren*{
            \pi_{k-1} P_{k-1-j}^{k-2} - \pi_* P_*^j
        } \phi^\top \kwsum(f, E_{k-1-j})
        - \sum_{j=0}^{k-1} \gamma^j \pi_{k-1} P_{k-1-j}^{k-2} \phi^\top \kwsum(f, \varepsilon_{k-j})
    \end{aligned}
    \end{equation}
    for any $k \in [K]$.
    \cref{lemma:v error prop} holds by combining \eqref{eq:vk bound rhs} and \eqref{eq:vk bound lhs}.
\end{proof}

\subsection{Lemmas and Proofs of $\phi^\top \kwsum(f, \varepsilon_k)$ and $\phi^\top \kwsum(f, E_k)$ Bounds (\textbf{Step 2})}\label{subsec:kwsum bounds proofs}

This section provides formal lemmas and proofs about the high-probability bounds of $\phi^\top \kwsum(f, \varepsilon_k)$ and $\phi^\top \kwsum(f, E_k)$.

We first introduce the necessary events for the proofs.
\begin{event}[$\funcE$]\label{event:f}
    The input $f$ of \wlsmdvi satisfies
    $\sigma(v_*)(x, a) \leq f(x, a) \leq \sigma(v_*)(x, a) + 2\sqrt{H}$, and $\sqrt{H} \leq f(x, a) \leq H$ for all $(x, a) \in \XA$.
\end{event}

\begin{event}[$\vboundE$]\label{event:v is bounded}
    $v_k$ is bounded by $2H$ for all $k \in [K]$.
\end{event}

\begin{event}[$\EkboundE$]\label{event:coarse E_k bound}
    $
    \left|\phi^\top (x, a)\kwsum(f, E_k)\right|
    \leq (8H\uf / \lf)\sqrt{dA_\infty \iota_{2, 5} / \NP}
    $
    for all $(x, a, k) \in \XA \times [\NIter]$.
\end{event}

\begin{event}[$\epskboundE$]\label{event:eps_k bound core set}
    $
        \left|\phi^\top (x, a)\kwsum(f, \varepsilon_k)\right|
        \leq 
        (8\gamma H f(x, a) / \lf) \sqrt{d\iota_{2, 5} / \NP}
    $
    for all $(x, a, k) \in \XA \times [\NIter]$.
\end{event}

\begin{event}[$\EkrboundE$]\label{event:refined E_k bound}
    $
    |\phi^\top (x, a)\kwsum(f, E_k)| \leq \sqrt{2d} f(x, a)
    \paren*{
    {8 H\xi_{2, 5}} / (\lf \NP)
    + 
    2 \sqrt{2\xi_{2, 5} / (\lf^2 M)}
    +
    V_k
    }
    $
    where 
    $$V_k := 2\sqrt{\frac{2\xi_{2, 5}}{M} \sum_{j=1}^k \alpha^{2 (k-j)} \coremaxf{y, b}{f} \frac{\sigma^2(v_{j-1})(y, b)}{f^2(y, b)}}\;,$$
    for all $(x, a, k) \in \XA \times [K] $ 
\end{event}

$\funcE$ is for the condition of $f$ in \cref{theorem: epsilon bound}, and $\vboundE$ is for the use of concentration inequalities.
Our goal is to show that $\EkboundE$, $\epskboundE$, and $\EkrboundE$ occur with high probability in \cref{theorem: sqrt H bound} and \cref{theorem: epsilon bound}.

\subsubsection{Lemmas and Proofs of $v_k$ Bound ($\vboundE$)}

We first show that $\vboundE$ occurs with high probability. 
The following \cref{lemma:v is bounded} is for \cref{theorem: sqrt H bound}, and \cref{lemma:v is bounded two} is for \cref{theorem: epsilon bound}.

\begin{lemma}\label{lemma:v is bounded}
    With the settings of \cref{theorem: sqrt H bound}, there exists $c_4 \geq 1$ independent of $d$, $H$, $\aX$, $\aA$, $\varepsilon$ and $\delta$ such that $\P \paren*{\vboundE^c} \leq \delta / \pc$.
\end{lemma}
\begin{lemma}\label{lemma:v is bounded two}
    With the settings of \cref{theorem: epsilon bound}, there exists $c_2 \geq 1$ independent of $d$, $H$, $\aX$, $\aA$, $\varepsilon$ and $\delta$ such that $\P \parenc*{\vboundE^c}{\funcE} \leq \delta / \pc$.
\end{lemma}
\begin{proof}\label{proof: v is bounded}
    From the greediness of the policies $\pi_k$ and $\pi_{k-1}$,
    \begin{equation}\label{eq: v is bounded}
        \pi_{k-1} \phi^\top \theta_k
        = \pi_{k-1} (s_k - \alpha s_{k-1})
        \leq v_k
        \leq \pi_k (s_k - \alpha s_{k-1})
        = \pi_k \phi^\top \theta_k.
    \end{equation}
  
    Let $\varepsilon'_k := \varepsilon_k / \|v_{k-1}\|_\infty$ be a normalized error.
    We prove the claim by bounding $\phi^\top \theta_k$ as 
    \begin{align}
    \left|\phi^\top \theta_k\right| =
    \left|\phi^\top \kwsum(f, \hat{q}_k)\right|
    &\numeq{\leq}{a}
    \left|\phi^\top \kwsum(f, \phi^\top  \theta_k^*)\right| + \left|\phi^\top \kwsum(f, \varepsilon_k) \right|
    =
    \left|r + \gamma P v_{k-1}\right| + \left|\phi^\top \kwsum(f, \varepsilon_k) \right| \nonumber \\
    &\numeq{\leq}{b}
    \paren*{1 + \gamma \|v_{k-1}\|_\infty} \bone  + \frac{\uf\sqrt{2d}}{\lf}\coremaxf{x, a}{f}{\left|\varepsilon_k(x, a)\right|\bone} \nonumber\\
    &\numeq{=}{c}
    \paren*{1 + \gamma \|v_{k-1}\|_\infty} \bone + \frac{\uf\sqrt{2d}}{\lf}\|v_{k-1}\|_\infty \coremaxf{x, a}{f}{\left|\varepsilon'_k(x, a)\right|\bone}\;, \label{eq:phi theta decompose}
    \end{align}
    where (a) uses the triangle inequality, (b) is due to \cref{lemma:kw bound} and since $r$ is bounded by $1$, and (c) uses $\varepsilon'_k = \varepsilon_k / \|v_{k-1}\|_\infty$.
    We also used the shorthand $\uf \df \max_{(x, a)\in \XA} f(x, a)$ and $\lf \df \min_{(x, a)\in \XA} f(x, a)$.
  
    We need to bound $\coremaxf{x, a}{f}{\left|\varepsilon'_k(x, a)\right|}$.
    For $(x, a) \in \XA$, 
    \begin{align*}
        \varepsilon'_k (x, a) =
        \frac{\gamma}{\NP} \sum_{m=1}^{\NP} \underbrace{
            \paren[\Big]{ v_{k-1} (y_{k-1, m, x, a}) - P v_{k-1} (x, a) } / \|v_{k-1}\|_\infty
        }_{\text{bounded by } 2}\,
    \end{align*}
    is a sum of bounded martingale differences with respect to $(\bF_{k, m})_{m = 1}^{M}$.
    Using the Azuma-Hoeffding inequality (\cref{lemma:hoeffding}) and taking the union bound over $(x, a) \in \cCf$ and $k \in [\NIter]$, we have
    \begin{equation}\label{eq:ek bound for vk bound}
        \P \paren*{
          \exists (x, a, k) \in \cCf \times [K]\,
          \text{ s.t. }
          |\varepsilon'_k(x, a)|
          \geq
          \gamma \sqrt{\frac{8\iota_1}{\NP}}
        }
        \leq \frac{\delta}{c_0}\;.
    \end{equation}
    
    We are now ready to prove \cref{lemma:v is bounded} and \cref{lemma:v is bounded two} by induction.
    The claims hold for $k=0$ since $v_0 = \bzero$.
    Assume that $v_{k-1}$ is bounded by $2H$ for some $k \geq 1$. 

    \paragraph{\cref{lemma:v is bounded} proof}
    Note that $\uf / \lf = 1$ due to the settings of \cref{theorem: sqrt H bound}.
    Therefore, the following inequality holds with probability at least $1 - \delta / \pc$.
    \begin{equation*}
        \|\phi^\top \theta_k\|_\infty 
        \numeq{\leq}{a} 1 + \gamma 2H + 2H\sqrt{2d} \coremaxf{x, a}{f}{\left|\varepsilon'_k(x, a)\right|}
        \numeq{\leq}{b} 1 + \gamma 2H + 8H\gamma \sqrt{\frac{d\iota_1}{\NP}}\;,
    \end{equation*}
    where (a) is due to \eqref{eq:phi theta decompose} with the induction hypothesis and (b) the second inequality is due to \eqref{eq:ek bound for vk bound}. 
    Since $H = 1 / (1 - \gamma)$, by simple algebra, some $\NP$ such that $\NP \geq 64\gamma^2 H^2d\iota_1$ satisfies $\|\phi^\top \theta_k\|_\infty \leq 2H$ with probability at least $1 - \delta / c_0$.

    Recall that $\NP = \left\lceil c_4dH^2\iota_{2, 5} / \varepsilon\right\rceil$ in \cref{theorem: sqrt H bound}. 
    Due to the assumption of $\varepsilon \leq 1 / H$ and $\iota_{2, 5} \geq \iota_1$ by \eqref{eq:iota bound}, 
    the value of $M$ in \cref{theorem: sqrt H bound} satisfies $\NP \geq 64\gamma^2 H^2d\iota_1$ for some $c_4$.
    \cref{lemma:v is bounded} hence holds by inserting the result into the inequality \eqref{eq: v is bounded} with induction.
     
    \paragraph{\cref{lemma:v is bounded two} proof}
    Note that $\uf / \lf \leq \sqrt{H}$ due to the condition of \cref{lemma:v is bounded two}.
    Therefore, the following inequality holds with probability at least $1 - \delta / \pc$.
    \begin{equation*}
        \|\phi^\top \theta_k\|_\infty 
        \numeq{\leq}{a} 1 + \gamma 2H + 2H\sqrt{2dH} \coremaxf{x, a}{f}{\left|\varepsilon'_k(x, a)\right|}
        \numeq{\leq}{b} 1 + \gamma 2H + 8H\gamma \sqrt{\frac{dH\iota_1}{\NP}}\;,
    \end{equation*}
    where (a) is due to \eqref{eq:phi theta decompose} with the induction hypothesis and (b) the second inequality is due to \eqref{eq:ek bound for vk bound}.
    Since $H = 1 / (1 - \gamma)$, by simple algebra, some $\NP$ such that $\NP \geq 64\gamma^2 H^3d\iota_1$ satisfies $\|\phi^\top \theta_k\|_\infty \leq 2H$ with probability at least $1 - \delta / c_0$.

    \looseness=-1
    Recall that $M = \left\lceil{ {c_2dH^2 \xi_{2, 5}}/{\varepsilon^2}}\right\rceil$ in \cref{theorem: epsilon bound}.
    Due to the assumption of $\varepsilon \leq 1 / H$ and $\xi_{2, 5} \geq \iota_1$ by \eqref{eq:iota bound}, the value of $M$ in \cref{theorem: epsilon bound} satisfies $\NP \geq 64\gamma^2 H^3d\iota_1$ for some $c_2$.
    \cref{lemma:v is bounded two} hence holds by inserting the result into the inequality \eqref{eq: v is bounded} with induction.

\end{proof}

\subsubsection{Lemmas and Proofs of Coarse $\phi^\top \kwsum(f, E_k)$ Bound ($\EkboundE$)}

The following \cref{lemma:coarse E_k bound} is for \cref{theorem: sqrt H bound}, and \cref{lemma:coarse E_k bound two} is for \cref{theorem: epsilon bound}.
\begin{lemma}\label{lemma:coarse E_k bound}
    With the settings of \cref{theorem: sqrt H bound}, there exists $c_4 \geq 1$ independent of $d$, $H$, $\aX$, $\aA$, $\varepsilon$ and $\delta$ such that $\P \parenc{\EkboundE^c}{\vboundE} \leq \delta / \pc$.
\end{lemma}
\begin{lemma}\label{lemma:coarse E_k bound two}
    With the settings of \cref{theorem: epsilon bound}, there exists $c_2 \geq 1$ independent of $d$, $H$, $\aX$, $\aA$, $\varepsilon$ and $\delta$ such that $\P \parenc{\EkboundE^c}{\funcE \cap \vboundE} \leq \delta / \pc$.
\end{lemma}

\begin{proof}

For any $(x, a)\in \XA$ and $k \in [K]$, we have
\begin{equation}\label{eq:coarse Ek bound}
    \left|\phi^\top (x, a)\kwsum(f, E_k)\right| \leq \frac{\sqrt{2d} f(x, a)}{\lf} \underbrace{\coremaxf{y', b'}{f}\left|\sum^k_{j=1}\alpha^{k-j}\varepsilon_j(y', b')\right|}_{\heartsuit_k}\;,
\end{equation}
where the inequality is due to the weighted KW bound (\cref{lemma:kw bound}).

We need to bound $\heartsuit_k$.
Note that for a fixed $k \in [\NIter]$ and $(x, a) \in \cCf$, 
\begin{align*}
    \sum^k_{j=1}\alpha^{k-j}\varepsilon_j (x, a) =
    \frac{\gamma}{\NP} \sum^{k}_{j=1} \alpha^{k-j}\sum_{m=1}^{\NP} \underbrace{
        \paren[\Big]{ v_{j-1} (y_{j-1, m, x, a}) - P v_{j-1} (x, a) }
    }_{\text{bounded by } 4 H \text{ due to } \vboundE}\,
\end{align*}
is a sum of bounded martingale differences with respect to $(\bF_{j, m})_{j = 1, m = 1}^{k, M}$.
We are now ready to prove \cref{lemma:coarse E_k bound} and \cref{lemma:coarse E_k bound two} using the conditional Azuma-Hoeffding inequality (\cref{lemma:conditional hoeffding}).

\paragraph{\cref{lemma:coarse E_k bound} proof}
In the settings of \cref{theorem: sqrt H bound}, some $c_4$ satisfies that $\P(\vboundE) \geq 1 - \delta / \pc$ due to \cref{lemma:v is bounded}.
Using the conditional Azuma-Hoeffding inequality (\cref{lemma:conditional hoeffding}) and taking the union bound over $(x, a) \in \cCf$ and $k \in [\NIter]$,
\begin{align*}
    \P \parenc*{
      \exists (x, a, k) \in \cCf \times [\NIter]\,
      \;\text{ s.t. }\;
      {\sum^k_{j=1}\alpha^{k-j}\varepsilon_j(x, a)}
      \geq
      \gamma H\sqrt{\frac{32A_\infty \iota_{2, 1}}{\NP}}
    }{\vboundE}
    \leq \frac{\delta}{\pc} \text{}.\,
\end{align*}
where $\iota_{2, 1} = \iota_1 + \log (\pc / (\pc - 1)) $ is due to the condition by $\vboundE$.
We used $\iota_{2, 1}$ since $1 / (1 - \delta / \pc) \leq \pc / (\pc - 1)$.

Therefore, $\heartsuit_k \leq H\sqrt{32A_\infty \iota_{2, 1} / \NP}$ with probability at least $1 - \delta / \pc$ for all $k \in [K]$. The claim holds by inserting $\heartsuit_k$ into the inequality \eqref{eq:coarse Ek bound}.

\paragraph{\cref{lemma:coarse E_k bound two} proof}
In the settings of \cref{theorem: epsilon bound},
we have $\P(\funcE) \geq 1 - 4\delta / \pc$
and 
some $c_2$ satisfies that $\P\parenc*{\vboundE^c}{\funcE}\leq \delta / \pc$ due to \cref{lemma:v is bounded two}. 
Therefore, $\P(\funcE \cap \vboundE) \geq 1 - 5\delta / \pc$ holds due to \cref{lemma:cond prob inequality}.

Using \cref{lemma:conditional hoeffding} and taking the union bound over $(x, a) \in \cCf$ and $k \in [\NIter]$,
\begin{align*}
    \P \parenc*{
      \exists (x, a, k) \in \cCf \times [\NIter]\,
      \;\text{ s.t. }\;
      \sum^k_{j=1}\alpha^{k-j}\varepsilon_j(x, a)
      \geq
      \gamma H\sqrt{\frac{32A_\infty \iota_{2, 5}}{\NP}}
    }{\funcE \cap \vboundE}
    \leq \frac{\delta}{\pc} \text{},\,
\end{align*}
where $\iota_{2, 5} = \iota_1 + \log (\pc / (\pc - 5)) $ is due to the condition by $\funcE \cap \vboundE$.
We used $\iota_{2, 5}$ since $1 / (1 - 5\delta / \pc) \leq \pc / (\pc - 5)$.
\cref{lemma:coarse E_k bound two} holds in the same way as the proof of \cref{lemma:coarse E_k bound}.

\end{proof}

\subsubsection{Lemmas and Proofs of Coarse $\phi^\top \kwsum(f, \varepsilon_k)$ Bound ($\epskboundE$)}

The following \cref{lemma:eps_k bound core set} is for \cref{theorem: sqrt H bound}, and \cref{lemma:eps_k bound core set two} is for \cref{theorem: epsilon bound}.

\begin{lemma}\label{lemma:eps_k bound core set}
    With the settings of \cref{theorem: sqrt H bound}, there exists $c_4 \geq 1$ independent of $d$, $H$, $\aX$, $\aA$, $\varepsilon$ and $\delta$ such that $\P \parenc{\epskboundE^c}{\vboundE} \leq \delta / \pc$.
\end{lemma}
\begin{lemma}\label{lemma:eps_k bound core set two}
    With the settings of \cref{theorem: epsilon bound}, there exists $c_2 \geq 1$ independent of $d$, $H$, $\aX$, $\aA$, $\varepsilon$ and $\delta$ such that $\P \parenc{\epskboundE^c}{\funcE\cap\vboundE} \leq \delta / \pc$.
\end{lemma}

\begin{proof}\label{proof: eps_k bound core set}
For any $(x, a)\in \XA$ and $k \in [K]$, we have
\begin{equation}\label{eq:coarse ek bound}
|\phi^\top (x, a)\kwsum(f, \varepsilon_k)| \leq \frac{\sqrt{2d} f(x, a)}{\lf}\underbrace{\coremaxf{y', b'}{f}\left|{\varepsilon_k(y', b')}\right|}_{\heartsuit_k}\;,
\end{equation}
where the inequality is due to the weighted KW bound (\cref{lemma:kw bound}).

We need to bound $\heartsuit_k$.
Note that for a fixed $k \in [\NIter]$ and $(x, a) \in \cCf$, 
\begin{align*}
    \varepsilon_k (x, a) =
    \frac{\gamma}{\NP} \sum_{m=1}^{\NP} \underbrace{
        \paren[\Big]{ v_{k-1} (y_{k-1, m, x, a}) - P v_{k-1} (x, a) }
    }_{\text{bounded by } 4 H \text{ due to } \vboundE}\,
\end{align*}
is a sum of bounded martingale differences with respect to $(\bF_{k, m})_{m = 1}^{M}$.
We are ready to prove \cref{lemma:eps_k bound core set} and \cref{lemma:eps_k bound core set two} using the conditional Azuma-Hoeffding inequality (\cref{lemma:conditional hoeffding}).

\paragraph{\cref{lemma:eps_k bound core set} proof}
Note that some $c_4$ satisfies that $\P(\vboundE) \geq 1 - \delta / \pc$ in the settings of \cref{theorem: sqrt H bound} due to \cref{lemma:v is bounded}.
Using the conditional Azuma-Hoeffding inequality (\cref{lemma:conditional hoeffding}) and taking the union bound over $(x, a) \in \cCf$ and $k \in [\NIter]$, we have
\begin{align*}
    \P \parenc*{
      \exists (x, a, k) \in \cCf \times [\NIter]\,
      \text{ s.t. }
      |\varepsilon_k(x, a)|
      \geq
      \gamma H\sqrt{\frac{32\iota_{2, 1}}{\NP}}
    }{\vboundE}
    \leq \frac{\delta}{\pc} \text{}.\,
\end{align*}
where $\iota_{2, 1} = \iota_1 + \log (\pc / (\pc - 1)) $ is due to the condition by $\vboundE$.
We used $\iota_{2, 1}$ since $1 / (1 - \delta / \pc) \leq \pc / (\pc - 1)$.

Therefore, $\heartsuit_k \leq \gamma H\sqrt{32 \iota_{2, 1} / \NP}$ with probability at least $1 - \delta / \pc$ for all $k \in [K]$. The claim holds by inserting $\heartsuit_k$ into the inequality \eqref{eq:coarse ek bound}.

\paragraph{\cref{lemma:eps_k bound core set two} proof}
Due to \cref{lemma:cond prob inequality} and \cref{lemma:v is bounded two}, some $c_2$ satisfies that $\P(\funcE \cap \vboundE) \geq 1 - 5\delta / \pc$ in the settings of \cref{theorem: epsilon bound}.
Therefore, using \cref{lemma:conditional hoeffding} and taking the union bound over $(x, a) \in \cCf$ and $k \in [\NIter]$,
\begin{align*}
    \P \parenc*{
      \exists (x, a, k) \in \cCf \times [\NIter]\,
      \text{ s.t. }
      |\varepsilon_k(x, a)|
      \geq
      \gamma H\sqrt{\frac{32\iota_{2, 5}}{\NP}}
    }{\funcE \cap \vboundE}
    \leq \frac{\delta}{\pc} \text{}.\,
\end{align*}
where $\iota_{2, 5} = \iota_1 + \log (\pc / (\pc - 5)) $ is due to the condition by $\funcE \cap \vboundE$.
We used $\iota_{2, 5}$ since $1 / (1 - 5\delta / \pc) \leq \pc / (\pc - 5)$.

The claim holds in the same way as the proof of \cref{lemma:eps_k bound core set}.
\end{proof}

\subsubsection{Lemma and Proof of Refined $\phi^\top \kwsum(f, E_k)$ Bound ($\EkrboundE$)}

The following \cref{lemma:refined E_k bound} is for \cref{theorem: epsilon bound}.

\begin{lemma}\label{lemma:refined E_k bound}
    With the settings of \cref{theorem: epsilon bound}, there exists $c_2 \geq 1$ independent of $d$, $H$, $\aX$, $\aA$, $\varepsilon$ and $\delta$ such that $\P \parenc*{\EkrboundE^c}{\funcE\cap\vboundE} \leq \delta / \pc$.
\end{lemma}

\begin{proof}

For any $(x, a)\in \XA$ and $k \in [K]$, we have
\begin{equation}\label{eq:refined Ek bound}
    \left|\phi^\top (x, a)\kwsum(f, E_k)\right| \leq \sqrt{2d} f(x, a) {\coremaxf{y', b'}{f}\frac{1}{f(y', b')}\left|\sum^k_{j=1}\alpha^{k-j}\varepsilon_j(y', b')\right|}\;.
\end{equation}
where the inequality is due to the weighted KW bound (\cref{lemma:kw bound}).

We further bound \cref{eq:refined Ek bound} by bounding $\left|\sum^k_{j=1}\alpha^{k-j}\varepsilon_j(x, y)\right|$ over $(x, y) \in \cCf$.
For a fixed $k \in [\NIter]$ and $(x, a) \in \cCf$, 
\begin{align*}
    \sum^k_{j=1}\alpha^{k-j}\varepsilon_j (x, a) =
    \gamma \sum^{k}_{j=1} \alpha^{k-j} \frac{1}{M}\sum_{m=1}^{\NP}
        \underbrace{\paren[\Big]{ v_{j-1} (y_{j-1, m, x, a}) - P v_{j-1} (x, a) }
    }_{\text{bounded by } 4 H \text{ due to } \vboundE}\,
\end{align*}
is a sum of bounded martingale differences with respect to $(\bF_{j, m})_{j = 1, m = 1}^{k, M}$.

In the settings of \cref{theorem: epsilon bound},
we have $\P(\funcE) \geq 1 - 4\delta / \pc$
and 
some $c_2$ satisfies that $\P\parenc*{\vboundE^c}{\funcE}\leq \delta / \pc$ due to \cref{lemma:v is bounded two}. 
Therefore, $\P(\funcE \cap \vboundE) \geq 1 - 5\delta / \pc$ holds due to \cref{lemma:cond prob inequality}.
Using the conditional Bernstein-type inequality (\cref{lemma:conditional bernstein}) and taking the union bound over $k \in [\NIter]$ and $(x, a) \in \cCf$, we have
\begin{equation}\label{eq:Ek bound bernstein}
    \P \parenc*{
      \left|\sum^k_{j=1}\alpha^{k-j}\varepsilon_j(x, a)\right|
      \geq
      \frac{8 H\xi_{2, 5}}{\NP}
      + 
      2\sqrt{2}\sqrt{ 
      \frac{\xi_{2, 5}}{M} \paren*{1 + \sum_{j=1}^k \alpha^{2 (k-j)} \PVar(v_{j-1})(x, a)}
      }
    }{\funcE\cap\vboundE}
    \leq \frac{\delta}{\pc} \;,
\end{equation}
for all $(x, a, k) \in \cCf \times [\NIter]$.
Here, $\xi_{2, 5} = \iota_1 + \log (\pc / (\pc - 5)) + \log \log_2 (16KH^2)$ is due to the condition by $\funcE \cap \vboundE$.
We used $\xi_{2, 5}$ since $1 / (1 - 5\delta / \pc) \leq \pc / (\pc - 5)$.

Using the result, we have the following inequality with probability at least $1 - \delta / \pc$.
For all $(x, a, k) \in \cCf \times [\NIter]$,
\begin{align*}
    &{\coremaxf{y', b'}{f}\frac{1}{f(y', b')}\left|\sum^k_{j=1}\alpha^{k-j}\varepsilon_j(y', b')\right|}\\
    \numeq{\leq}{a}\;
    &\coremaxf{y', b'}{f}\frac{1}{f(y', b')}
    \paren*{
    \frac{8 H\xi_{2, 5}}{\NP}
      + 
      2\sqrt{2}\sqrt{ 
      \frac{\xi_{2, 5}}{M} \paren*{1 + \sum_{j=1}^k \alpha^{2 (k-j)} \PVar(v_{j-1})(y', b')}
      }
    }\\
    \numeq{\leq}{b}\;
    &\coremaxf{y', b'}{f}\frac{1}{f(y', b')}
    \paren*{
    \frac{8 H\xi_{2, 5}}{\NP}
      + 
      2\sqrt{2}\sqrt{ 
      \frac{\xi_{2, 5}}{M}} 
    + 2\sqrt{2}
    \sqrt{
    \frac{\xi_{2, 5}}{M}
     \sum_{j=1}^k \alpha^{2 (k-j)} \PVar(v_{j-1})(y', b')}
     }
    \\
    \numeq{\leq}{c}\;
    & \frac{8 H\xi_{2, 5}}{\NP \lf} + 
      2\sqrt{2}\sqrt{ 
      \frac{\xi_{2, 5}}{M\lf^2}} 
      +
    2\sqrt{2}
    \sqrt{ 
      \frac{\xi_{2, 5}}{M}\sum_{j=1}^k \alpha^{2 (k-j)} 
      \coremaxf{y', b'}{f}\frac{\PVar(v_{j-1})(y', b')}{f(y', b')}}
\end{align*}

where (a) is due to \eqref{eq:Ek bound bernstein}, (b) is due to \cref{lemma:sqrt inequality}, 
and (c) uses $\lf = \min_{(x, a)\in \XA} f(x, a)$.
The claim holds by inserting the result into the inequality \eqref{eq:refined Ek bound}.

\end{proof}

We are now ready to prove \cref{theorem: sqrt H bound}.

\subsection{Proof of \cref{theorem: sqrt H bound} (\textbf{Step2})}\label{subsec: proof of sqrt H bound}

\begin{proof}[Proof of \cref{theorem: sqrt H bound}]\label{proof:sqrt H bound}
    We condition the proof by the event $\vboundE \cap \EkboundE \cap \epskboundE$.
    Note that when \wlsmdvi is run with the settings defined in \cref{theorem: sqrt H bound}, 
    $
    {\P \paren{\vboundE^c} \leq {\delta} / {\pc}}
    $
    due to \cref{lemma:v is bounded}, 
    $
    \P \parenc{\EkboundE^c}{\vboundE} \leq {\delta} / {\pc}
    $
    due to \cref{lemma:coarse E_k bound}, 
    and 
    $\P \parenc{\epskboundE^c}{\vboundE} \leq \delta / \pc$
    due to \cref{lemma:eps_k bound core set}
    .
    Using \cref{lemma:cond prob inequality}, these indicate that
    \begin{align*}
        \P (\vboundE \cap \EkboundE \cap \epskboundE) 
        \geq 
        \P (\vboundE) - 
        \P \parenc*{\paren*{\EkboundE \cap \epskboundE}^c}{\vboundE}
        \geq
        \P (\vboundE) - \P \parenc{\EkboundE^c}{\vboundE} - 
        \P \parenc{\epskboundE^c}{\vboundE}
        \geq 1 - \frac{3\delta}{\pc}\;.
    \end{align*}
    Therefore, $\vboundE \cap \EkboundE \cap \epskboundE$ occurs with probability at least $1 - 2\delta / \pc$.
    
    We now prove the claim by bounding $v_* - v_K$.
    Recall \cref{lemma:v error prop} that, for any $k \in [K]$,
    \begin{align*}
        - 2 \gamma^k H \bone - \sum_{j=0}^{k-1} \gamma^j \pi_{k-1} P_{k-j}^{k-1} \phi^\top\kwsum(f, \varepsilon_{k-j})
        \leq
        \vf{*} - v_k
        \leq
        \Gamma_{k-1} + 2 H \gamma^k \bone - \sum_{j=0}^{k-1} \gamma^j \pi_{k-1} P_{k-1-j}^{k-2} \phi^\top\kwsum(f, \varepsilon_{k-j})\,,
    \end{align*}
    where 
    \begin{align*}
        \Gamma_k
        \df \displaystyle \frac{1}{A_\infty} \sum_{j=0}^{k-1} \gamma^j \paren*{
            \pi_k P_{k-j}^{k-1} - \pi_* P_*^j
        } \phi^\top \kwsum(f, E_{k-j})
        + 2 H \paren*{ \alpha^k + \frac{A_{\gamma, k}}{A_\infty} } \bone\,.
    \end{align*}
    When $\alpha = \gamma$, this bounds $\infnorm{v_* - v_K}$ as
    \begin{equation}\label{eq:coarse last value error}
        \infnorm{v_* - v_K} \leq 
        \underbrace{{\frac{1}{H}\sum_{j=0}^{K-1} \gamma^j \left\lVert \paren*{\pi_K P_{K-j}^{K-1} - \pi_* P_*^j} \phi^\top \kwsum(f, E_{K-j})\right\rVert_\infty}}_{\heartsuit}
        + \underbrace{\square (H + K) \gamma^K}_{\clubsuit}
        + \underbrace{H\max_{j\in [K]}\left\lVert{\phi^\top \kwsum(\bone, \varepsilon_{j})}\right\rVert_\infty}_{\diamondsuit}\,.
    \end{equation}
    
    We bound for each of them.
    Note that $\uf /\lf = 1$, 
    $\NIter=\left\lceil\frac{3}{1-\alpha} \log {c_{3} H}+1\right\rceil$,
    and $\NP = \left\lceil c_4dH^2\iota_{2, 5} / \varepsilon\right\rceil$ due to the settings of \cref{theorem: sqrt H bound}.
    
    First, $\heartsuit$ can be bounded as
    \begin{align*}
        \heartsuit
        \leq
        \frac{2}{H} \sum_{j=0}^{K-1} \gamma^j \left\lVert{\phi^\top \kwsum(f, E_{K-j})}\right\rVert_\infty
        \numeq{\leq}{a}
        \frac{2}{H} \sum_{j=0}^{K-1} \gamma^j \paren*{8 H
        \sqrt{\frac{dH\iota_{2, 5}}{\NP}}}
        \numeq{\leq}{b}
        \square \sqrt{\frac{H}{c_4}} \varepsilon
        \numeq{\leq}{c}
        \square \sqrt{\frac{H}{c_4}}\,,
    \end{align*}
    where (a) is due to $\EkboundE$, (b) is due to the value of $M$, and (c) is due to $\varepsilon \in (0, 1 / H]$.
    
    Second, \cref{lemma:k gamma to k-th inequality} with the value of $K$ indicates that 
    $$
        \clubsuit \leq \frac{\square}{c_3}\;.
    $$
    
    Finally, $\diamondsuit$ can be bounded as 
    \begin{align*}
        \diamondsuit
        \numeq{\leq}{a}
        8\gamma H^2 \sqrt{\frac{d\iota_{2, 5}}{\NP}}
        \numeq{\leq}{b}
        \square H \sqrt{\frac{\varepsilon}{c_4}}
        \numeq{\leq}{c}
        \square \sqrt{\frac{H}{c_4}}\,,
    \end{align*} 
    where (a) is due to $\epskboundE$, (b) is due to the value of $\NP$, and (c) is due to $\varepsilon \in (0, 1 / H]$.
    
    Inserting these results into the inequality \eqref{eq:coarse last value error}, we have
    $
        \infnorm{\vf{*} - \vf{K}}
        \leq
        \square \sqrt{H}(c_3^{-1} + c_4^{-0.5})
    $.
    Therefore, for some $c_3$ and $c_4$, the claim holds.
\end{proof}

\subsection{Proof of \cref{theorem: epsilon bound} (\textbf{Step 3}) and (\textbf{Step 4})}

As discribed in \cref{subsec:proof sketch}, the proof requires 
tight bounds on $\heartsuit_k = H^{-1}\sum_{j=0}^{k-1} \gamma^j \pi_k P_{k-j}^{k-1} |\phi^\top \kwsum(f, E_{k-j})|$ and $\clubsuit_k = H^{-1}\sum_{j=0}^{k-1} \gamma^j \pi_* P_*^j |\phi^\top \kwsum(f, E_{k-j})|$.
We first derive the bound of $\clubsuit_K$ and then derive the bound of $\heartsuit_K$.

\subsubsection{$H^{-1}\sum_{j=0}^{K-1} \gamma^j \pi_* P_*^j |\phi^\top \kwsum(f, E_{K-j})|$ Bound ($\clubsuit_K$)}\label{subsubsec:clubsuit proof}
As discribed in \textbf{Step 3} of \cref{subsec:proof sketch}, we need a bound of the discounted sum of $\sigma(\vf{*} - \vf{k})$ for $\clubsuit_k$.
Then, the following lemma is useful.

\begin{lemma}\label{lemma:v-diff bound}
    Conditioned on $\vboundE \cap \EkboundE \cap \epskboundE$,
    \begin{align*}
        \infnorm{\vf{*} - v_k}
        < \min \brace*{3 H, \Psi_k} 
        \text{ and }
        \infnorm{\sigma({\vf{*} - v_k})}
        < \min \brace*{3 H, \Psi_k} ,
    \end{align*}
    where
    $$
    \Psi_k = 
    3 H \paren*{
    \max(\gamma, \alpha)^{k-1}
        + \frac{A_{\gamma, k-1}}{A_\infty}
    }
    + \frac{24H^2\uf}{\lf}\sqrt{\frac{d\iota_2}{\NP}} \paren*{1 + \sqrt{\frac{1}{A_\infty}}}
    $$
    for all $k \in [K]$.
\end{lemma}

\begin{proof}\label{proof:v-diff bound core set}

Let
$
    e_k \df \displaystyle \gamma^k H + H \max_{j \in [k]} \norm*{\phi^\top \kwsum(f, \varepsilon_j)}_\infty
$. 
From \cref{lemma:v error prop}, for any $k \in [K]$,
\begin{align*}
    \vf{*} - v_k
    \geq &
    - 2 \gamma^k H \bone - \sum_{j=0}^{k-1} \gamma^j \pi_{k-1} P_{k-j}^{k-1} \phi^\top\kwsum(f, \varepsilon_{k-j})
    \geq 
    - 2e_k\bone\;, \\
    \text{ and }\; \vf{*} - v_k
    \leq 
    &\Gamma_{k-1} + 2 H \gamma^k \bone - \sum_{j=0}^{k-1} \gamma^j \pi_{k-1} P_{k-1-j}^{k-2} \phi^\top\kwsum(f, \varepsilon_{k-j})\,\\
    \leq & 2H \paren*{
        \alpha^{k-1}
        + \frac{A_{\gamma, k-1}}{A_\infty}
        + \frac{1}{A_\infty} \max_{j\in[k-1]}\norm*{\phi^\top \kwsum(f, E_j)}_\infty
    }\bone
    + 2e_k \bone\;.
\end{align*}

Note that $\infnorm{\vf{*} - v_k} \leq 3H$ due to $\vboundE$ for any $k \in [\NIter]$.
Also, due to $\epskboundE$ and $\EkboundE$, 
$ \norm*{\phi^\top \kwsum(f, \varepsilon_k)}_\infty \leq (8H \uf / \lf)\sqrt{d\iota_{2, 5} / \NP}$ and $ \norm*{\phi^\top \kwsum(f, E_k)}_\infty \leq (8 H \uf / \lf)\sqrt{dA_\infty \iota_{2, 5}/\NP}$ for any $k \in [K]$.
Therefore,
$$
    \abs{\vf{*} - v_k}
    \leq 
    3 H \min \brace*{
        1,
        \max(\gamma, \alpha)^{k-1}
        + \frac{A_{\gamma, k-1}}{A_\infty}
        + \frac{8H\uf}{\lf}\sqrt{\frac{d\iota_{2, 5}}{\NP}} \paren*{1 + \sqrt{\frac{1}{A_\infty}}}
    }\bone
$$
for all $k \in [K]$.
Also, due to \cref{lemma:popoviciu}, 
\begin{align*}
    \sigma(\vf{*} - v_k)
    \leq
    3 H \min \brace*{
        1,
        2\max(\gamma, \alpha)^{k-1}
        + \frac{A_{\gamma, k-1}}{A_\infty}
        + \frac{8H\uf}{\lf}\sqrt{\frac{d\iota_{2, 5}}{\NP}} \paren*{1 + \sqrt{\frac{1}{A_\infty}}}
    }\bone\,.
\end{align*}
This concludes the proof.
\end{proof}

Now we have the following bound on $\clubsuit_K$.

\begin{lemma}\label{lemma:sum of E_k star bound}
    Assume that $\varepsilon \in (0, 1/H]$.
    Conditioned on $\funcE \cap \vboundE \cap \EkboundE \cap \epskboundE \cap \EkrboundE$, with the settings of \cref{theorem: epsilon bound}, 
    \begin{align*}
        \clubsuit_K = \frac{1}{H}\sum_{k=0}^{K-1} \gamma^k \pi_* P_{*}^{k} \abs{\phi^\top\kwsum(f, E_{K-k})}
        \leq
        \square \paren*{c_1^{-1} +  c_2^{-0.5}}\varepsilon \bone\,.
    \end{align*}
\end{lemma}

\begin{proof}
Using the conditions $\funcE \cap \vboundE \cap \EkboundE \cap \epskboundE \cap \EkrboundE$, for all $k \in [\NIter]$, we have
\begin{equation}\label{eq:sigma-f ineq}
\coremaxf{x, a}{f} \frac{\sigma(v_{k})(x, a)}{f(x, a)}
\numeq{\leq}{a}
\coremaxf{x, a}{f} \underbrace{\frac{\sigma(v_*)(x, a)}{f(x, a)}}_{\leq 1 \text{ from } \funcE} + \frac{\sigma(v_* - v_{k})(x, a)}{\lf}
\numeq{\leq}{b}
1 + \frac{\Psi_k}{\sqrt{H}}\;, 
\end{equation}
where (a) is due to \cref{lemma:variance decomposition} and (b) is due to the conditions of $\vboundE \cap \EkboundE \cap \epskboundE$ and \cref{lemma:v-diff bound}.

Note that with the conditions and the settings of $\varepsilon \in (0, 1 / H]$, $\alpha=\gamma$, and $\NP=\left\lceil{\frac{c_2dH^2\xi_{2, 5}}{\varepsilon^2}}\right\rceil$,
we have $A_\infty = H$, $A_{\gamma, k} = k\gamma^k$, and $\uf/\lf \leq \sqrt{H}$.
Therefore,
\begin{align*}
\frac{\Psi_k}{\sqrt{H}} 
&= 3 \sqrt{H} \paren*{
\max(\gamma, \alpha)^{k-1}
    + \frac{A_{\gamma, k-1}}{A_\infty}
}
+ \frac{24H\uf}{\lf}\sqrt{\frac{dH\iota_2}{\NP}} \paren*{1 + \sqrt{\frac{1}{A_\infty}}}\\
&\leq 
\square \paren[\Bigg]{\sqrt{H}\gamma^{k-1} + \frac{(k-1)\gamma^{k-1}}{\sqrt{H}} + \underbrace{H^2\sqrt{\frac{d\iota_{2, 5}}{\NP}}}_{\leq \varepsilon H} }
\leq 
\square\paren*{\sqrt{H}\gamma^{k-1} + \frac{(k-1)\gamma^{k-1}}{\sqrt{H}} + 1}\;,
\end{align*}
where the last inequality uses that $\varepsilon \in (0, 1 / H]$.
Using this result, for any $k \in [K]$,
\begin{equation}\label{eq:tmp one plus phi-sqrtH}
    \gamma^{2 (K-k)} \paren*{1 + \frac{\Psi_k}{\sqrt{H}}}^2
    \leq
    \square \paren*{\sqrt{H}\gamma^{K-1} + \frac{(k-1)\gamma^{K-1}}{\sqrt{H}} + \gamma^{K-k}}^2
    \leq
    \square \paren*{H\gamma^{2K-2} + \frac{(k-1)^2\gamma^{2K-2}}{H} + \gamma^{2K-2k}}^2
\end{equation}
where the last inequality is due to the Cauchy-Schwarz inequality (\cref{lemma:square inequality}).
This result implies that
\begin{equation}
\begin{aligned}
V_K &= 2\sqrt{\frac{2\xi_{2, 5}}{M} \sum_{j=1}^K \alpha^{2 (K-j)} \coremaxf{y, b}{f} \frac{\sigma^2(v_{j-1})(y, b)}{f^2(y, b)}}
\numeq{\leq}{a} 2\sqrt{\frac{2\xi_{2, 5}}{M} \sum_{j=1}^K \gamma^{2 (K-j)} \paren*{1 + \frac{\Psi_j}{\sqrt{H}}}^2} \\
&\numeq{\leq}{b} \sqrt{\frac{\square \xi_{2, 5}}{M} \paren[\Bigg]{HK\gamma^{2K-2} + \frac{\gamma^{2K-2}}{H}\underbrace{\sum_{i=1}^K (i-1)^2}_{\leq K^3 \text{ by \cref{lemma:sum of k power}}} + \underbrace{\sum_{j=1}^K \gamma^{2 (K-j)}}_{H}}}\\
&\numeq{\leq}{c} \sqrt{\frac{\square \xi_{2, 5}}{M}} \paren*{\sqrt{HK}\gamma^{K-1} + \frac{K^{1.5}\gamma^{K-1}}{\sqrt{H}} + \sqrt{H}}
\numeq{\leq}{d} \sqrt{\frac{\square H \xi_{2, 5}}{M}} \paren*{1 + \frac{1}{c_1}}
\numeq{\leq}{e} \frac{\square \varepsilon}{\sqrt{c_2 H d}} \paren*{1 + \frac{1}{c_1}}\;, \label{eq:Vk bound}
\end{aligned}
\end{equation}
where (a) is due to \eqref{eq:sigma-f ineq}, (b) is due to \eqref{eq:tmp one plus phi-sqrtH}, and (c) is due to \cref{lemma:sqrt inequality}. 
(d) uses that $\sqrt{K}\gamma^{K-1} \leq K^{1.5}\gamma^{K-1} \leq \square / c_1$ due to the value of $K$ and \cref{lemma:k gamma to k-th inequality}, and (e) is due to the definition of $M$.

Finally, 
\begin{equation*}
\begin{aligned}
    \frac{1}{H}\sum_{k=0}^{K-1} \gamma^k \pi_* P_{*}^{k} \abs*{\phi^\top \kwsum(f, E_{K-k})}
    &\numeq{\leq}{a} 
    \frac{\sqrt{2d}}{H} \paren*{
    \frac{8 H\xi_{2, 5}}{\lf \NP}
    + 
    2 \sqrt{\frac{2\xi_{2, 5}}{\lf^2 M}}
    + V_K } 
    \sum_{k=0}^{K-1} \gamma^k \pi_* P_{*}^{k} 
    f \\
    &\numeq{\leq}{b} \frac{\square \sqrt{d}}{H}\paren*{\frac{\xi_{2, 5}\sqrt{H}}{M} + \frac{1}{H}\sqrt{\frac{\xi_{2, 5}}{M}} + V_K} \sum_{k=0}^{K-1} \gamma^k \pi_* P_{*}^{k}f \\
    &\numeq{\leq}{c} \frac{\square \sqrt{d}}{H}\paren*{\frac{\xi_{2, 5}\sqrt{H}}{M} + \frac{1}{H}\sqrt{\frac{\xi_{2, 5}}{M}} + V_K} \paren[\Bigg]{H\sqrt{H}\bone + \underbrace{\sum_{k=0}^{K-1} \gamma^k \pi_* P_{*}^{k} \sigma(v_*)}_{\leq \sqrt{2H^3}\bone \text{ by \cref{lemma:total variance}}}}\\
    &\numeq{\leq}{d} \frac{\square \sqrt{d}}{H}
    \paren*{
    \frac{\varepsilon^2}{c_2 d H\sqrt{H}} + 
    \frac{\varepsilon}{H^2\sqrt{c_2 d}} + 
    \frac{\varepsilon}{\sqrt{c_2 H d}} \paren*{1 + \frac{1}{c_1}}}\bone\\
    &\leq \square \paren*{c_2^{-0.5} + c_1^{-1} c_2^{-0.5}}\varepsilon \bone\\
    &\leq \square \paren*{c_1^{-1} + c_2^{-0.5}}\varepsilon \bone\;.
\end{aligned}
\end{equation*}
where (a) is due to $\EkrboundE$ and since $V_k$ is increasing with respect to $k$,
(b) is due to $\lf \geq \sqrt{H}$ by $\funcE$, 
(c) is due to $\sigma(v_*) \leq f \leq \sigma(v_*) + 2\sqrt{H}\bone$ by $\funcE$,
and (d) is due to $\NP=\left\lceil{\frac{c_2dH^2\xi_{2, 5}}{\varepsilon^2}}\right\rceil$ with the inequality \eqref{eq:Vk bound}.
This concludes the proof.
\end{proof}

\subsubsection{$H^{-1}\sum_{k=0}^{K-1} \gamma^k \pi_K P_{K-k}^{K-1} \abs{\phi^\top\kwsum(f, E_{K-k})}$ Bound ($\heartsuit_K$)}\label{subsubsec:heartsuit proof}

As discribed in \textbf{Step 4} of \cref{subsec:proof sketch}, we need a coarse bound of $\sigma(\vf{*} - \vf{\pi_k'})$ for $\heartsuit_k$.
Then, the following lemma is useful.
\begin{lemma}\label{lemma:non-stationary coarse bound}
    Conditioned on $\funcE \cap \EkboundE$,
    with the settings of \cref{theorem: epsilon bound}, the output policies $(\pi_k)_{k=0}^{\NIter}$ satisfy that 
    $
        \infnorm{\vf{*} - \vf{\pi'_{k}}}
        \leq
        \square / \sqrt{c_2} \bone + 2 (H + k) \gamma^k
    $
    for all $k \in [K]$.
\end{lemma}
\begin{proof}[Proof of \cref{lemma:non-stationary coarse bound}]
    Note that $A_\infty = H$ due to $\alpha = \gamma $ and $\uf / \lf \leq \sqrt{H}$ due to $\funcE$.
    Moreover, $\EkboundE$ and the setting of $\NP=\left\lceil{\frac{c_2dH^2\xi_{2, 5}}{\varepsilon^2}}\right\rceil$ indicate that
    \begin{align*}
    \norm*{\phi^\top  \kwsum(f, E_k)}_\infty \leq \square\frac{\uf}{\lf}\sqrt{\frac{dA_\infty H^2\iota_{2, 5}}{\NP}} \leq \frac{\square H \epsilon}{\sqrt{c_2}} \leq \frac{\square}{\sqrt{c_2}} \;,
    \end{align*}
    for any $k \in [\NIter]$ where the last inequality is due to $\epsilon \in (0, 1 / H]$.
    
    Therefore, 
    \begin{align*}
        \frac{1}{A_\infty}\sum_{j=0}^{k-1} \gamma^j \paren*{
            \pi_k P_{k-j}^{k-1} - \pi_* P_*^j
        } \phi^T \kwsum(f, E_{k-j})
        \leq
        \frac{2}{H} \sum_{j=0}^{k-1} \gamma^j \norm*{\phi^T \kwsum(f, E_{k-j})}_\infty\bone
        \leq 
        {\frac{\square}{\sqrt{c_2}}}\bone\,,
    \end{align*}
    and thus
    $
        \vf{*} - \vf{\pi'_k}
        \leq
        \square / \sqrt{c_2} \bone + 2 (H + k) \gamma^k \bone
    $
    due to \cref{lemma:non-stationary error propagation}.
    This concludes the proof.
\end{proof}

Now we are ready to derive the bound of $\heartsuit_K$.
\begin{lemma}\label{lemma:sum of E_k bound}
    Assume that $\varepsilon \in (0, 1/H]$.
    Conditioned on $\funcE \cap \vboundE \cap \EkboundE \cap \epskboundE \cap \EkrboundE$, with the settings of \cref{theorem: epsilon bound}, 
    \begin{align*}
        \heartsuit_K = \frac{1}{H}\sum_{k=0}^{K-1} \gamma^k \pi_K P_{K-k}^{K-1} \abs*{\phi^\top\kwsum(f, E_{K-k})}
        \leq
        \square \paren*{c_1^{-1} +  c_2^{-0.5}}\varepsilon \bone\,.
    \end{align*}
\end{lemma}
\begin{proof}
Following similar steps as in the proofs of \cref{subsubsec:clubsuit proof}, we obtain the following bounds.
\begin{align}
    &\frac{\Psi_k}{\sqrt{H}} \leq  \square\paren*{\sqrt{H}\gamma^{k-1} + \frac{(k-1)\gamma^{k-1}}{\sqrt{H}} + 1}\; \text{ for any } k\in [K] \;, \nonumber \\
    &V_K \leq \frac{\square \varepsilon}{\sqrt{c_2 H d}} \paren*{1 + \frac{1}{c_1}}\;, \nonumber \\
    \text{ and }\; &\frac{1}{H}\sum_{k=0}^{K-1} \gamma^k \pi_K P_{K-k}^{K-1} \abs{\phi^\top \kwsum(f, E_{K-k})}
    \leq \frac{\square \sqrt{d}}{H}
    \paren*{\frac{\xi_{2, 5}\sqrt{H}}{M} + \frac{1}{H}\sqrt{\frac{\xi_{2, 5}}{M}} + V_K}
    \paren*{H \sqrt{H} \bone + \sum_{k=0}^{K-1} \gamma^k \pi_K P_{K-k}^{K-1}  \sigma(v_*)} \;. \label{eq:discounted sum of E_k}
\end{align}

We thus need to bound $\sum_{k=0}^{K-1} \gamma^k \pi_K P_{K-k}^{K-1} \sigma(v_*)$.
Note that $\sigma(v_*)$ can be decomposed as 
$$
\sigma(v_*) 
\numeq{\leq}{a} \sigma(v_* - v_{\pi'_{k}}) + \sigma(v_{\pi'_{k}})
\numeq{\leq}{b} \infnorm{v_* - v_{\pi'_{k}}}\bone + \sigma(v_{\pi'_{k}})
\numeq{\leq}{c} \frac{\square}{\sqrt{c_2}}\bone + 2 (H + k) \gamma^k \bone + \sigma(v_{\pi'_{k}})
\;,
$$ 
where (a) is due to \cref{lemma:variance decomposition}, (b) is due \cref{lemma:popoviciu}, and (c) is due to \cref{lemma:non-stationary coarse bound}.
Accordingly, 
\begin{align*}
    \sum_{k=0}^{K-1} \gamma^k \pi_K P_{K-k}^{K-1} \sigma(v_*) 
    & \leq
    \sum_{k=0}^{K-1} \gamma^k \pi_K P_{K-k}^{K-1} 
    \left(\frac{\square}{\sqrt{c_2}}\bone + 2 (H + K-k) \gamma^{K-k} \bone + \sigma(v_{\pi'_{K-k}})\right)
    \\
    &\leq 
    \frac{\square H}{\sqrt{c_2}} \bone + \square \paren[\Bigg]{HK\gamma^K + \gamma^K\underbrace{\sum_{k=0}^{K-1} (K - k)}_{K^2 \text{ by \cref{lemma:sum of k power}}}} \bone + \underbrace{\sum_{k=0}^{K-1} \gamma^k \pi_K P_{K-k}^{K-1}\sigma(v_{\pi'_{K-k}})}_{\sqrt{2H^3} \bone \text{ by \cref{lemma:total variance}}}\\
    &\numeq{\leq}{a} \square H \sqrt{H}\paren*{c_2^{-0.5} + c_1^{-1} + 1}\bone 
    \numeq{\leq}{b} \square H \sqrt{H}\bone
\end{align*}
where (a) uses that $K\gamma^{K} \leq K^2\gamma^{K} \leq \square / c_1$ due to the value of $K$ and \cref{lemma:k gamma to k-th inequality}, and (b) uses $c_1, c_2 \geq 1$.

Inserting the result into the inequality \eqref{eq:discounted sum of E_k}, and following similar steps as in the proof of \cref{subsubsec:clubsuit proof}, we have
\begin{align*}
    \frac{1}{H}\sum_{k=0}^{K-1} \gamma^k \pi_K P_{K-k}^{K-1} \abs{\phi^\top \kwsum(f, E_{K-k})}
    \leq \frac{\square \sqrt{d}}{H}
    \paren*{
    \frac{\varepsilon^2}{c_2 d H\sqrt{H}} + 
    \frac{\varepsilon}{H^2\sqrt{c_2 d}} + 
    \frac{\varepsilon}{\sqrt{c_2 H d}} \paren*{1 + \frac{1}{c_1}}}\bone
    \leq \square \paren*{c_1^{-1} + c_2^{-0.5}}\varepsilon \bone\;.
\end{align*}
This concludes the proof.
\end{proof}

\subsubsection{Proof of \cref{theorem: epsilon bound}}\label{subsec: proof of epsilon bound}
The derived bounds of $\heartsuit_K$ and $\clubsuit_K$ yield the following proof of \cref{theorem: epsilon bound}.

\begin{proof}[Proof of \cref{theorem: epsilon bound}]
    We condition the proof by the event $\funcE \cap \vboundE \cap \EkboundE \cap \epskboundE \cap \EkrboundE$.
    Note that when \wlsmdvi is run with the settings defined in \cref{theorem: epsilon bound}, 
    \begin{align*}
    &\P(\funcE) \geq 1 - 4\delta/\pc \;, \quad
    \underbrace{\P \parenc*{\vboundE^c}{\funcE} \leq \delta / \pc}_{\text{from \cref{lemma:v is bounded two}}} \;, \quad
    \underbrace{\P \parenc*{\EkboundE^c}{\funcE \cap \vboundE} \leq \delta / \pc}_{\text{ from \cref{lemma:coarse E_k bound two}}}\;, \\
    &\underbrace{\P \parenc*{\epskboundE^c}{\funcE \cap \vboundE} \leq \delta / \pc}_{\text{ from \cref{lemma:eps_k bound core set two}}}\; \quad 
    \underbrace{\P \parenc*{\EkrboundE^c}{\funcE \cap \vboundE} \leq \delta / \pc}_{\text{ from  \cref{lemma:refined E_k bound}}}\;.
    \end{align*}
    With \cref{lemma:cond prob inequality}, these indicates that
    \begin{align*}
        \P (\funcE \cap \vboundE \cap \EkboundE \cap \epskboundE \cap \EkrboundE) 
        &\geq \P\paren{\funcE \cap \vboundE} - \P \parenc{(\EkboundE \cap \epskboundE \cap \EkrboundE)^c}{\funcE \cap \vboundE}\\
        &\geq \P\paren{\funcE} - \P\parenc{\vboundE^c}{\funcE} \\
        & \quad - \P\parenc{\EkboundE^c}{\funcE \cap \vboundE} - \P\parenc{\epskboundE^c}{\funcE \cap \vboundE} - \P\parenc{\EkrboundE^c}{\funcE \cap \vboundE}\\
        & \geq 1 - 8\delta / \pc
    \end{align*}
    Therefore, these events occur with probability at least $1 - 8\delta / \pc$.
    
    Note that under the current settings of \cref{theorem: epsilon bound}, $A_\infty = H$ and $2 (H + K) \gamma^K \leq \square / c_1$ due to \cref{lemma:k gamma to k-th inequality}.
    Combined with \cref{lemma:non-stationary error propagation}, \cref{lemma:sum of E_k star bound}, and \cref{lemma:sum of E_k bound}, we have
    \begin{align*}
        \left|\vf{*} - \vf{\pi'_K}\right|
        &\leq
        \underbrace{\frac{1}{A_\infty} \sum_{i=0}^{K-1} \gamma^i \pi_K P_{K-i}^{K-1} \left|\phi^\top \kwsum(f, E_{K-i})\right|}_{\leq \square \paren*{c_1^{-1} + c_2^{-0.5}}\epsilon \bone\; \text{ 
due to \cref{lemma:sum of E_k bound} }} + 
        \underbrace{\frac{1}{A_\infty} \sum_{j=0}^{K-1} \gamma^j \pi_* P^{j}_{*} \left|\phi^\top \kwsum(f, E_{K-j})\right|}_{\leq \square \paren*{c_1^{-1} + c_2^{-0.5}}\epsilon \bone \; \text{ 
due to \cref{lemma:sum of E_k star bound} }} 
        + \underbrace{2 \paren*{H + K}\gamma^K\bone}_{\leq \square c_1^{-1}\bone}\\
        &\leq \square \paren*{\frac{1}{c_1} + \frac{1}{\sqrt c_2}}\epsilon \bone.
    \end{align*}
    Therefore, for some $c_1$ and $c_2$, the claim holds.
\end{proof}

\section{Formal Theorem and Proof of \cref{informal theorem: evaluate error}}\label{sec:evaluate proof}

Instead of the informal theorem \cref{informal theorem: evaluate error}, we are going to prove the following formal theorem.

\begin{theorem}[Accuracy of \varianceestimate]\label{theorem: evaluate error}
    Let $\pc$ be a positive constant such that $8 \geq \pc \geq 6$ and ${v} \in \cF_v$ be a random variable. 
    Assume that an event $$\infnorm{v_* - {v}} \leq \frac{1}{2}\sqrt{H}$$ occurs with probability at least $1 - 3\delta / \pc$. 
    With a positive constant $c_5 \geq 1$, define
    $$M^{\text{var}}\df\left\lceil{c_5dH^2}\log{\frac{2\pc^2 \ucC K}{(\pc - 3)\delta}}\right\rceil\;.$$ 
    When \varianceestimate is run with the settings $v_\sigma = {v}$ and $\NsigmaTwo=M^{\text{var}}$,
    there exists $c_5 \geq 1$ independent of $d$, $H$, $\aX$, $\aA$, and $\delta$ such that the output $\omega$ satisfies
    $\sigma(v_*) \leq \sqrt{\max(\phi^\top\omega, 0)} + \sqrt{H}\leq \sigma(v_*) + 2\sqrt{H}$
    with probability at least $1 - 4\delta / \pc$, using $\widetilde{\cO} \paren*{2\ucC\NsigmaTwo} = \widetilde{\cO} \paren*{d^2H^2}$ samples from the generative model.
\end{theorem}

We let $\bF_{m}$ be the $\sigma$-algebra generated by random variables $\{v\} \cup \{ y_{ n, x, a} | (n, x, a) \in [m-1] \times \XA \} \cup\{ z_{ n, x, a} | (n, x, a) \in [m-1] \times \XA \}$.
Recall that $v \in \cF_v$ is the random variable that is inputted to \varianceestimate as $v_\sigma = v$ and $\omega \in \R^d$ is the parameter to approximate the variance as $\PVar_\omega \df \phi^\top \omega$.

We first introduce the necessary events.

\begin{event}\label{event:phi q bound}
    $\phiqboundE$ denotes the event $\left|v_{*} - v_\sigma\right| \leq \frac{1}{2}\sqrt{H}\bone$.
\end{event}

\begin{event}\label{event:phi sigma bound}
    $\phisigmaboundE$ denotes the event $\left|\sigma_* - \sqrt{\max(\phi^\top\omega, 0)}\right| \leq \sqrt{H}\bone$.
\end{event}

Due to the setting of \cref{theorem: evaluate error}, $\P\paren{\phiqboundE} \geq 1 - 3\delta / \pc$.
We need the following pivotal lemma to show \cref{theorem: evaluate error}.
\begin{lemma}\label{lemma:phi sigma bound}
     When \varianceestimate is run with the settings of \cref{theorem: evaluate error}, there exists $c_5$ independent of $d$, $H$, $\aX$, $\aA$, and $\delta$ such that $\P \parenc{\phisigmaboundE^c}{\phiqboundE} \leq \delta / \pc$.
\end{lemma}
\begin{proof}\label{proof:phi sigma bound}

For the input $v_\sigma$, we write $\PVar(v_\sigma)$ as $\PVar_v$ by abuse of notation.
Let $\omega^*$ be the unknown underlying parameter that satisfies $\phi^\top  \omega^* = \PVar_v$.
This is ensured to exist by \cref{assumption:linear mdp}.

The weighted KW bound (\cref{lemma:kw bound}) indicates that 
\begin{equation}\label{eq:var estimation bound}
    \left|\PVar_\omega - \PVar_v \right|
    = \left|\phi^\top \omega - \phi^\top \omega^* \right|
    = \left|\phi^\top \kwsum\paren*{\bone, \hPVar - \PVar_v } \right|
    \leq \sqrt{2d} \bone {\coremax{y', b'} \underbrace{\left|\hPVar(y', b') - \PVar_v (y', b')\right|}_{\heartsuit}} \,.
\end{equation}
We are going to bound $\left|\hPVar(x, a) - \PVar_v (x, a)\right|$ for $(x, a) \in \cC$.
Note that $\paren[\Big]{v_\sigma\paren*{y_{m, x, a}} - v_\sigma(z_{m, x, a})}^2 / 2$ is the unbiased estimator of $\PVar_v$ since
\begin{align*}
\E\left[\paren[\Big]{v_\sigma\paren{y_{m, x, a}} - v_\sigma(z_{m, x, a})}^2\right] 
    &= 
    \E\left[\paren[\Big]{v_\sigma\paren*{y_{m, x, a}} - P v_\sigma(x, a) + P v_\sigma(x, a) - v_\sigma(z_{m, x, a}) }^2\right]\\
    &= 
    \E\left[\paren[\Big]{v_\sigma\paren*{y_{m, x, a}} - Pv_\sigma(x, a)}^2\right]
    + \E\left[\paren[\Big]{v_\sigma(z_{m, x, a}) - Pv_\sigma(x, a)}^2\right]\\
    &\;\;\;- 2\E\left[\paren[\Big]{v_\sigma\paren*{y_{m, x, a}} - Pv_\sigma(x, a)}\paren[\Big]{v_\sigma(z_{m, x, a}) - Pv_\sigma(x, a)}\right]\\
    &= 2 \PVar_v(x, a)\,.
\end{align*}

Moreover, $\phiqboundE$ implies $|v_\sigma| \leq |v_\sigma - v_{*}| + |v_{*}| \leq \frac{3}{2}H$.
Also, due to \cref{lemma:popoviciu}, $\PVar_v \leq \frac{9}{4}H^2 \leq 3H^2$.
For a fixed $(x, a) \in \cC$, 
\begin{align*}
    \hPVar (x, a) - \PVar_v (x, a)
    &=
    \frac{1}{\NsigmaTwo} \sum_{m=1}^{\NsigmaTwo} 
    \underbrace{
    \paren[\Big]{
    \paren*{v_\sigma\paren*{y_{m, x, a}} - v_\sigma(z_{m, x, a})}^2 / 2 - \PVar_v(x, a)}}_{\text{bounded by } 8 H^2 }
\end{align*}
is a sum of bounded martingale differences with respect to  $(\bF_{m})_{m = 1}^{\NsigmaTwo}$.
Therefore, using the conditional Azuma-Hoeffding inequality (\cref{lemma:conditional hoeffding}) and taking the union bound over $(x, a) \in \cC$, we have
\begin{align*}
    \P \parenc*{
      \exists (x, a, k) \in \cC \,
      \text{ s.t. }
      \left|\hPVar (x, a) - \PVar_v (x, a)\right|
      \geq
      H^2\sqrt{\frac{128\iota_{2, 3}}{\NsigmaTwo}}
    }{\phiqboundE}
    \leq \frac{\delta}{\pc} \;,
\end{align*}

where $\iota_{2, 3} = \iota_1 + \log (\pc / (\pc - 3)) $ is due to the condition by $\phiqboundE$ with $\P(\phiqboundE) \geq 1 - 3 \delta / \pc$.
We used $\iota_{2, 3}$ since $1 / (1 - 3\delta / \pc) \leq \pc / (\pc - 3)$.
Inserting the result into \eqref{eq:var estimation bound}, with probability $1 - \delta / \pc$,
\begin{equation*}
    \left|\PVar_\omega - \PVar_v \right|
    \leq 16 H^2\sqrt{\frac{d\iota_{2, 3}}{\NsigmaTwo}}\bone\;.
\end{equation*}

Due to the setting of $\NsigmaTwo=\left\lceil{c_5dH^2}\iota_{2, 3}\right\rceil$, 
some $c_5$ exists such that $|\phi^\top\omega - \PVar_v| \leq \frac{1}{4} H \bone$.
This implies that $\abs*{\max(\phi^\top\omega, 0) - \PVar_v} \leq \frac{1}{4} H \bone$ since $\PVar_v \geq 0$, and furthermore, $\abs*{\sqrt{\max(\PVar_\theta, 0)} - \sqrt{\PVar_v}} \leq \frac{1}{2} \sqrt{H} \bone$ due to \cref{lemma:square gap bound}. Finally,
\begin{align*}
\left|\sqrt{\max(\phi^\top \omega, 0)} - \sigma(v_*)\right| 
&\numeq{\leq}{a} {\left|\sqrt{\max(\phi^\top\omega, 0)} - \sqrt{\PVar_v}\right|} + {\left|\sigma(v_*) - \sqrt{\PVar_v}\right|} \\
&\numeq{\leq}{b} {\left|\sqrt{\max(\phi^\top\omega, 0)} - \sqrt{\PVar_v}\right|} + \underbrace{\left|v_* - v_\sigma \right|}_{\leq \frac{1}{2}\sqrt{H}\bone\; \text{ due to }\; \phiqboundE } 
\leq \sqrt{H}\bone\;,
\end{align*}
where (a) is due to \cref{lemma:variance decomposition}, (b) is due to \cref{lemma:popoviciu}.
This concludes the proof.

\end{proof}

We are now ready to prove \cref{theorem: evaluate error}.

\begin{proof}[Proof of \cref{theorem: evaluate error}]
    The claim holds by showing that the event $\phiqboundE \cap \phisigmaboundE$ occurs with high probability.
    Note that when \varianceestimate is run with the settings defined in \cref{theorem: evaluate error}, $\P\paren{\phiqboundE^c} \leq 3\delta / \pc$ and $\P \parenc{\phisigmaboundE^c}{\phiqboundE} \leq \delta / \pc$.
    According to \cref{lemma:cond prob inequality}, we have
    \begin{align*}
        \P (\phiqboundE \cap \phisigmaboundE)
        \geq \P (\phiqboundE) - \P \parenc{\phisigmaboundE^c}{\phiqboundE} 
        \geq 1 - 4\delta / \pc\;.
    \end{align*}
    Therefore, these events occur with probability at least $1 - 4\delta / \pc$ and thus the claim holds.
\end{proof}

\section{Formal Theorem and Proof of \cref{informal theorem: sample complexity of vwls mdvi}}\label{sec:proof of vwls mdvi}

Instead of the informal \cref{theorem: sample complexity of vwls mdvi}, we prove the following formal theorem.
In the theorem, we denote $\NIter^{\text{ls}}$ and $\NP^{\text{ls}}$ as the values defined in \cref{theorem: sqrt H bound}, $\NIter^{\text{wls}}$ and $\NP^{\text{wls}}$ as the values defined in \cref{theorem: epsilon bound}, and $M^{\text{var}}$ as the value defined in \cref{theorem: evaluate error}.

\begin{theorem}[Sample complexity of \vwlsmdvi]\label{theorem: sample complexity of vwls mdvi}
    Assume that $\varepsilon \in (0, 1 / H]$ and $\pc = 8$.
    There exist positive constants $c_1, c_2, c_3, c_4, c_5 \geq 1$ independent
    of $d$, $H$, $\aX$, $\aA$, $\varepsilon$ and $\delta$ such that
    when \vwlsmdvi is run with the settings 
    $\alpha = \gamma$, 
    $\NIter=\NIter^{\text{ls}}$, 
    $\NP=\NP^{\text{ls}}$, 
    $\widetilde{\NIter}=\NIter^{\text{wls}}$,
    $\widetilde{\NP}=\NP^{\text{wls}}$, and 
    $\NsigmaTwo=M^{\text{var}}$,
    the output sequence of policies $\pi'$ satisfy
    $
        \infnorm{\vf{*} - \vf{\pi'}}
        \leq
        \varepsilon
    $
    with probability at least $1 - \delta$, using 
    $$\widetilde{\cO} \paren*{\ucC \NIter \NP + \ucC\widetilde{\NIter} \widetilde{\NP} + \ucC (\NsigmaTwo + 1)} = \widetilde{\cO} \paren*{d^2H^3 / \varepsilon^2}$$ 
    samples from the generative model.
\end{theorem}

\begin{proof}\label{proof: proposal bound}
The claim is easily seen from \cref{theorem: sqrt H bound}, \cref{theorem: evaluate error}, and \cref{theorem: epsilon bound}.

From \cref{theorem: sqrt H bound}, the first \wlsmdvi in \vwlsmdvi outputs $v_{\NIter}$ such that $\infnorm{v_* - v_\NIter} \leq 1 / 2 \sqrt{H}$ with probability $1 - 3\delta / \pc$.
This $v_\NIter$ satisfies the requirement of \cref{theorem: evaluate error}.

According to \cref{theorem: evaluate error}, \varianceestimate in \vwlsmdvi outputs $\omega$ such that $\sigma(v_*) \leq \sqrt{\max(\phi^T\omega, \bzero)} + \sqrt{H} \bone \leq \sigma(v_*) + 2\sqrt{H}\bone$ with probability $1 - 4\delta / \pc$. Therefore, $\tsigma = \min\paren*{\sqrt{\max(\phi^T\omega, \bzero)} + \sqrt{H}\bone, H \bone}$ defined in the algorithm can be used as the weighting function of \cref{theorem: epsilon bound}.

Finally, \cref{theorem: epsilon bound} indicates that the second \wlsmdvi in \vwlsmdvi outputs the $\varepsilon$-optimal policy with probability $1 - 8\delta / \pc$.
When $\pc=8$, \vwlsmdvi outputs the $\varepsilon$-optimal policy with probability at least $1 - \delta$.
\end{proof}

\section{Pseudocode of Missing Algorithms}\label{appendix:missing algorithms}

\begin{algorithm}[H]
    \caption{Tabular MDVI $(\alpha, K, M)$}\label{algo:tabular mdvi}
    \begin{algorithmic}
    \STATE {\bfseries Input:} {$\alpha \in [0, 1)$, $K$, and $M$.}
    \STATE Initialize $s_0 = \bzero \in \R^{\XA}$ and $w_0 = w_{-1} = \bzero \in \R^{\X}$.
    \FOR{$k=0$ {\bfseries to} $K-1$}
        \STATE $v_k = w_k - \alpha w_{k-1}$.
        \FOR{{ each state-action pair} $\paren*{x, a} \in \XA$}
            \STATE $q_{k+1} (x, a) = r (x, a) + \gamma \widehat{P}_k(M)v_k(x, a)$.
        \STATE $s_{k+1} = q_{k+1} + \alpha s_k$ and $w_{k+1} (x) = \max_{a \in \A} s_{k+1} (x, a)$ for each $x \in \X$.
        \ENDFOR
    \ENDFOR
    \STATE {\bfseries Return:} {$(\pi_k)_{k=0}^K$ , where $\pi_k$ is greedy policy with respect to $s_k$.}
    \end{algorithmic}
\end{algorithm}

\begin{algorithm}[H]
    \caption{InitializeDesign}\label{algo:KY-initialization}
    \begin{algorithmic}
    \STATE Choose an arbitrary nonzero $c_0 \in \R^d$
    \FOR{$j=0$ {\bfseries to} $d-1$}
        \STATE $(\overline{x}_j, \overline{a}_j) = \argmax_{(x, a) \in \XA}c_j^\top \phi(x, a)$.
        \STATE $(\underline{x}_j, \underline{a}_j) = \argmin_{(x, a) \in \XA}c_j^\top \phi(x, a)$.
        \STATE $y_j =\phi (\overline{x}_j, \overline{a}_j) - \phi (\underline{x}_j, \underline{a}_j)$.
        \STATE Choose an arbitrary nonzero $c_{j+1}$ orthogonal to $y_0, \dots, y_j$.
    \ENDFOR
    \STATE Let $Z:=\{(\overline{x}_j, \overline{a}_j), (\underline{x}_j, \underline{a}_j) \mid j=0, \dots, d-1\}$.
    \STATE Choose $\rho$ to put equal weight on each of the distinct points of $Z$.
    \STATE {\bfseries Return:} {$\rho$.}
    \end{algorithmic}
\end{algorithm}

\begin{algorithm}[H]
    \caption{Frank-Wolfe for finite $\X$ $(f, \varepsilon^{\mathrm{FW}})$}\label{algo:frank wolfe}
    \begin{algorithmic}
    \STATE {\# \it We write $X$ be the size of $\X$. Without loss of generality, we assume $\X=[X]$ and $\A=[A]$.}
    \STATE {\bfseries Input:} {$f: \XA \to (0, \infty)$, $\varepsilon^{\mathrm{FW}}\in \R$.}
    \STATE $\rho = {\texttt{InitializeDesign}}()$ by \cref{algo:KY-initialization}.
    \STATE Let $U: \rho \mapsto \operatorname{diag}(\rho) \in \R^{XA\times XA}$ where $\operatorname{diag}$ constructs a diagonal matrix with elements of $\rho$.
    \STATE For $(x, a) \in \XA$, let $\Phi_f \in \R^{XA\times d}$ be a matrix where its $(xA + a)$ th row is $\phi(x, a) / f(x, a)$.
    \STATE Let $H : \rho \mapsto (\Phi_f U(\rho) \Phi_f)^{-1}$.
    \STATE Let $\omega: (x, a, \rho) \mapsto \phi(x, a)^\top H(\rho) \phi(x, a)$
    \STATE Let $\delta: \rho \mapsto \max_{(x,a)\in \XA}(\omega(x, a, \rho) - d) / d$
    \WHILE{$\delta(\rho) > \varepsilon^{\mathrm{FW}}$}
        \STATE Let $(y, b) := \argmax_{(x, a) \in \XA}\omega(x, a)$
        \STATE Let $\lambda^* := \paren*{\omega(y, b) - d} / \paren*{(d - 1) \omega(y, b)}$
        \STATE $\rho(y, b) \leftarrow \rho(y, b) + \lambda^*$
        \STATE $\rho \leftarrow \rho / (1 + \lambda^*)$
    \ENDWHILE
    \STATE $C = \brace*{(x, a) \mid \; \omega(x, a, \rho) \geq d\paren*{1 + \frac{\delta(\rho)d}{2} - \sqrt{\delta(\rho)(d - 1) + \frac{\delta(\rho)^2 d^2}{4}}}}$
    \STATE $G = \sum_{(y, b) \in \cC} \frac{\rho(y, b)}{f^2(y, b)} \phi (y, b) \phi (y, b)^\top$
    \STATE {\bfseries Return:} {$\rho, \cC, G$.}
    \end{algorithmic}
\end{algorithm}

\begin{algorithm}[H]
    \caption{DVW for online (Munchausen-)DQN}\label{algo:DVW for online RL}
    \begin{algorithmic}
    \STATE {\bfseries Input:} {$K \in \N$ the number of update iteration, $F \in \N$ the target update interval, $T \in \N$ the number of environment steps in one iteration.}
    \STATE Initialize $\theta$ and $\omega$ at random. $\btheta = \htheta = \theta$ and $\bomega = \omega$.
    \STATE Initialize $\eta = 1$.
    \STATE Initialize $\cB = \{\}$.
    \FOR{$k=0$ {\bfseries to} $K-1$}
        \FOR{$t=0$ {\bfseries to} $T-1$}
            \STATE Collect a transition $b = (x, a, r, x^\prime)$ from the environment.
            \STATE $\cB \leftarrow \cB \cup \{b\}$.
        \ENDFOR
        \STATE Sample a random batch of transition $B_k \in \cB$. 
        \STATE On $B_k$, update $\omega$ with one step of SGD on $\cL(\omega)$, see \cref{eq:variance loss}.
        \STATE On $B_k$, update $\eta$ with one step of SGD on $\cL(\eta)$, see \cref{eq:weight scaler loss}.
        \STATE On $B_k$ and with $f$ of \cref{eq:practical weight function}, update $\theta$ with one step of SGD on $\cL(\theta)$, see \cref{eq:weighted M-DQN loss}.
        \IF{ $k \mod F = 0$}
            \STATE $\htheta \leftarrow \btheta$.
            \STATE $\btheta \leftarrow \theta$.
            \STATE $\bomega \leftarrow \omega$.
        \ENDIF
    \ENDFOR
    \STATE {\bfseries Return:} {A greedy policy with respect to $q_\theta$\;}
    \end{algorithmic}
\end{algorithm}

\newpage
\section{Experiment Details}\label{sec:experiment details}

\looseness=-1
The source code for all the experiments is available at \url{https://github.com/matsuolab/Variance-Weighted-MDVI}.

\subsection{Details of \cref{subsec:linear mdp experiment}} \label{subsec:linear mdp details}

\paragraph{Hard linear MDP.} 
The hard linear MDP we used is based on the \textbf{Theorem H.3} in \citet{weisz2022confident}.
Specifically, the MDP has two states: $\X = \{x_0, x_1\}$ with $x_0$ being the initial state and $x_1$ being the absorbing state. 
To add randomness to the MDP, the action space is constructed as $\A=\{a_0, a_1, \dots, a_{A}\}$ where $a_i \in \R^{d-2}$ for $i=[A]$ is randomly sampled from a multivariate uniform distribution of $\mathcal{U}(\bzero, \bone)$ with $d-2$ dimension.
Same as \citet{weisz2022confident}, for all $a$, the feature map is defined as

$$
\phi\left(x_0, a\right)=\left(1,0, a^{\top}\right)^{\top} \quad \text { and } \quad \phi\left(x_1, a\right)=(0,1,0, \ldots, 0)^{\top}\;.
$$

Using $\psi=(1,0, \ldots, 0)^{\top}$, we make the state $x_0$ be the rewarding state as

$$
r \left(x_0, a\right)=\phi\left(x_0, a\right)^\top \psi =1 \quad \text { and } \quad r \left(x_1, a\right)=\phi\left(x_0, a\right)^\top\psi=0\;.
$$

Let 
$
\mu\left(x_0\right)=\left(\gamma, 0, 0.01 a_{0}^{\top}\right)^{\top}$ and $\mu\left(x_1\right)=\left(1-\gamma, 1,-0.01 a_{0}^{\top}\right)^{\top}$
be the design parameters for the transition probability kernel.
This implies that 
$$
\begin{aligned}
& P\left(x_0 \mid x_0, a\right)=\gamma+0.01 \cdot a_0^{\top} a, &&\quad P\left(x_1 \mid x_0, a\right)=1-\gamma-0.01 \cdot a_0^{\top} a, \\
& P\left(x_0 \mid x_1, a\right)=0, &&\quad P\left(x_1 \mid x_1, a\right)=1\;.
\end{aligned}
$$

Intuitively, choosing an action similar to $a_0$ increases the probability of transitioning to $x_0$ and yields a higher return.
We provide the hyperparameters of the MDP in \cref{tab:hard linear mdp params}.

\begin{table}[H]
\caption{Hyperparameters of hard linear MDP in \cref{subsec:linear mdp experiment}}
\label{tab:hard linear mdp params}
\vskip 0.15in
\begin{center}
\begin{small}
\begin{tabular}{l l| l }
\toprule
\multicolumn{2}{l|}{Parameter} &  Value\\
\midrule
\multicolumn{2}{l|}{\it{MDP parameter}}& \\
& action space size & $A=30$ \\
& dimension of the feature map & $d=4$ \\
& discount factor & $\gamma=0.9$ \\
\midrule
\multicolumn{2}{l|}{\it{Algorithm parameter}}& \\
& weight for MDVI update & $\alpha=0.9$ \\
& accuracy of the Frank-Wolfe algorithm & $\varepsilon^{\mathrm{FW}}=0.01$ \\
\bottomrule
\end{tabular}
\end{small}
\end{center}
\vskip -0.1in
\end{table}

\paragraph{Algorithm implementations.}

The algorithms \wlsmdvi and \vwlsmdvi are implemented according to \cref{algo:wlsmdvi} and \cref{algo:vwlsmdvi}.
The optimal designs are computed using the Frank-Wolfe (FW) algorithm (\cref{algo:frank wolfe}).
We provide the hyperparameters of the algorithms in \cref{tab:hard linear mdp params}.

\subsection{Details of \cref{subsec:gridworld experiment}} \label{subsec:gridworld details}

\paragraph{Gridworld environment.} 
The gridworld environment we used is a $25 \times 25$ grid with randomly placed $8$ pitfalls.
This is similar to the gridworld environment of \citet{fu2019diagnosing}, but there are some differences.

The agent starts from the top left grid and can move to any of its neighboring grids with success probability $0.6$, or to a different random direction with probability $0.4$.
The agent receives $+1$ reward when it reaches the goal grid located at the bottom right grid.
Other rewards are set to $0$.
When the agent enters a pitfall, the agent can no longer move and receives $0$ reward until the environment terminates. 
We use $\gamma=0.995$ and the environment terminates after $200$ steps.

\paragraph{Algorithm implementations.}
We implement the environment and algorithms using ShinRL \citep{kitamura2021shinrl}.
For the implementation of M-DQN, same as \citet{vieillard2020munchausen}, we clip the value of log-policy term by $\max(\log \pi_{\btheta}, l_0)$ with $l_0 > 0$ to avoid numerical issues in \cref{eq:weighted M-DQN loss}.
We provide the hyperparameters used in the experiment in \cref{tab:gridworld-params}. 
In the table, we denote $\text{FC}(n)$ be a fully convolutional layer with $n$ neurons.

\begin{table}[H]
\caption{Hyperparameters of algorithms in \cref{subsec:gridworld experiment}}
\label{tab:gridworld-params}
\vskip 0.15in
\begin{center}
\begin{small}
\begin{tabular}{l l| l }
\toprule
\multicolumn{2}{l|}{Parameter} &  Value\\
\midrule
\multicolumn{2}{l|}{\it{Shared}}& \\
& optimizer &Adam \\
& iteration ($K$) & $2000000$\\
& target update interval ($F$) & $100$\\
& learning rate & $10^{-3}$\\
& discount factor ($\gamma$) &  $0.995$\\
& horizon ($H$) &  $200$\\
& $q$-network structure & $\text{FC}(128)-\text{FC}(128)-\text{FC}(|\A|)$\\
& activations & Relu \\
\midrule
\multicolumn{2}{l|}{\it{Munchausen-DQN parameters}}& \\
& entropy regularization coefficient ($\kappa$) & $10^{-5}$ \\
& KL regularization coefficient ($\tau$) & $\kappa \gamma / (1 - \gamma)$ \\
& clipping value ($l_0$) & $-1$ \\
\midrule
\multicolumn{2}{l|}{\it{DVW parameters}}& \\
& $\omega$ and $\eta$-optimizer &Adam \\
& activations & Relu \\
& $\text{Var}$-network structure & $\text{FC}(128)-\text{FC}(128)-\text{FC}(|\A|)$\\
& lower threshold parameter ($\underline{c}_f$) & $0.1$ \\
& upper threshold parameter ($\overline{c}_f$) & $0.1$ \\
& learning rate of $\text{Var}_\omega$ & $10^{-3}$\\
& learning rate of $\eta$ & $5.0\times10^{-3}$\\
\bottomrule
\end{tabular}
\end{small}
\end{center}
\vskip -0.1in
\end{table}

\newpage

\subsection{Details of \cref{subsec:minatar experiment}} \label{subsec:minatar details}

\paragraph{Algorithm implementation.}
We implement algorithms as variations of DQN from CleanRL \citep{huang2022cleanrl}.
For a fair comparison, all the algorithms use the same epsilon-greedy exploration strategy.
It randomly chooses an action with probability $e_t$ otherwise chooses a greedy action w.r.t. $q_\theta$.
For the implementation of M-DQN, same as \citet{vieillard2020munchausen}, we clip the value of log-policy term by $\max(\log \pi_{\btheta}, l_0)$ with $l_0 > 0$ to avoid numerical issues in \cref{eq:weighted M-DQN loss}.

We provide the hyperparameters used in the experiment in \cref{tab:minatar-params}. 
In the table, we denote $\text{FC}(n)$ be a fully convolutional layer with $n$ neurons, and  $\text{Conv}^{d}_{a, b}(c)$ be a 2D convolutional layer with $c$ filters of size $a \times b$ and of stride $d$.

\begin{table}[H]
\caption{Hyperparameters of algorithms in \cref{subsec:minatar experiment}}
\label{tab:minatar-params}
\vskip 0.15in
\begin{center}
\begin{small}
\begin{tabular}{l l| l }
\toprule
\multicolumn{2}{l|}{Parameter} &  Value\\
\midrule
\multicolumn{2}{l|}{\it{Shared}}& \\
& $e_k$ (random actions rate) & start from $1.0$ and linearly decay to $0.1$ until the period of $10^6$ steps\\
& $\theta$-optimizer &Adam \\
& iteration ($K$) & $10^7$\\
& target update interval ($F$) & $1000$\\
& learning rate of $q_\theta$ & $2.5 \times 10^{-4}$\\
& replay buffer size ($|\cB|$) & $10^{5}$\\
& batch size ($|\cB_k|$) & $32$\\
& train frequency ($T$) & $4$\\
& discount factor ($\gamma$) &  $0.99$\\
& $q$-network structure & $\text{Conv}^{1}_{3, 3}(16)-\text{FC}(128)-\text{FC}(|\A|)$\\
& activations & Relu \\
\midrule
\multicolumn{2}{l|}{\it{Munchausen-DQN parameters}}& \\
& entropy regularization coefficient ($\kappa$) & $0.003$ \\
& KL regularization coefficient ($\tau$) & $0.027$ \\
& clipping value ($l_0$) & $-1$ \\
\midrule
\multicolumn{2}{l|}{\it{DVW parameters}}& \\
& $\omega$ and $\eta$-optimizer &Adam \\
& $\text{Var}$-network structure & $\text{Conv}^{1}_{3, 3}(16)-\text{FC}(128)-\text{FC}(|\A|)$\\
& lower threshold parameter ($\underline{c}_f$) & $0.1$ \\
& upper threshold parameter ($\overline{c}_f$) & $0.1$ \\
& learning rate of $\text{Var}_\omega$ & $2.5 \times 10^{-4}$\\
& learning rate of $\eta$ & $10^{-3}$\\
\bottomrule
\end{tabular}
\end{small}
\end{center}
\vskip -0.1in
\end{table}

\end{document}